\documentclass[twoside,11pt]{article}

\usepackage{jmlr2e}

\usepackage{url}
\usepackage{amsmath}
\usepackage{amssymb}
\usepackage{epsfig}
\usepackage{graphicx}
\usepackage[]{caption,subfig}
\usepackage{mdwlist}
\usepackage[lined, boxed]{algorithm2e}
\usepackage{multirow}
\usepackage{color}

\graphicspath{{figs/}}

\renewenvironment{proof}{{\noindent \it Proof:}}{\null\hfill$\blacksquare$}

\ShortHeadings{Learning Transformations for Clustering and Classification}{Qiu and Sapiro}
\firstpageno{1}

\begin{document}

\title{Learning Transformations for Clustering and Classification}

\author{\name Qiang Qiu \email qiang.qiu@duke.edu \\
       \addr Department of Electrical and Computer Engineering\\
       Duke University\\
       Durham, NC 27708, USA
       \AND
       \name Guillermo Sapiro \email guillermo.sapiro@duke.edu \\
       \addr Department of Electrical and Computer Engineering,\\
       Department of Computer Science, \\
       Department of Biomedical Engineering\\
       Duke University\\
       Durham, NC 27708, USA}

\editor{*}

\maketitle

\begin{abstract}%   <- trailing '%' for backward compatibility of .sty file
A low-rank transformation learning framework for subspace clustering and classification is here proposed. Many high-dimensional data, such as face images and motion sequences, approximately lie in a union of low-dimensional subspaces. The corresponding subspace clustering problem has been extensively studied in the literature to partition such high-dimensional data into clusters corresponding to their underlying low-dimensional subspaces. However, low-dimensional intrinsic structures are often violated for real-world observations, as they can be corrupted by errors  or deviate from ideal models.
{
We propose to address this by learning a linear transformation on subspaces using nuclear norm as the modeling and optimization criteria.}
The learned linear transformation restores a low-rank structure for data from the same subspace, and, at the same time, forces a maximally separated structure for data from different subspaces.
In this way, we reduce variations within the subspaces, and increase separation between the subspaces for a more robust subspace clustering.
This proposed learned robust subspace clustering framework significantly enhances the performance of existing subspace clustering methods.
Basic theoretical results here presented help to further support the underlying framework.
To exploit the low-rank structures of the transformed subspaces, we further introduce a fast subspace clustering technique, which efficiently combines robust PCA with sparse modeling.
When class labels are present at the training stage, we show this low-rank transformation framework also significantly enhances classification performance.
Extensive experiments using public datasets are presented, showing that the proposed approach
significantly outperforms state-of-the-art methods for subspace clustering and classification.
The learned low cost transform is also applicable to other classification frameworks.
\end{abstract}


\section{Introduction}

High-dimensional data often have a small intrinsic dimension.
For example, in the area of computer vision, face images of a subject (\cite{9point}, \cite{Wright09}), handwritten images of a digit (\cite{ocr}), and trajectories of a moving object (\cite{sfm}) can all be well-approximated by a low-dimensional subspace of the high-dimensional ambient space. Thus, multiple class data often lie in a union of low-dimensional subspaces.
The ubiquitous {subspace clustering} problem is to partition high-dimensional data into clusters corresponding to their underlying subspaces.

Standard clustering methods such as k-means in general are not applicable to subspace clustering. Various methods have been recently suggested for subspace clustering, such as Sparse Subspace Clustering (SSC) (\cite{SSC}) (see also its extensions and analysis in \cite{robustsubspace, ga-ssc, rssc, nssc}), Local Subspace Affinity (LSA) (\cite{LSA}), Local Best-fit Flats (LBF) (\cite{SLBF}),
 Generalized Principal Component Analysis (\cite{gpca}),
 Agglomerative Lossy Compression (\cite{alc}),  Locally Linear Manifold Clustering (\cite{llmc}),
 and Spectral Curvature Clustering (\cite{scc}). A recent survey on subspace clustering can be found in \cite{SubspaceClustering}.

Low-dimensional intrinsic structures, which enable subspace clustering,  are often violated for real-world data.
For example, under the assumption of Lambertian reflectance, \cite{9point} show that face images of a subject obtained under a wide variety of lighting conditions can be accurately approximated  with a 9-dimensional linear subspace. However,  real-world face images are often captured under pose variations; in addition, faces are not perfectly Lambertian, and exhibit cast shadows and specularities (\cite{rpca}).
Therefore, it is critical for subspace clustering to handle corrupted underlying structures of realistic data, and as such, deviations from ideal subspaces.

When data from the same low-dimensional subspace are arranged as columns of a single matrix,  the matrix should be approximately low-rank.  Thus, a promising way to handle corrupted data for subspace clustering is to restore such low-rank structure. Recent efforts have been invested in seeking transformations such that the transformed data can be decomposed as the sum of a low-rank matrix component and a sparse error one (\cite{RASL, lrsalient, TILT}).
\cite{RASL} and \cite{TILT} are proposed for image alignment (see \cite{3dalign} for the extension to multiple-classes with applications in cryo-tomograhy), and \cite{lrsalient} is discussed in the context of salient object detection. All these methods build on recent theoretical and computational advances in rank minimization.

In this paper, we propose to improve subspace clustering and classification by learning a linear transformation on subspaces using matrix rank, via its nuclear norm convex surrogate,  as the optimization criteria.
The learned linear transformation recovers a low-rank structure for data from the same subspace, and, at the same time, forces a maximally separated structure for data from different subspaces (actually high nuclear norm, which as discussed later, improves the separation between the subspaces).
In this way, we reduce variations within the subspaces, and increase separations between the subspaces for more accurate subspace clustering and classification.

For example, as shown in Fig.~\ref{fig:overview}, after faces are detected and aligned, e.g., using \cite{posemodel}, our approach learns linear transformations for face images to restore for the same subject a low-dimensional structure.
By comparing the last row to the first row in Fig.~\ref{fig:overview}, we can easily notice that faces from the same subject across different poses are more visually similar in the new transformed space, enabling better face clustering and classification across pose.

This paper makes the following main contributions:
\begin{itemize*}
\item Subspace low-rank transformation (LRT) is introduced and analyzed in the context of subspace clustering and classification;
\item A Learned Robust Subspace Clustering framework (LRSC) is proposed to enhance existing subspace clustering methods;
\item A discriminative low-rank (nuclear norm) transformation approach is proposed to reduce the variation within the classes and increase separations between the classes for improved classification;
\item We propose a specific fast subspace clustering technique, called Robust Sparse Subspace Clustering (R-SSC), by exploiting low-rank structures of the learned transformed subspaces;
\item We discuss online learning of subspace low-rank transformation for big data;
\item We demonstrate through extensive experiments that the proposed approach significantly outperforms  state-of-the-art methods for subspace clustering and classification.
  \end{itemize*}

The proposed approach can be considered as a way of learning data features, with such features learned in order to reduce within-class rank (nuclear norm), increase between class separation, and encourage robust subspace clustering. As such, the framework and criteria here introduced can be incorporated into other data classification and clustering problems.

\begin{figure*} [t]
\centering
\includegraphics[angle=0, height=0.43\textwidth, width=.70\textwidth]{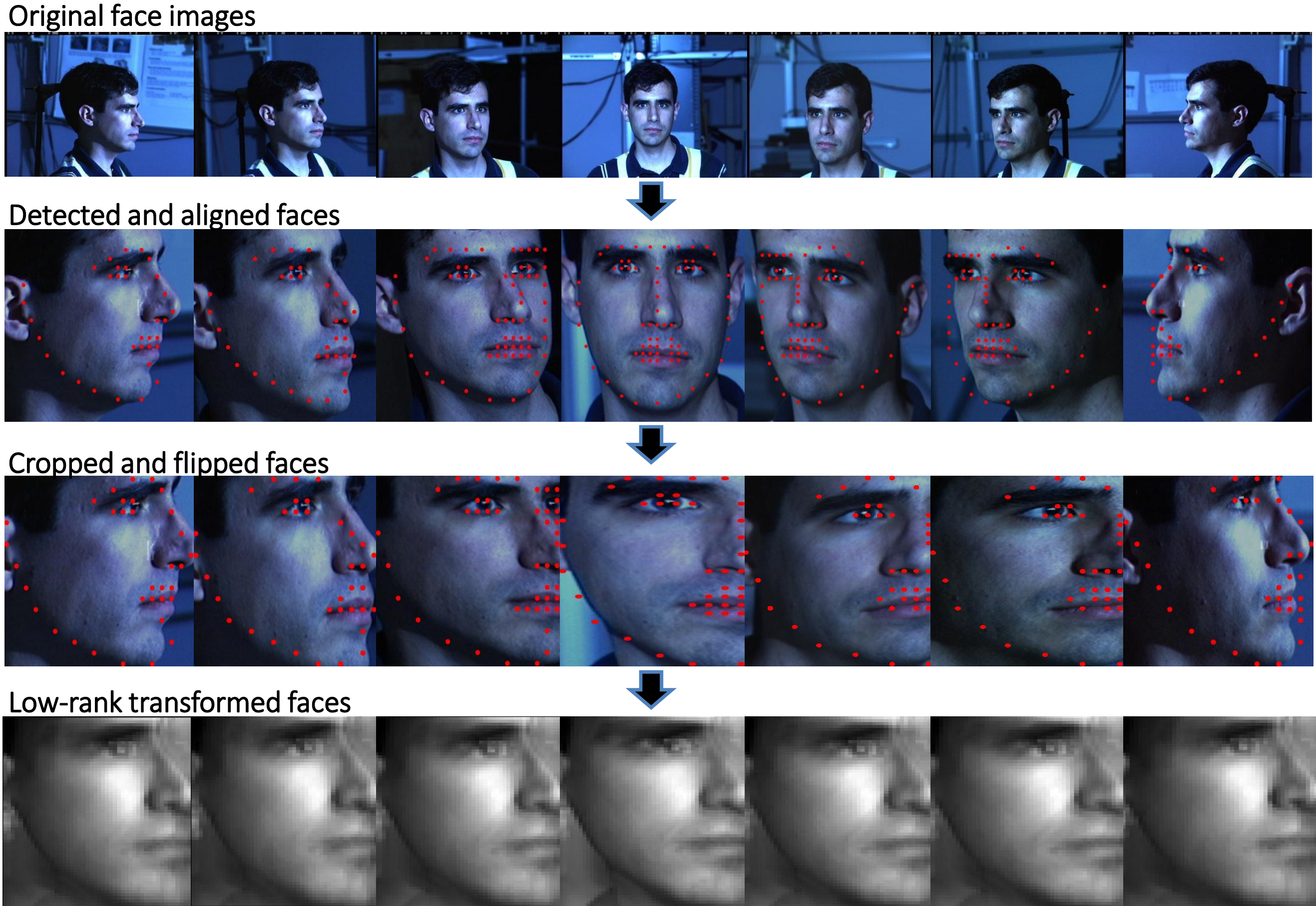}
\caption{Learned low-rank transformation on faces across pose.  In the second row, the input faces are first detected and aligned, e.g., using the method in \cite{posemodel}. Pose models defined in \cite{posemodel} enable an optional crop-and-flip step to retain the more informative side of a face in the third row. Our proposed approach learns linear transformations for face images to restore for the same subject a low-dimensional structure as shown in the last row. By comparing the last row to the first row, we can easily notice that faces from the same subject across different poses are more visually similar in the new transformed space, enabling better face clustering or recognition across pose (note that the goal is clustering/recognition and not reconstruction).}
\label{fig:overview}
\end{figure*}

 In Section~\ref{sec:form}, we formulate and analyze the low-rank transformation learning problem.
 In sections~\ref{sec:sc} and \ref{sec:rec}, we discuss the low-rank transformation for subspace clustering and classification respectively. Experimental evaluations are given in Section~\ref{sec:expr} on public datasets commonly used for subspace clustering evaluation. Finally, Section~\ref{sec:con} concludes the paper.

\section{Learning Low-rank Transformations (LRT) }
\label{sec:form}

Let $\{\mathcal{S}_c\}_{c=1}^C$ be $C$ $n$-dimensional subspaces of $\mathbb{R}^d$ (not all subspaces are necessarily of the same dimension, this is only here assumed to simplify notation).
Given a data set $\mathbf{Y} = \{ \mathbf{y}_i\}_{i=1}^N \subseteq \mathbb{R}^d$,
with each data point $\mathbf{y}_i$ in one of the $C$ subspaces, and in general the data arranged as columns of $Y$.
$\mathbf{Y}_c$ denotes the set of points in the $c$-th subspace $\mathcal{S}_c$,
 points arranged as columns of the matrix $\mathbf{Y}_c$.

As data points in $\mathbf{Y}_c$ lie in a low-dimensional subspace, the matrix $\mathbf{Y}_c$ is expected to be \emph{low-rank}, and such low-rank structure is critical for accurate subspace clustering.
However, as discussed above, this low-rank structure is often violated for real data.

Our proposed approach is to learn a global linear transformation on subspaces.
Such linear transformation restores a low-rank structure for data from the same subspace, and, at the same time, encourages a maximally separated structure for data from different subspaces. In this way, we reduce the variation within the subspaces and introduce separations between the subspaces for more robust subspace clustering or classification.

\subsection{Preliminary Pedagogical Formulation using Rank}

We first assume the data cluster labels are known beforehand for training purposes, assumption to be removed when discussing the full clustering approach in Section~\ref{sec:sc}.
We adopt matrix rank as the key learning criterion (presented here first for pedagogical reasons, to be later replaced by the nuclear norm), and compute one global linear transformation on all subspaces as
\begin{align} \label{rank_obj}
{ \underset{\mathbf{T}} \arg \min \sum_{c=1}^C rank(\mathbf{T Y}_c) - rank(\mathbf{T Y}), ~~s.t. ||\mathbf{T}||_2 =  \gamma},
\end{align}
where $\mathbf{T} \in \mathbb{R}^{d \times d}$ is one global linear transformation on all data points (we will later discuss then $\mathbf{T}$'s dimension is less than $d$), $||\mathbf{\cdot}||_2$ denotes the matrix induced 2-norm,
and $\gamma$ is a positive constant.
Intuitively, minimizing the first \emph{representation} term $\sum_{c=1}^C rank(\mathbf{T Y}_c)$ encourages a consistent representation for the transformed data from the same subspace; and minimizing the second \emph{discrimination} term $-rank(\mathbf{T Y})$ encourages a diverse representation for transformed data from different subspaces (we will later formally discuss that the convex surrogate nuclear norm actually has this desired effect).
The normalization condition $||\mathbf{T}||_2 = \gamma$  prevents the trivial solution $\mathbf{T}=0$.

{ We now explain that the pedagogical formulation in (\ref{rank_obj}) using rank  is however not optimal to simultaneously reduce the variation within the same class subspaces and introduce separations between the different class subspaces, motivating the use of the nuclear norm not only for optimization reasons but for modeling ones as well.}
Let $\mathbf{A}$ and $\mathbf{B}$ be matrices of the same dimensions (standing for two classes $\mathbf{Y}_1$ and $\mathbf{Y}_2$ respectively), and $\mathbf{[A,B]}$ (standing for $\mathbf{Y}$) be the concatenation of $\mathbf{A}$ and $\mathbf{B}$,  we have (\cite{rank-add})
\begin{align} \label{rank_concat}
rank(\mathbf{[A,B]})  \le rank(\mathbf{A}) + rank(\mathbf{B}),
\end{align}
with equality if and only if $\mathbf{A}$ and $\mathbf{B}$ are {disjoint}, i.e., they intersect only at the origin (often the analysis of subspace clustering algorithms considers disjoint spaces, e.g., \cite{SSC}).

It is easy to show that (\ref{rank_concat}) can be extended for the concatenation of multiple matrices,
\begin{align} \label{rank_concat2}
rank([\mathbf{Y}_1, \mathbf{Y}_2, \mathbf{Y}_3, \cdots, \mathbf{Y}_C]) &\le rank(\mathbf{Y}_1) + rank([\mathbf{Y}_2, \mathbf{Y}_3, \cdots, \mathbf{Y}_C]) \\ \nonumber
&\le rank(\mathbf{Y}_1) + rank(\mathbf{Y}_2) + rank([\mathbf{Y}_3, \cdots, \mathbf{Y}_C]) \\ \nonumber
&\ldots \\ \nonumber
&\le \sum_{c=1}^C rank(\mathbf{Y}_c),
\end{align}
%with equality if every pair of matrices is disjoint.
with equality if matrices are independent.
Thus, for (\ref{rank_obj}), we have
\begin{align} \label{rank_objval}
\sum_{c=1}^C rank(\mathbf{T Y}_c) - rank(\mathbf{T Y}) \ge 0,
\end{align}
{
and the objective function (\ref{rank_obj}) reaches the minimum $0$ if
matrices are independent
after applying the learned transformation $\mathbf{T}$.
However, independence does not infer maximal separation, an important goal for robust clustering and classification. For example, two lines intersecting only at the origin are independent regardless of the angle in between, and they are maximally separated only when the angle becomes $\frac{\pi}{2}$.
With this intuition in mind, we now proceed to describe our proposed formulation based on the nuclear norm.
}

\subsection{Problem Formulation using Nuclear Norm}
\label{sec:nuclear}

\begin{figure*} [ht]
\centering
  \subfloat[][ $\theta_{AB}=\frac{\pi}{2}=1.57$.] {\label{fig:2d00O} \includegraphics[angle=0, height=0.25\textwidth, width=.25\textwidth]{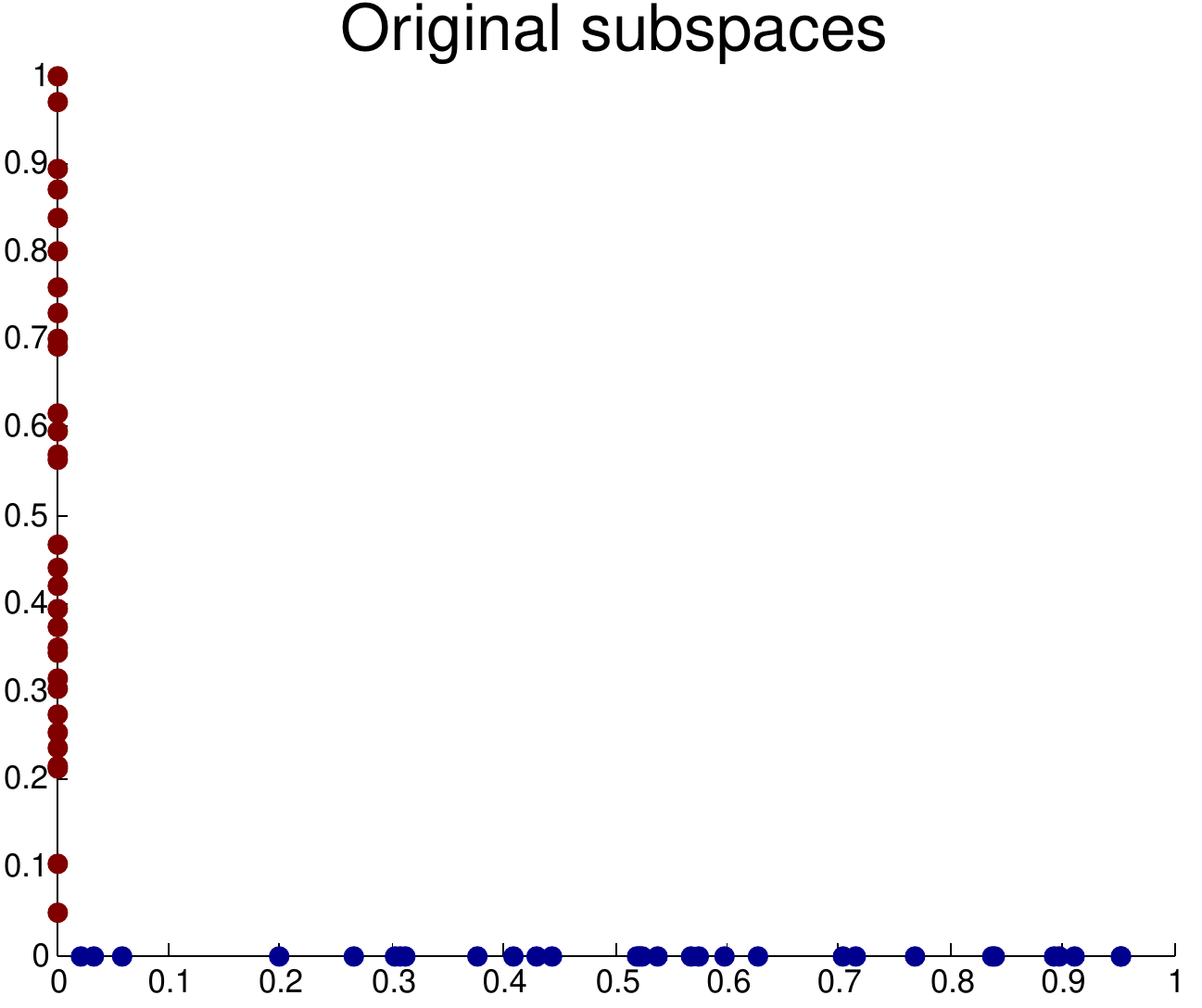}}
 \subfloat[][ $\mathbf{T} = \begin{bmatrix} 1.00 & 0 \\  0& 1.00 \end{bmatrix}$; \\$~~~~~~~~\theta_{AB}=1.57$. ] {\label{fig:2d00T} \includegraphics[angle=0, height=0.25\textwidth, width=.25\textwidth]{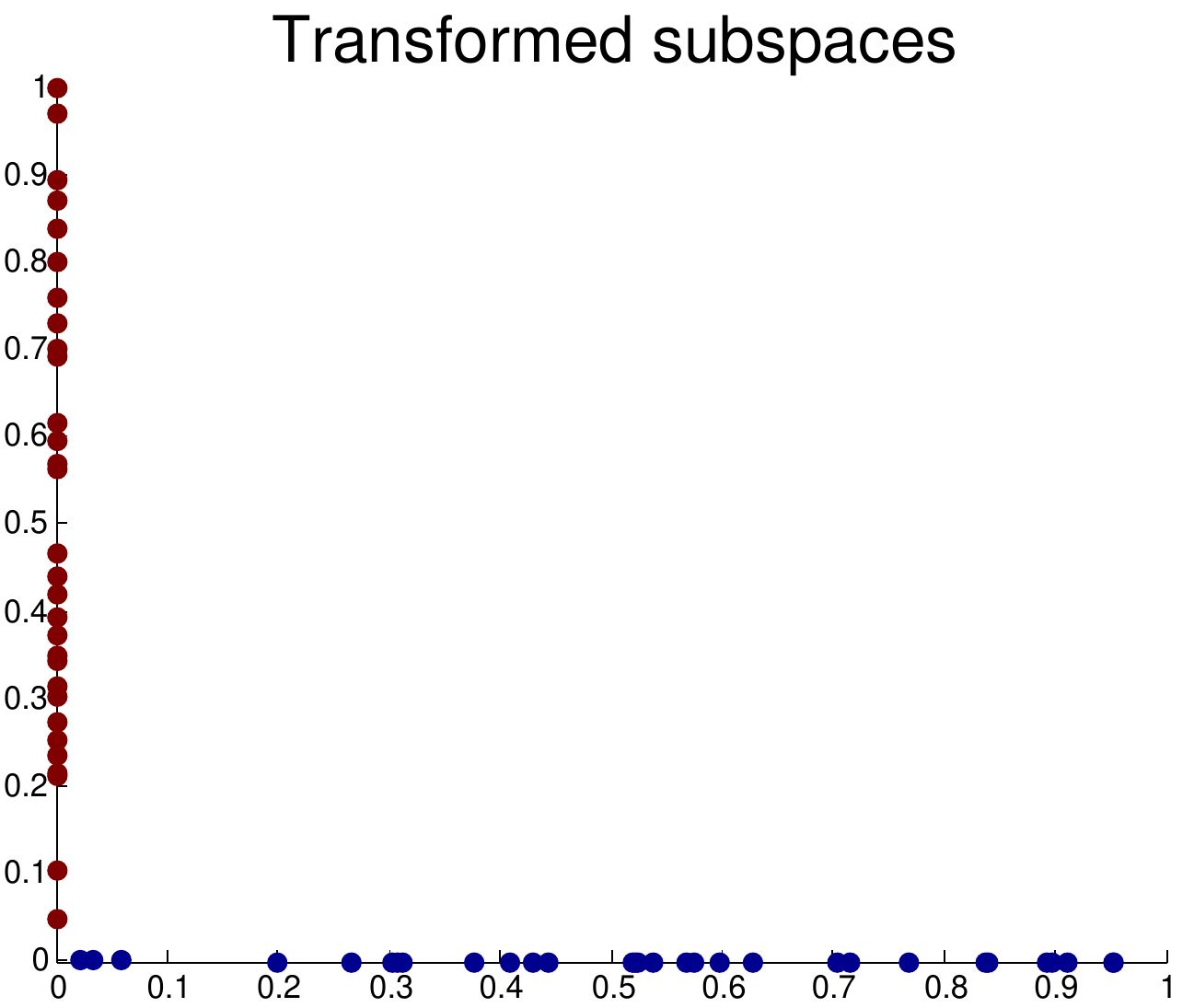} }
 \subfloat[][$\theta_{AB}=\frac{\pi}{4}=0.79$,\\
  $~~~~~|\mathbf{A}|_*=1, |\mathbf{B}|_*=1$, \\$~~~~~|\mathbf{[A,B]}|_*= 1.41$
 ] {\label{fig:2d45O} \includegraphics[angle=0, height=0.25\textwidth, width=.25\textwidth]{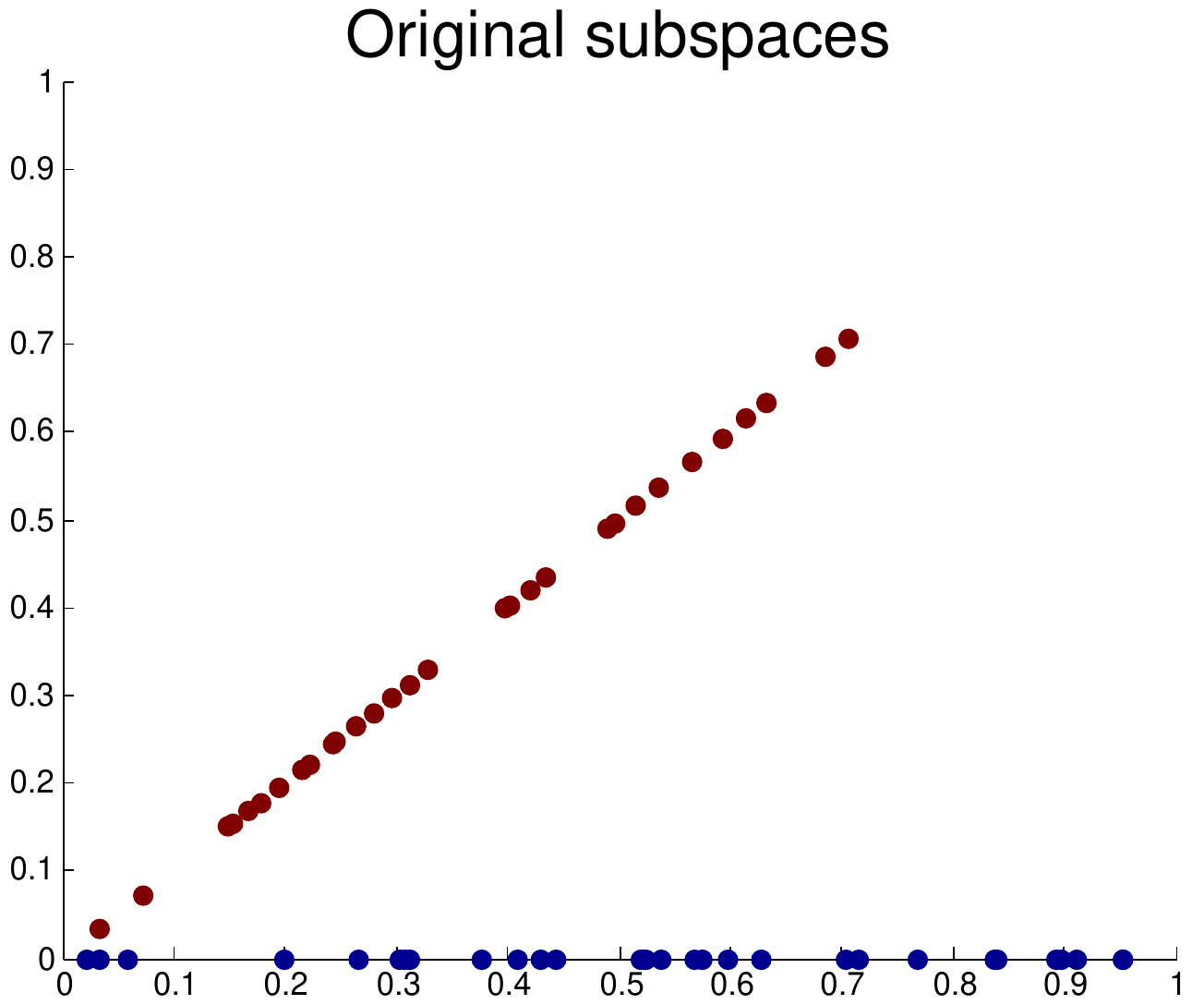}}
  \subfloat[][$\mathbf{T} = \begin{bmatrix} 0.50 & -0.21 \\  -0.21& 0.91 \end{bmatrix}$;\\ $~~~~~~~\theta_{AB}=1.57$,\\
     $~~~~~|\mathbf{A}|_*=1, |\mathbf{B}|_*=1$, \\$~~~~~|\mathbf{[A,B]}|_*= 1.95$
   ] {\label{fig:2d45T} \includegraphics[angle=0, height=0.25\textwidth, width=.25\textwidth]{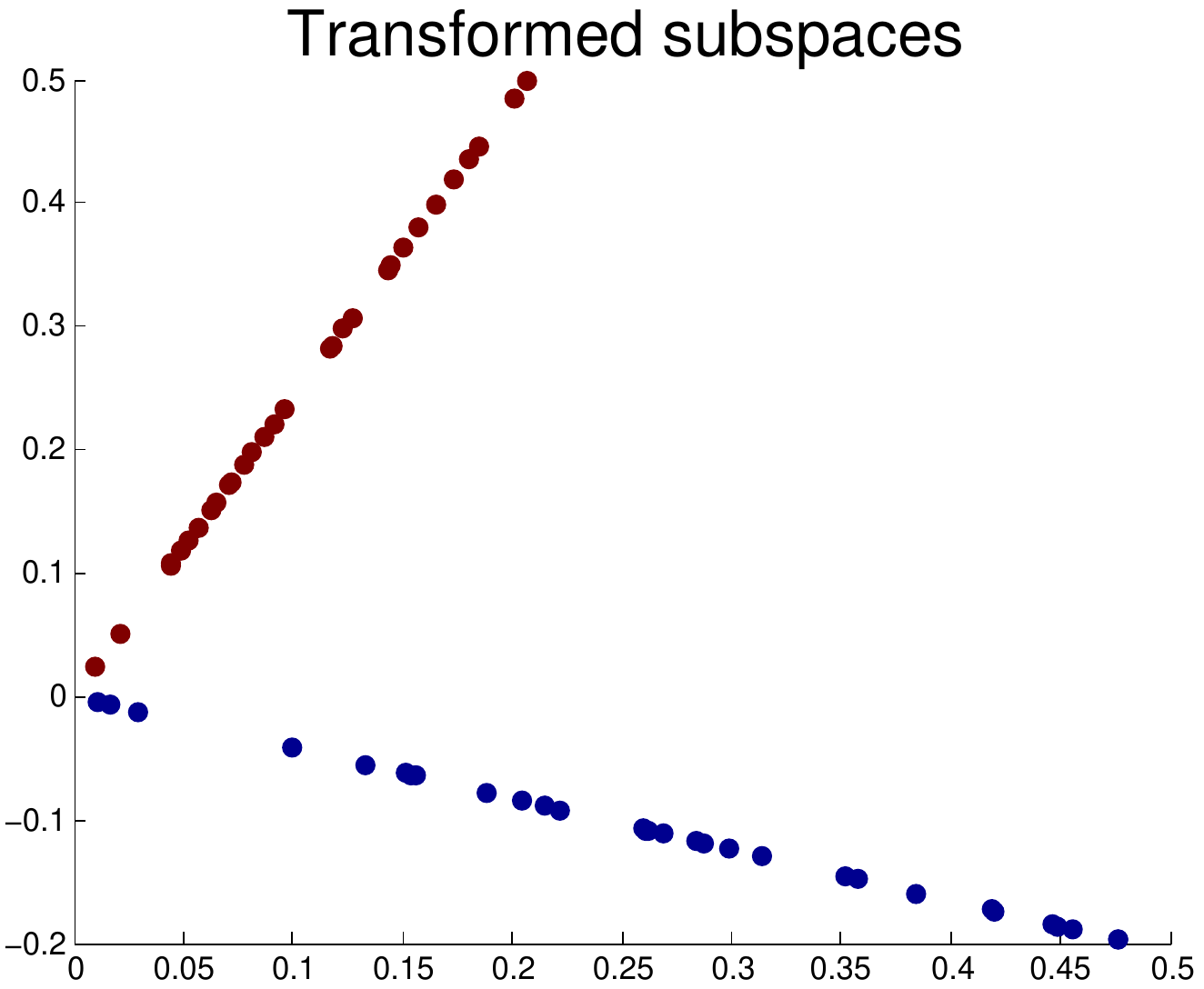}} \\
 \subfloat[][$\begin{bmatrix} \theta_{AB}=0.79,& \theta_{AC}=0.79, &\theta_{BC}=1.05 \\ \epsilon_{A}=0.0141, &\epsilon_{B}=0.0131, &\epsilon_{C}=0.0148 \end{bmatrix}$,\\
 $~~~~~~\begin{matrix} |\mathbf{A}|_*=4.06, & |\mathbf{B}|_*=4.08, & |\mathbf{C}|_*= 4.16\end{matrix}$.] {\label{fig:3d45O} \includegraphics[angle=0, height=0.3\textwidth, width=.45\textwidth]{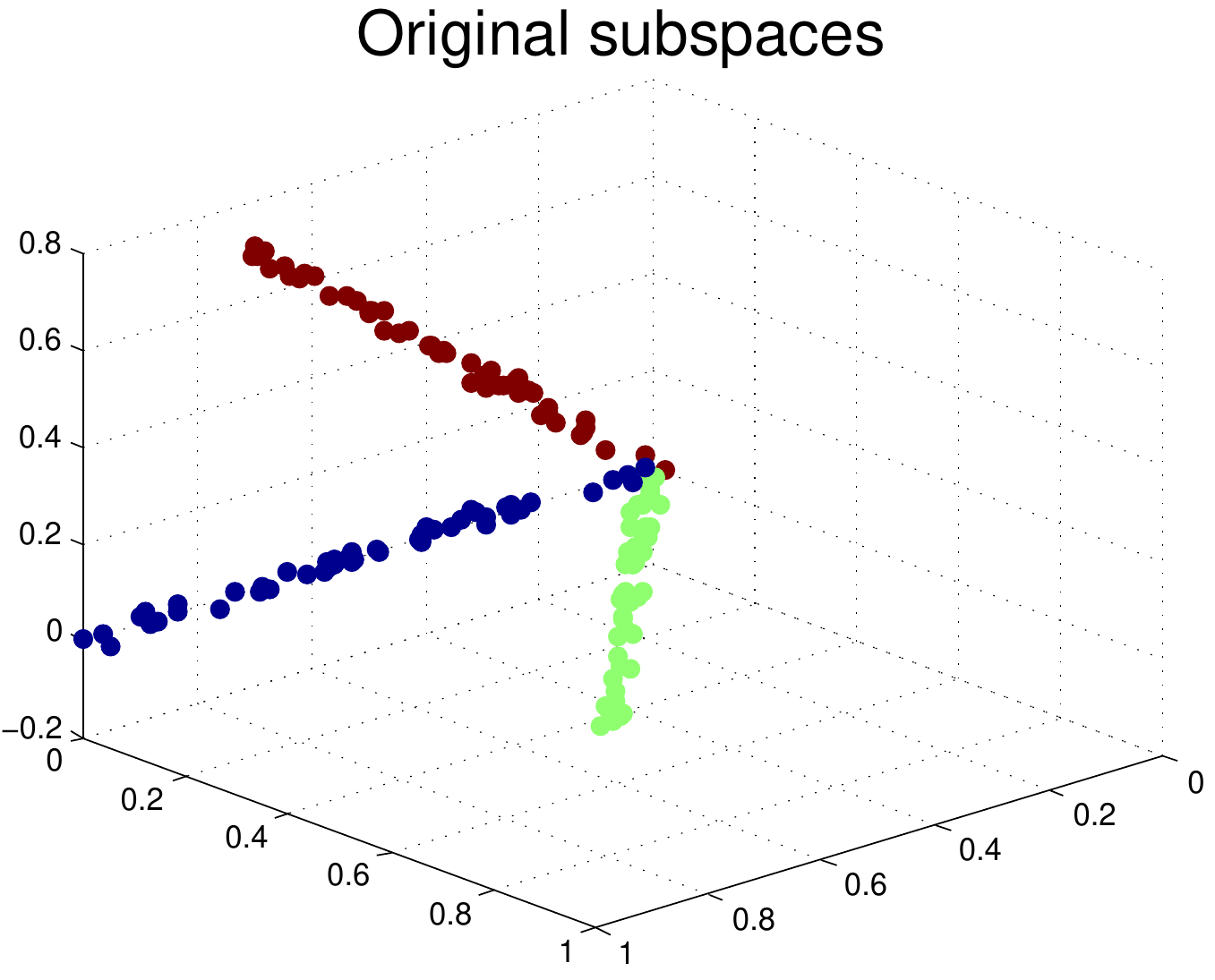}}
  \subfloat[][$\mathbf{T} = \begin{bmatrix} 0.39  & -0.16 &  -0.16 \\  -0.13  &  0.90  &  0.11 \\ -0.23  &  0.11  &  0.57 \end{bmatrix}$; \\ $~~~~~\begin{bmatrix} \theta_{AB}= 1.51, &\theta_{AC}=1.49, &\theta_{BC}=1.57 \\ \epsilon_{A}=0.0091, &\epsilon_{B}=0.0085, &\epsilon_{C}=0.0114 \end{bmatrix}$,\\
  $~~~~~\begin{matrix} |\mathbf{A}|_*=1.93, & |\mathbf{B}|_*=2.37, & |\mathbf{C}|_*= 1.20 \end{matrix}$.] {\label{fig:3d45T} \includegraphics[angle=0, height=0.3\textwidth, width=.45\textwidth]{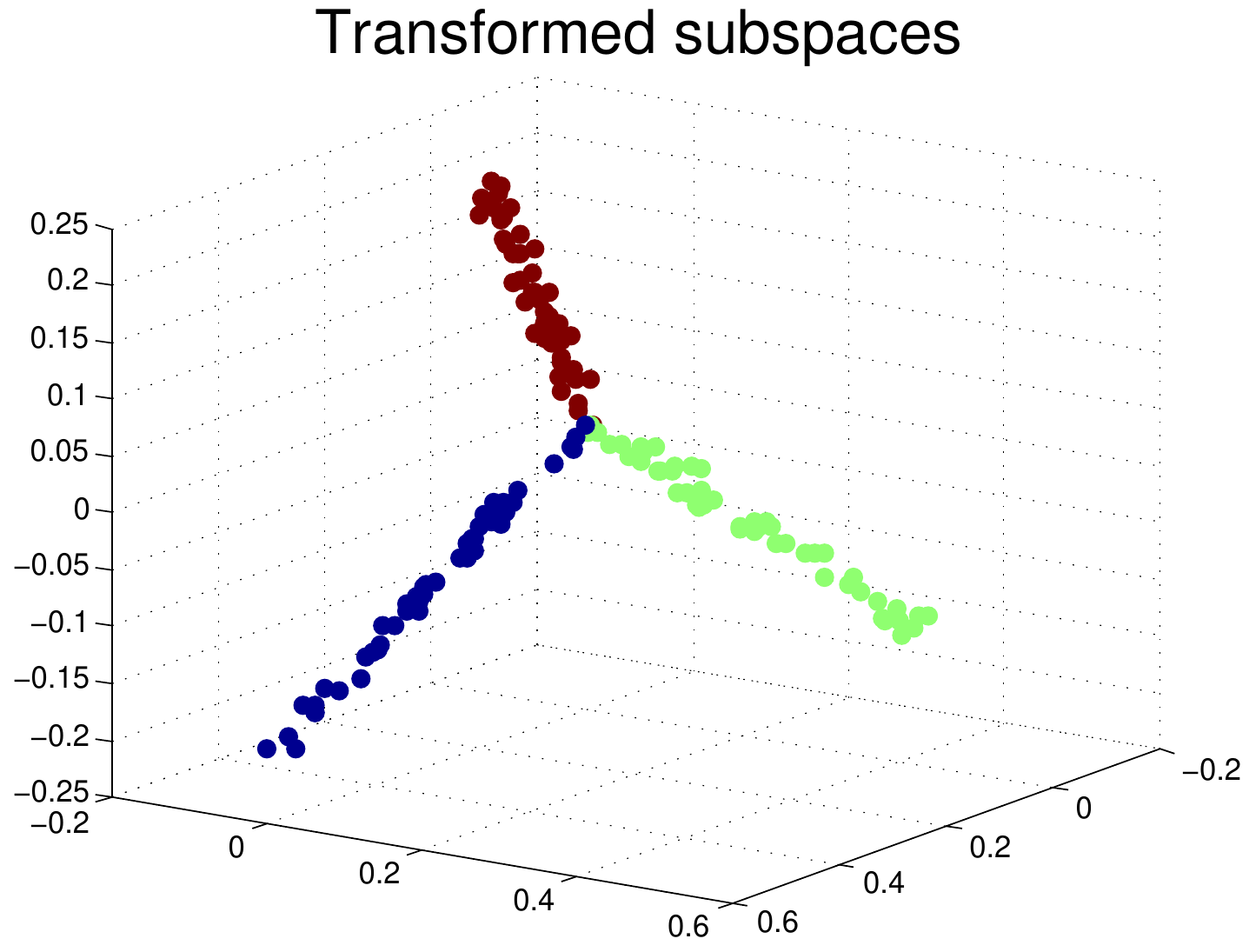}} \\
\caption{The learned transformation $\mathbf{T}$ using (\ref{nuclear_obj}) with the nuclear norm as the key criterion. Three subspaces in $\mathbb{R}^3$ are denoted as $\mathbf{A}$(red), $\mathbf{B}$(blue), $\mathbf{C}$(green). We denote the angle between subspaces $\mathbf{A}$ and $\mathbf{B}$ as $\theta_{AB}$ (and analogous for the other pairs of subspaces).
Using (\ref{nuclear_obj}), we transform $\mathbf{A}$, $\mathbf{B}$, $\mathbf{C}$ in (a),(c),(e) to (b),(d),(f) respectively  (in the first row the subspace $C$ is empty, being this basically a two dimensional example).
Data points in (e) are associated with random noises $\sim \mathcal{N}(0,0.01)$.
We denote the root mean square deviation of points in $\mathbf{A}$ from the true subspace as $\epsilon_{A}$  (and analogous for the other subspaces).
We observe that the learned transformation $\mathbf{T}$ maximizes the distance between every pair of subspaces towards $\frac{\pi}{2}$, and reduces the deviation of points from the true subspace when noise is present, note how the individual subspaces nuclear norm is significantly reduced.
{ Note that, in (c) and (d), we have the same rank values  $rank(\mathbf{A})=1, rank(\mathbf{B})=1, rank(\mathbf{[A,B]})= 2$, but different nuclear norm values, manifesting the improved between-subspaces separation.}
}
\label{fig:synthetic}
\end{figure*}

Let $||\mathbf{A}||_*$ denote the nuclear norm of the matrix  $\mathbf{A}$, i.e., the sum of the singular values of  $\mathbf{A}$. The nuclear norm $||\mathbf{A}||_*$ is the convex envelop of $rank(\mathbf{A})$ over the unit ball of matrices \cite{rank-min}. As the nuclear norm  can be optimized efficiently, it is often adopted as the best convex approximation of the rank function in the literature on rank optimization (see, e.g., \cite{rpca} and \cite{Recht}).

One factor that fundamentally affects the performance of subspace clustering and classification algorithms is the distance between subspaces.
An important notion to quantify the distance (separation) between two subspaces $\mathcal{S}_i$ and $\mathcal{S}_j$ is the smallest principal angle $\theta_{ij}$ (\cite{prin_angle}, \cite{SSC}), which is defined as
\begin{align} \label{pangle}
\theta_{ij} = \underset{\mathbf{u} \in \mathcal{S}_i, \mathbf{v} \in \mathcal{S}_j} \min \arccos
\frac{\mathbf{u}'\mathbf{v}}{||\mathbf{u}||_2||\mathbf{v}||_2},
\end{align}
Note that $\theta_{ij} \in [0, \frac{\pi}{2}].$

We replace the rank function in (\ref{rank_obj}) with the nuclear norm,
\begin{align} \label{nuclear_obj}
{ \underset{\mathbf{T}} \arg \min \sum_{c=1}^C ||\mathbf{T Y}_c||_* - ||\mathbf{T Y}||_*, ~~s.t. ||\mathbf{T}||_2 = \gamma. }
\end{align}
{
The normalization condition $||\mathbf{T}||_2 = \gamma$ prevents the trivial solution $\mathbf{T}=0$.
 Without loss of generality, we set $\gamma=1$ unless otherwise specified.
 However, understanding the effects of adopting a different normalization here is interesting and is the subject of future research. Throughout this paper we keep this particular form of the normalization which was already proven to lead to excellent results.
}

{
It is important to note that (\ref{nuclear_obj}) is not simply a relaxation of (\ref{rank_obj}).
Not only the replacement of the rank by the nuclear norm is critical for optimization considerations in reducing the variation within same class  subspaces, but as we show next,
 the learned transformation $\mathbf{T}$ using the objective function (\ref{nuclear_obj}) also maximizes the separation between different class  subspaces (a missing property in (\ref{rank_obj})), leading to improved clustering and classification performance.

 We start by presenting some basic norm relationships for matrices and their corresponding concatenations.
 }

\begin{theorem} \label{nuclear_ineq}
Let $\mathbf{A}$ and $\mathbf{B}$ be matrices of the same row dimensions, and $\mathbf{[A,B]}$ be the concatenation of $\mathbf{A}$ and $\mathbf{B}$,  we have
\begin{align} \nonumber
||\mathbf{[A,B]}||_*  \le ||\mathbf{A}||_* + ||\mathbf{B}||_*.
\end{align}
%with equality if the column spaces of $\mathbf{A}$ and $\mathbf{B}$ are orthogonal.
\end{theorem}

\begin{proof} See Appendix~\ref{sec:nuclear_ineq}.
\end{proof}

\begin{theorem} \label{nuclear_eq}
Let $\mathbf{A}$ and $\mathbf{B}$ be matrices of the same row dimensions, and $\mathbf{[A,B]}$ be the concatenation of $\mathbf{A}$ and $\mathbf{B}$,  we have
\begin{align} \nonumber
||\mathbf{[A,B]}||_*  = ||\mathbf{A}||_* + ||\mathbf{B}||_*.
\end{align}
when the column spaces of $\mathbf{A}$ and $\mathbf{B}$ are orthogonal.
\end{theorem}

\begin{proof} See Appendix~\ref{sec:nuclear_eq}.
\end{proof}

It is easy to see that theorems \ref{nuclear_ineq} and \ref{nuclear_eq} can be extended for the concatenation of multiple matrices.
Thus, for (\ref{nuclear_obj}), we have,
\begin{align} \label{nuclear_objval}
\sum_{c=1}^C ||\mathbf{T Y}_c||_* - ||\mathbf{T Y}||_* \ge 0.
\end{align}
Based on (\ref{nuclear_objval}) and Theorem~\ref{nuclear_eq}, the proposed objective function (\ref{nuclear_obj}) reaches the minimum $0$ if the column spaces of every pair of matrices are orthogonal after applying the learned transformation $\mathbf{T}$; or equivalently,
(\ref{nuclear_obj}) reaches the minimum $0$ when the separation between every pair of subspaces is maximized after transformation, i.e., the smallest principal angle between subspaces equals $\frac{\pi}{2}$.  Note that such improved separation is not obtained if the rank is used in the second term in (\ref{nuclear_obj}), thereby further justifying the use of the nuclear norm instead.

We have then, both intuitively and theoretically, justified the selection of the criteria (\ref{nuclear_obj}) for learning the transform $\mathbf{T}$.
We now illustrate the properties of the learned transformation $\mathbf{T}$
using synthetic examples in Fig.~\ref{fig:synthetic} (real examples are presented in Section~\ref{sec:expr}). Here  we adopt a
projected subgradient method described in Appendix~\ref{gradesc} (though other modern nuclear norm optimization techniques could be considered, including recent real-time formulations \cite{pablo-lr}) to search for the transformation matrix T that minimizes (\ref{nuclear_obj}). As shown in Fig.~\ref{fig:synthetic}, the learned transformation $\mathbf{T}$ via (\ref{nuclear_obj}) maximizes the separation between every pair of subspaces towards $\frac{\pi}{2}$, and reduces the deviation of the data  points to the true subspace when noise is present.
{
Note that, comparing Fig.~\ref{fig:2d45O} to Fig.\ref{fig:2d45T}, the learned transformation using (\ref{nuclear_obj}) maximizes the angle between subspaces, and  the nuclear norm changes from $|\mathbf{[A,B]}|_*= 1.41$ to $|\mathbf{[A,B]}|_*= 1.95$ to make $|\mathbf{A}|_* + |\mathbf{B}|_* - |\mathbf{[A,B]}|_* \approx 0$; However, in both cases, where subspaces are independent,  $rank(\mathbf{[A,B]})= 2$,  and $rank(\mathbf{A}) + rank(\mathbf{B}) - rank(\mathbf{[A,B]})= 0$.
}

\subsection{Comparisons with other Transformations}

\begin{figure*} [ht]
\centering
 \subfloat[][200 samples per plane, \\$\theta_{AB}=0.31$,\\  $|\mathbf{A}|_*=1.91,  |\mathbf{B}|_*=1.88$.] {\label{fig:2plane_200O} \includegraphics[angle=0, height=0.3\textwidth, width=.33\textwidth]{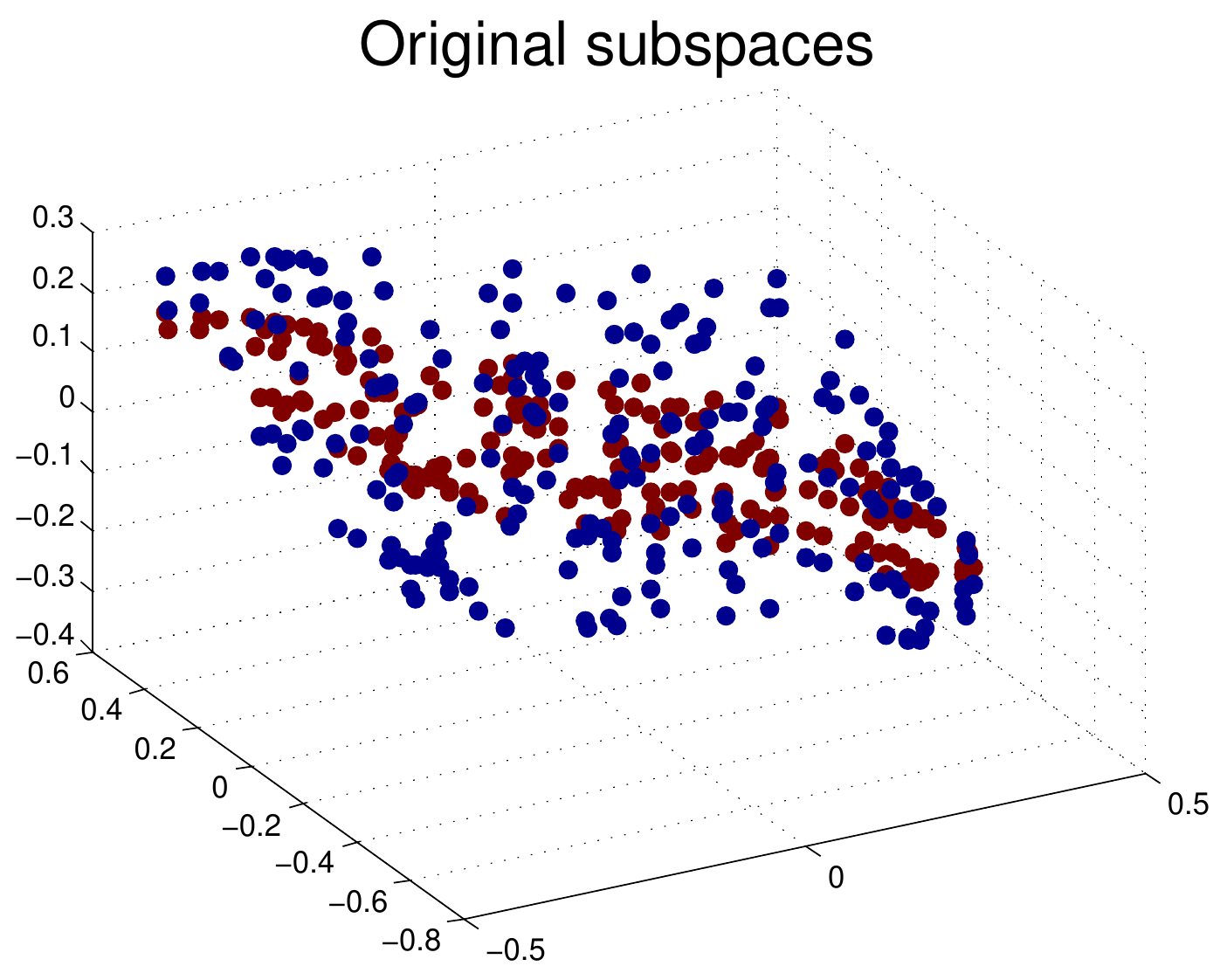}}
   \subfloat[][$\mathbf{T} = \begin{bmatrix} -0.36  & 0 &  0.10 \\  -1.4  &  0.09 &  -2.98 \\ -0.92 &  0  &  1.54 \\ 1.06 &  -1.12  &  2.63 \end{bmatrix}$;  $~~~~~\theta_{AB}= 1.41$, \\ $~~~~~\begin{matrix} |\mathbf{A}|_*=1.61, &|\mathbf{B}|_*=1.62 \end{matrix}$.] {\label{fig:2plane_200OT} \includegraphics[angle=0, height=0.3\textwidth, width=.33\textwidth]{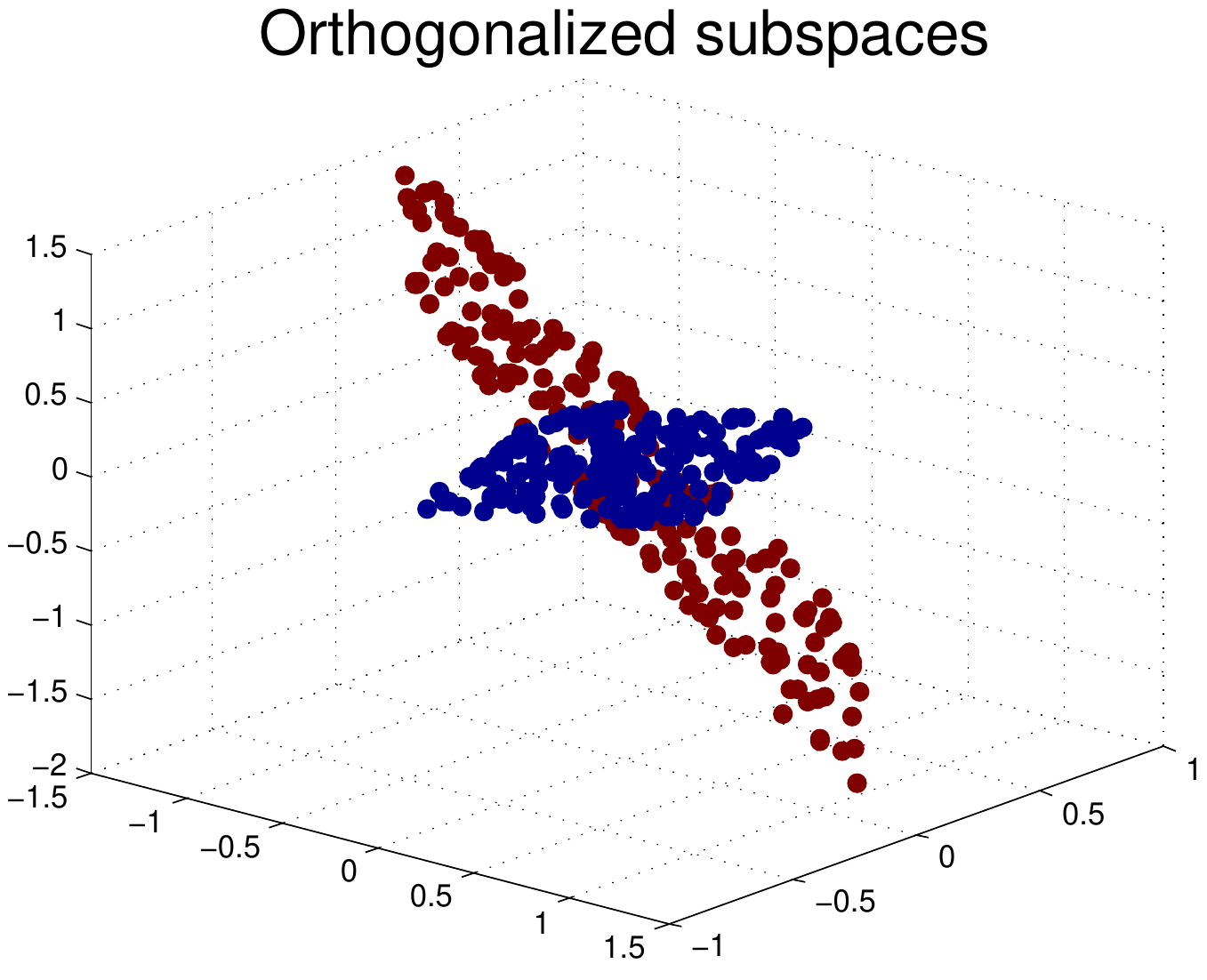}}
   \subfloat[][$\mathbf{T} = \begin{bmatrix} 0.14  & 0.03 &  0.33 \\  -0.01  &  0.14  &  -0.16 \\ 0.29 &  -0.16  &  0.86 \end{bmatrix}$;  $~~~~~\theta_{AB}= 1.55$, \\ $~~~~~\begin{matrix} |\mathbf{A}|_*=1.06, & |\mathbf{B}|_*=1.06 \end{matrix}$.] {\label{fig:2plane_200T} \includegraphics[angle=0, height=0.3\textwidth, width=.33\textwidth]{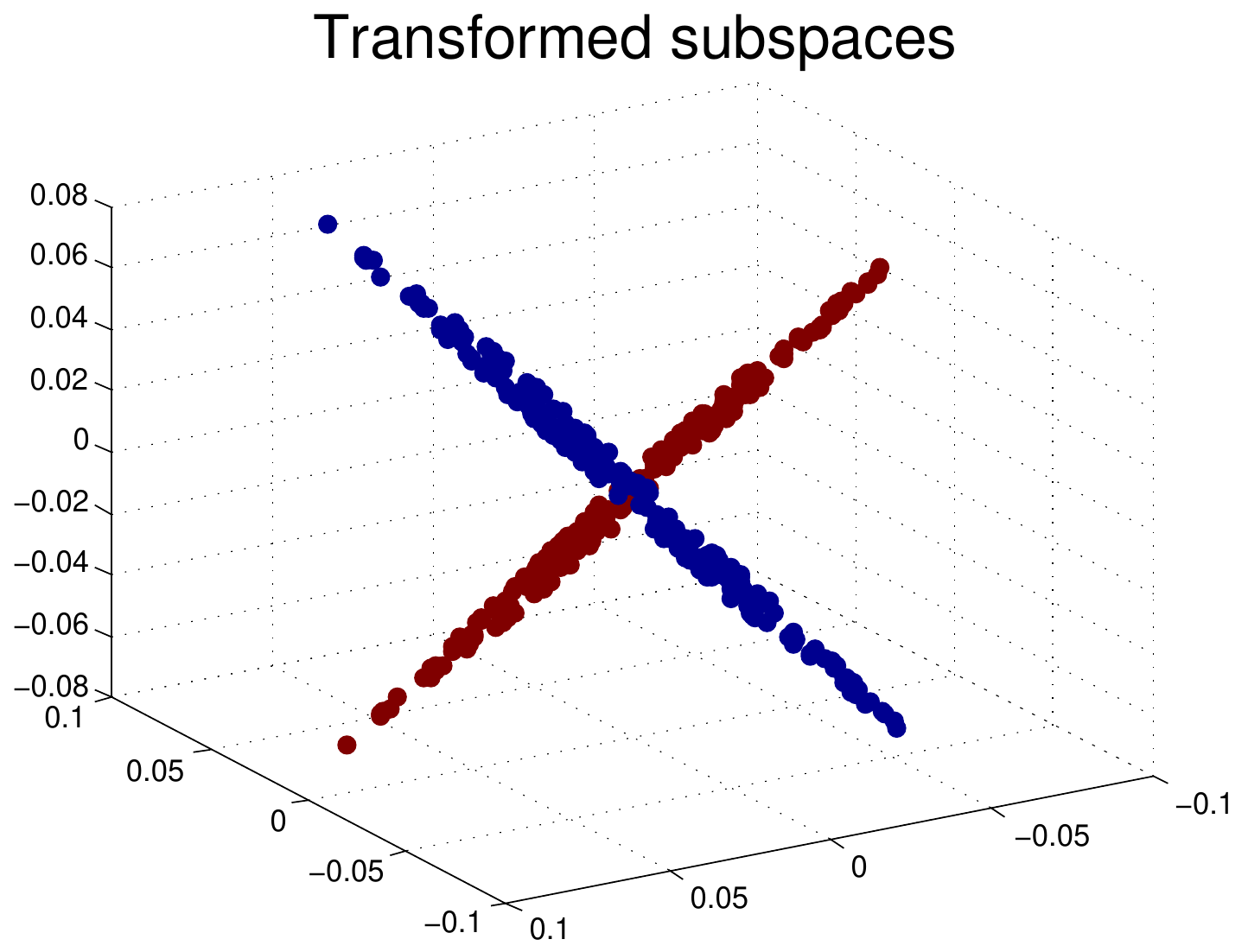}} \\
   \subfloat[][75 samples per plane, \\$\theta_{AB}=0.31$,\\  $|\mathbf{A}|_*=1.92,  |\mathbf{B}|_*=1.81$.] {\label{fig:2plane_75O} \includegraphics[angle=0, height=0.3\textwidth, width=.33\textwidth]{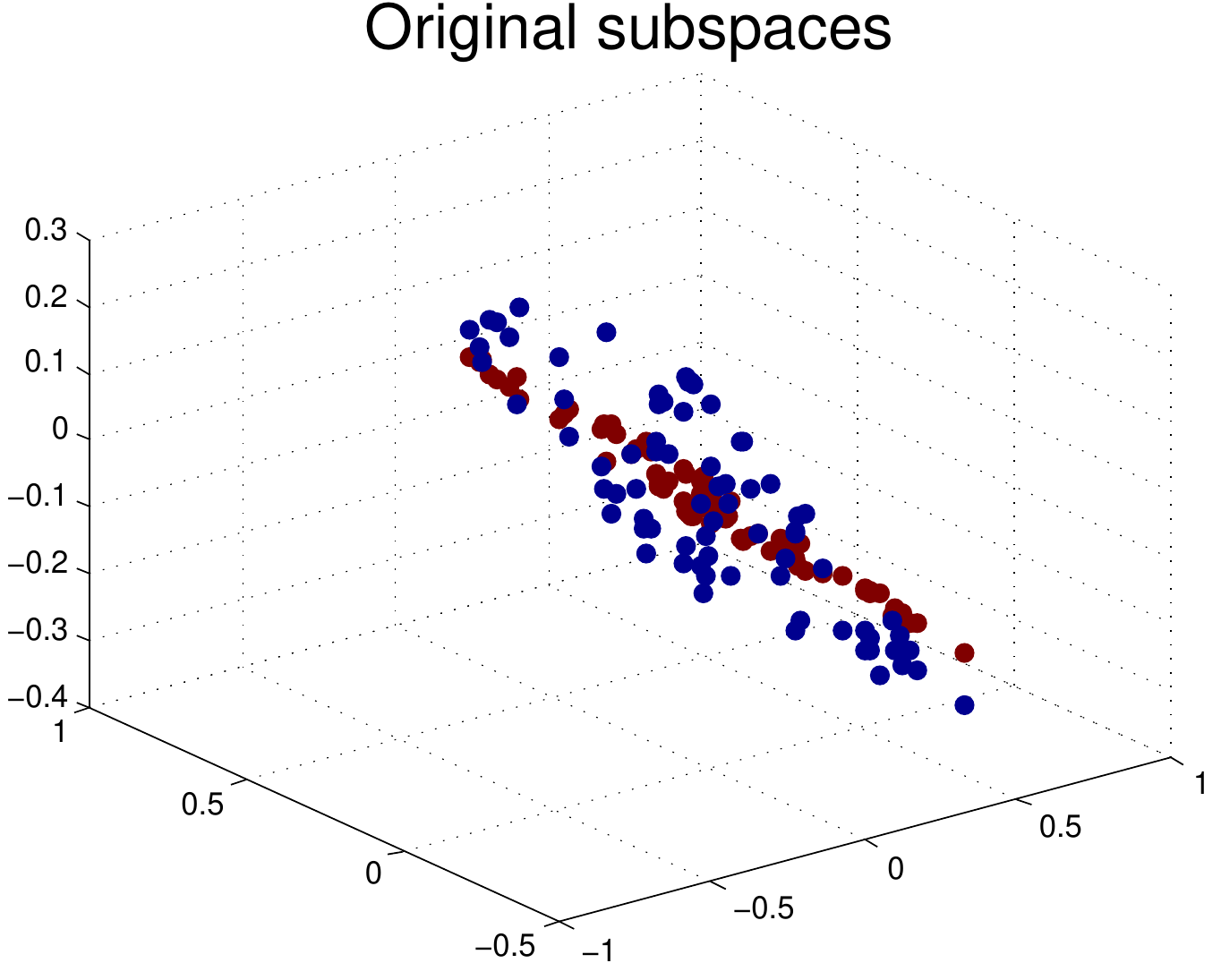}}
   \subfloat[][$\mathbf{T} = \begin{bmatrix} -1.62 & 0 & -1.94 \\  -0.50  &  0 &  -2.38 \\ 0.25 &  -0.50  &  1.81 \\ 1.12 &  -1.50  &  2.37 \end{bmatrix}$;  $~~~~~\theta_{AB}= 1.04$, \\ $~~~~~\begin{matrix} |\mathbf{A}|_*=1.75, &|\mathbf{B}|_*=1.71 \end{matrix}$.] {\label{fig:2plane_75OT} \includegraphics[angle=0, height=0.3\textwidth, width=.33\textwidth]{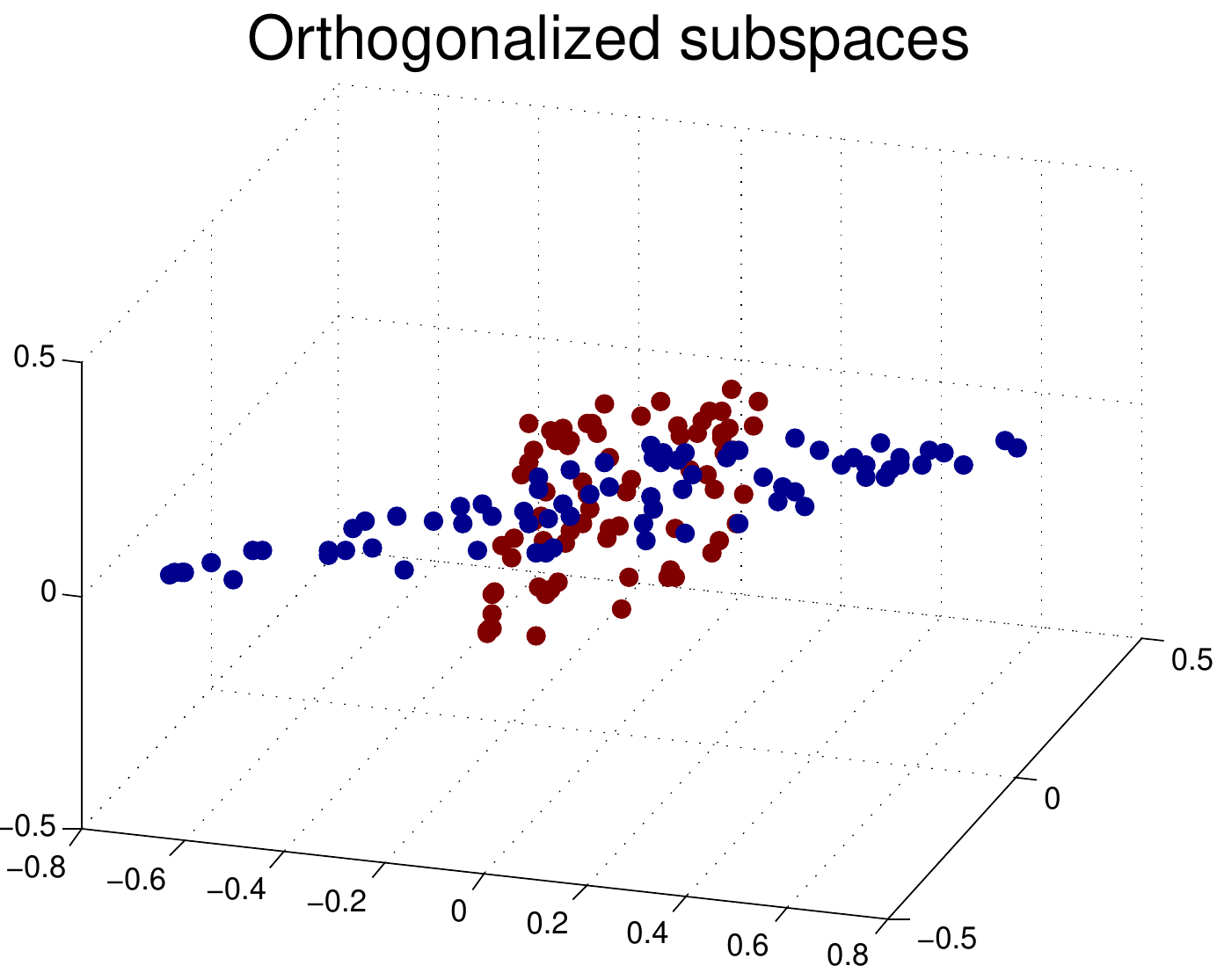}}
   \subfloat[][$\mathbf{T} = \begin{bmatrix} 0.06  & -0.12 &  0.30 \\  -0.01  &  0.13  &  -0.15 \\ 0.33 &  -0.12  &  0.86 \end{bmatrix}$;  $~~~~~\theta_{AB}= 1.55$, \\ $~~~~~\begin{matrix} |\mathbf{A}|_*=1.08, & |\mathbf{B}|_*=1.07 \end{matrix}$.] {\label{fig:2plane_75T} \includegraphics[angle=0, height=0.3\textwidth, width=.33\textwidth]{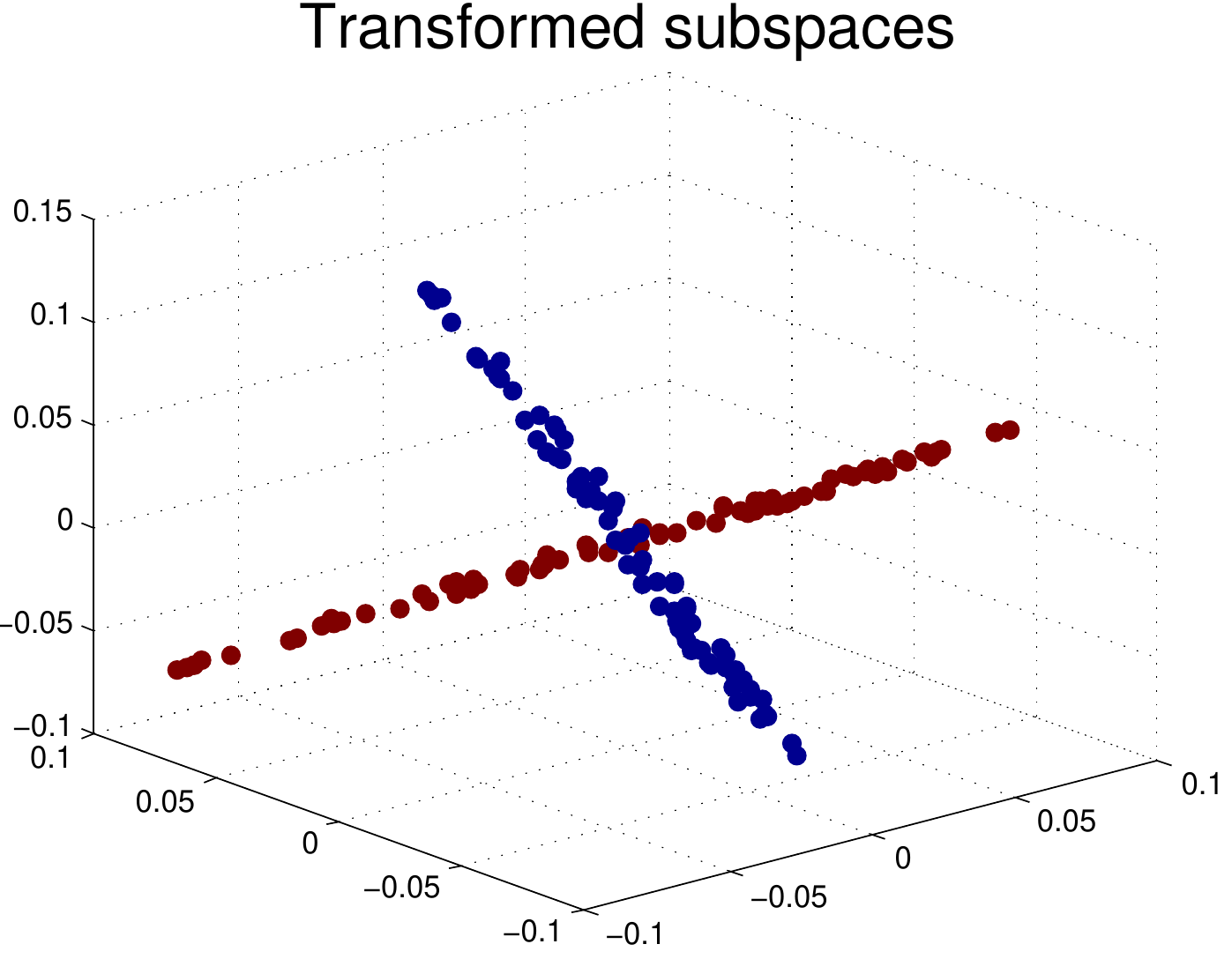}}
\caption{
{ Comparisons with the closed-form orthogonalizing transformation.
Two intersecting planes are shown in (a), and  each plane contains 200 points.
The closed-form orthogonalizing transformation significantly increase the angle between the two planes towards $\frac{\pi}{2}$ in (b).
Our leaned transformation in (c) introduces similar subspace separation, but simultaneously enables significantly reduced within subspace variation, indicated by the smaller nuclear norm values (close to 1).
The same set of experiments with 75 points per subspace are shown in the second row.
}
}
\label{fig:2plane}
\end{figure*}

\begin{figure*} [ht]
\centering
 \subfloat[][Two classes \{$\mathbf{Y}_+$,$\mathbf{Y}_-$\},
 $~\mathbf{Y}_+=\{\mathbf{A}(blue), \mathbf{B}(cyan)\}$, \\$~\mathbf{Y}_-=\{\mathbf{C}(yellow), \mathbf{D}(red)\} $, \\
 $\begin{bmatrix} \theta_{AB}=1.1,\theta_{AC}=1.1, \theta_{AD}=1.1, \\ \theta_{BC}=1.3, \theta_{BD}=1.4, \theta_{CD}=0.5 \end{bmatrix}$,\\
 $~~\begin{matrix} |\mathbf{Y}_+|_*=1.58, & |\mathbf{Y}_-|_*=1.27 \end{matrix}$. ] {\label{fig:4line_O} \includegraphics[angle=0, height=0.3\textwidth, width=.33\textwidth]{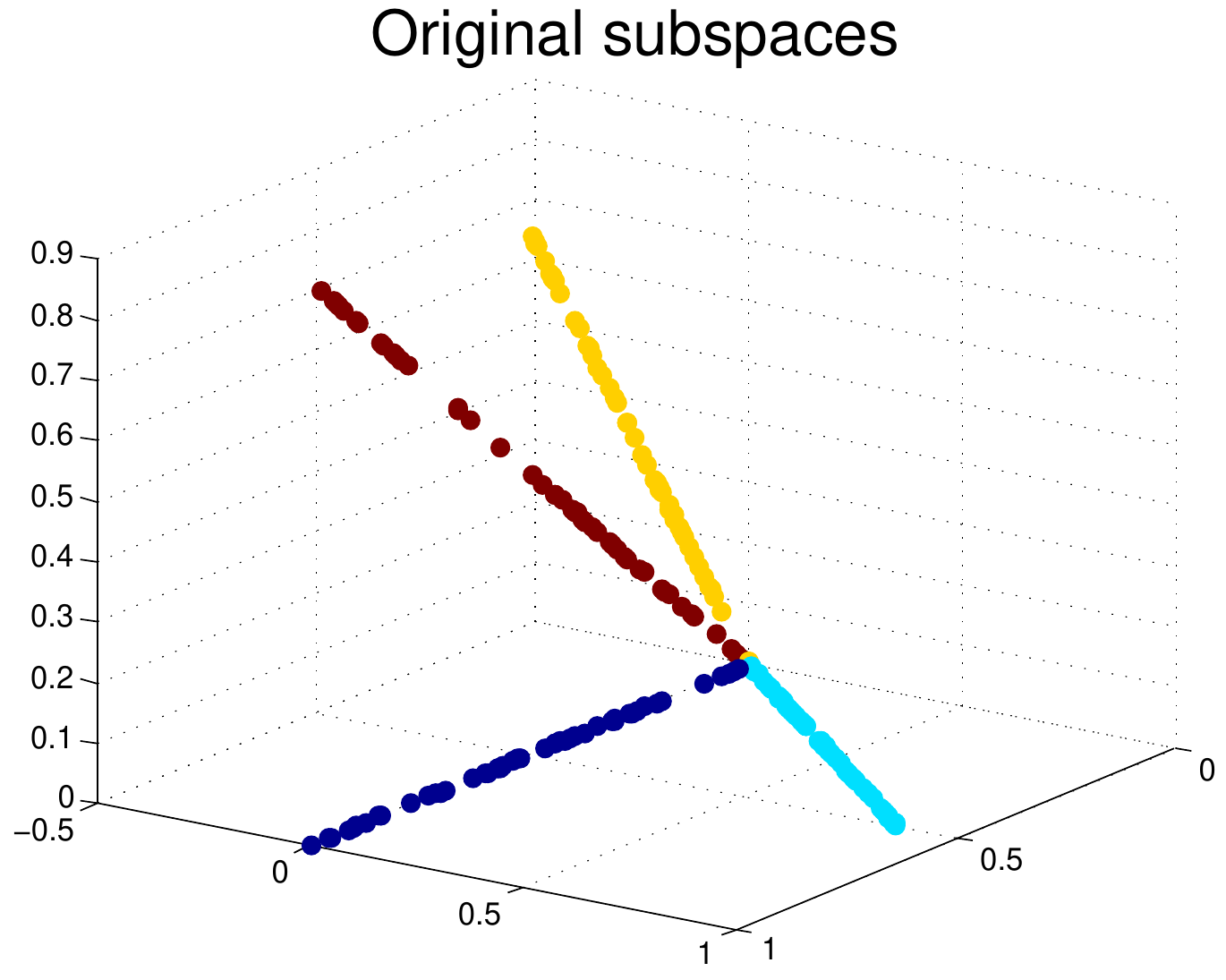}}
   \subfloat[][$\mathbf{T} = \begin{bmatrix} -3.64  & -1.95 & 5.98 \\  0.19  &   3.87 &  3.35  \end{bmatrix}$;\\  $\begin{bmatrix} \theta_{AB}=0.7,\theta_{AC}=0.78, \theta_{AD}= 0.2, \\ \theta_{BC}=1.5, \theta_{BD}=0.9, \theta_{CD}=0.57 \end{bmatrix}$,\\
   $~~\begin{matrix} |\mathbf{Y}_+|_*=1.35, & |\mathbf{Y}_-|_*=1.27 \end{matrix}$.] {\label{fig:4line_LDA} \includegraphics[angle=0, height=0.3\textwidth, width=.33\textwidth]{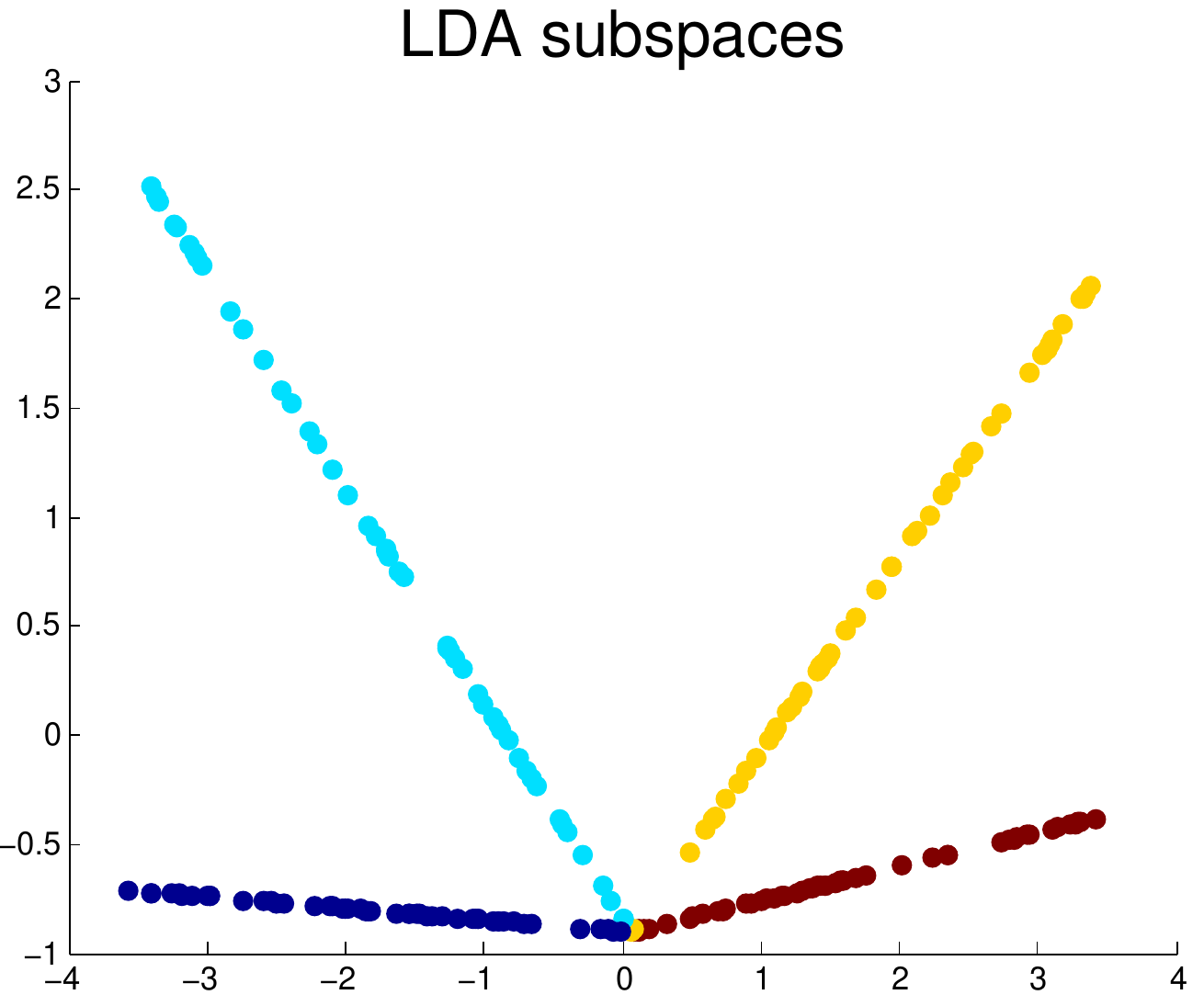}}
   \subfloat[][$\mathbf{T} = \begin{bmatrix} 1.47  & 0.26 &  -0.73 \\ 0.07  &  0.06 &  -1.62  \end{bmatrix}$;\\  $\begin{bmatrix} \theta_{AB}=0.04,\theta_{AC}=1.54, \theta_{AD}= 1.54, \\ \theta_{BC}=1.55, \theta_{BD}=1.56, \theta_{CD}=0.01 \end{bmatrix}$,\\
   $~~\begin{matrix} |\mathbf{Y}_+|_*=1.02, & |\mathbf{Y}_-|_*=1.00 \end{matrix}$.] {\label{fig:4line_T} \includegraphics[angle=0, height=0.3\textwidth, width=.33\textwidth]{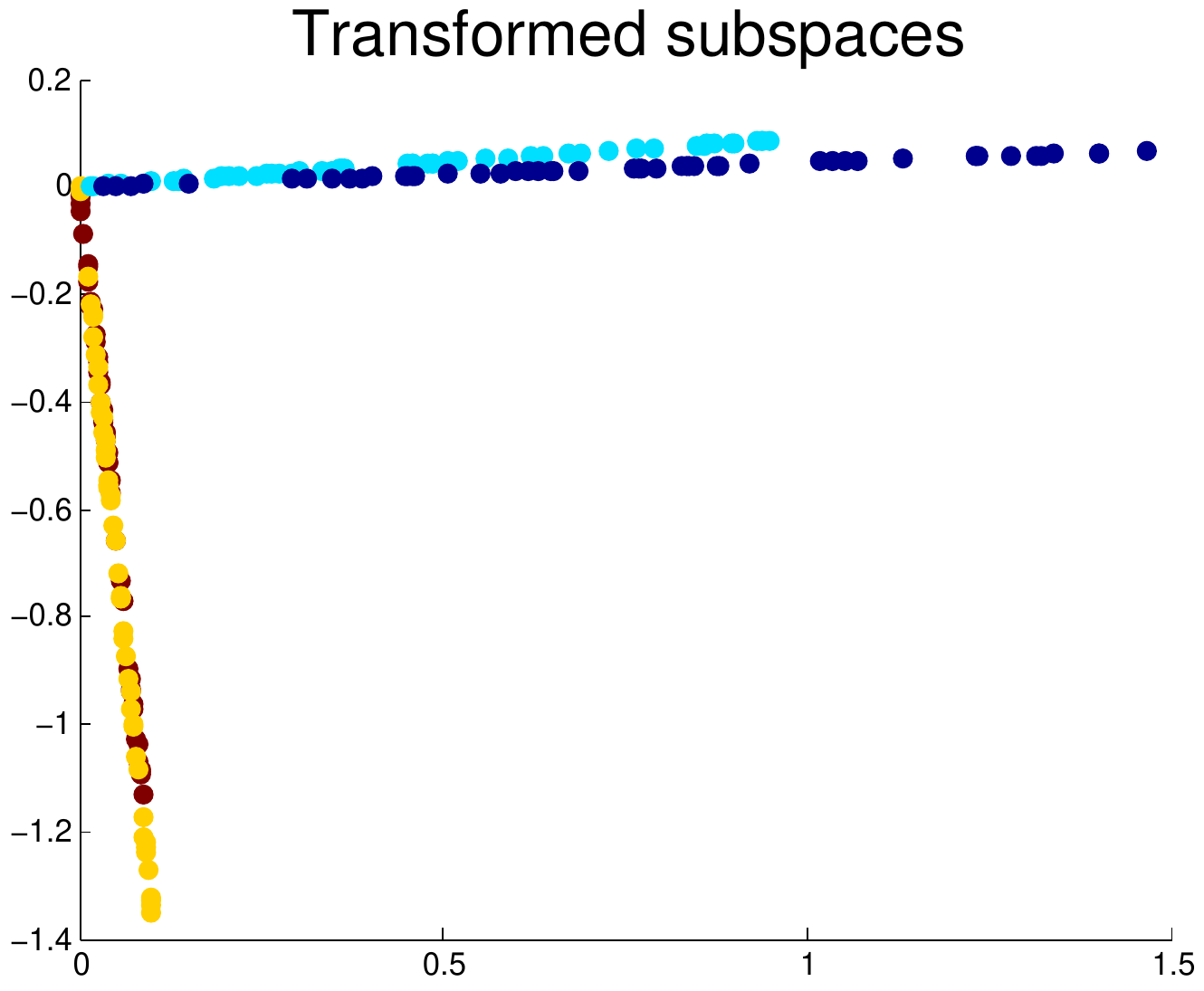}} \\
   \subfloat[][Two classes \{$\mathbf{A}$(blue),$\mathbf{B}$(red)\},\\
   $\theta_{AB}=0.31$,\\  $|\mathbf{A}|_*=1.91,  |\mathbf{B}|_*=1.88$.] {\label{fig:2plane_200O2D} \includegraphics[angle=0, height=0.3\textwidth, width=.33\textwidth]{rev1/2plane_200O.pdf}}
   \subfloat[][$\mathbf{T} = \begin{bmatrix} -0.54 & 2.60 &  -9.51 \\  0.56  &  -3.21 &  -1.02  \end{bmatrix}$;  $~~~~~\theta_{AB}= 0$, \\ $~~~~~\begin{matrix} |\mathbf{A}|_*=1.52, &|\mathbf{B}|_*=1.69 \end{matrix}$.] {\label{fig:2plane_200LDA} \includegraphics[angle=0, height=0.3\textwidth, width=.33\textwidth]{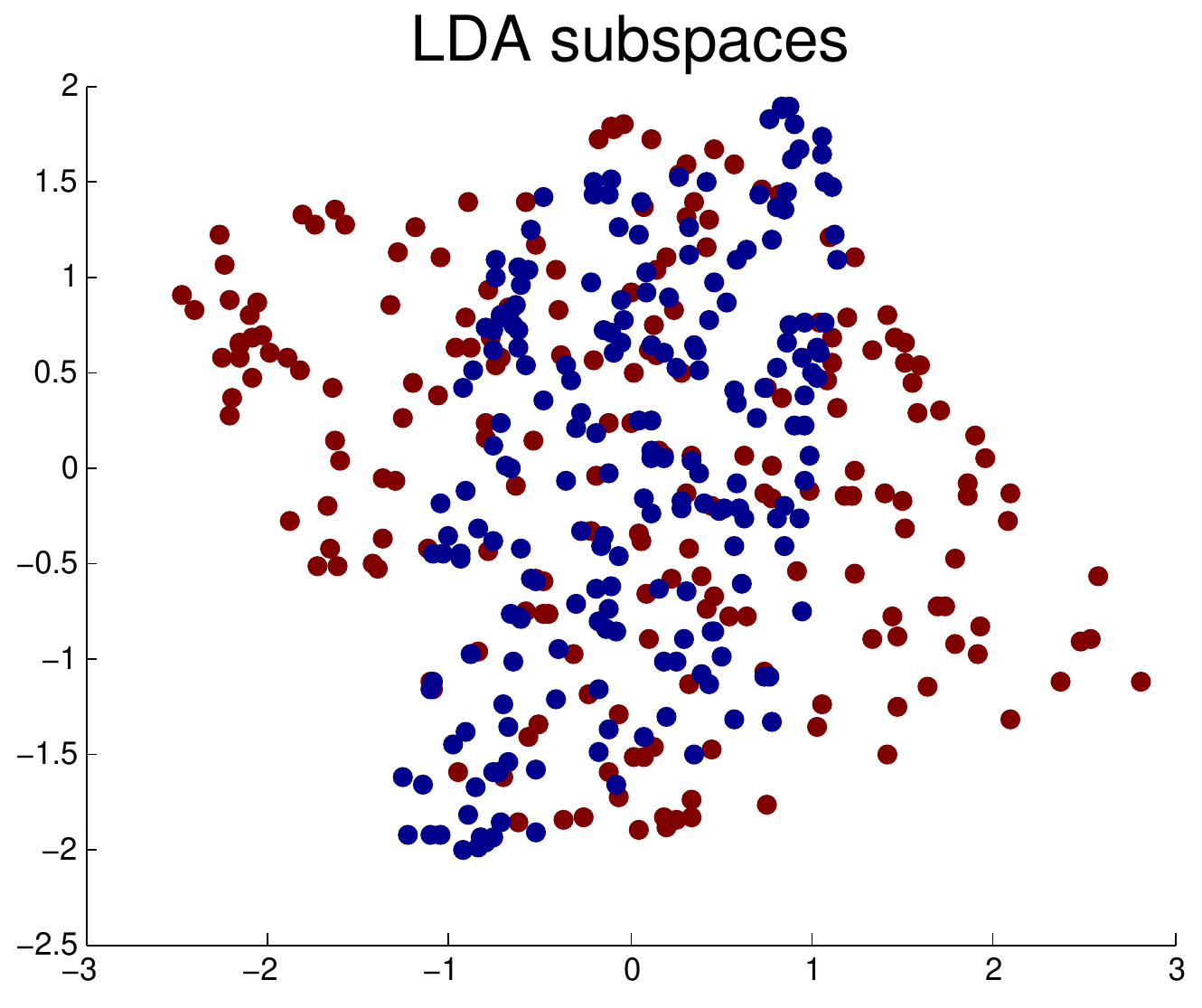}}
   \subfloat[][$\mathbf{T} = \begin{bmatrix} 0.49  & -0.11 &  1.27 \\  -0.09  & 0.29  &  -0.59 \end{bmatrix}$;  $~~~~~\theta_{AB}= 1.57$, \\ $~~~~~\begin{matrix} |\mathbf{A}|_*=1.08, & |\mathbf{B}|_*=1.03 \end{matrix}$.] {\label{fig:2plane_200T2D} \includegraphics[angle=0, height=0.3\textwidth, width=.33\textwidth]{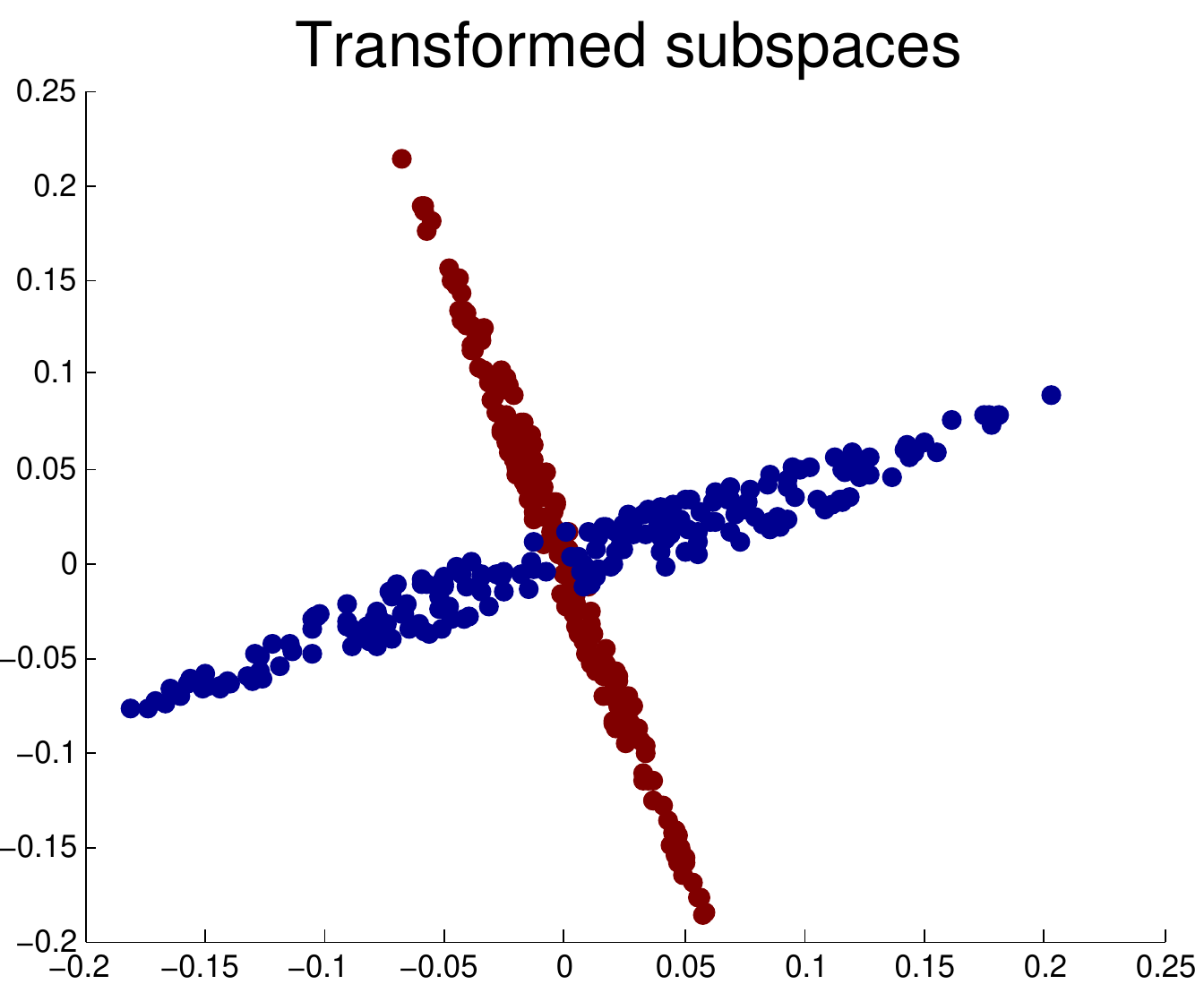}}
\caption{
{ Comparisons with the linear discriminant analysis (LDA).
Two classes $\mathbf{Y}_+$ and $\mathbf{Y}_-$ are shown in (a), each class consisting of two lines.
We notice that our learned transformation (c) shows smaller intra-class variation than LDA in (b) by merging two lines in each class, and simultaneously maximizes the angle between two classes towards $\frac{\pi}{2}$  (such two-class clustering and classification is critical for example for trees-based techniques \cite{forest_qiu}).
(d) shows an example of two non-linearly separable classes, i.e., two intersecting planes, which cannot be improved by LDA in (e). However, our learned transformation in (f) prepares data to be separable using subspace clustering.}
}
\label{fig:4line}
\end{figure*}

{
For independent subspaces, a transformation that renders them pairwise orthogonal can be obtained in a closed-form as follows: we take a basis $\mathbf{U}_c$ for the column space of $\mathbf{Y}_c$ for each subspace, form a matrix $\mathbf{U} = [ \mathbf{U}_1, ..., \mathbf{U}_C ]$, and then obtain the orthogonalizing transformation as $\mathbf{T} = (\mathbf{U}' \mathbf{U})^{-1} \mathbf{U}'$. To further elaborate the properties of our learned transformation, using synthetic examples, we compare with the closed-form orthogonalizing transformation in Fig.~\ref{fig:2plane} and with linear discriminant analysis (LDA) in Fig.~\ref{fig:4line}.

Two intersecting planes are shown in Fig.~\ref{fig:2plane_200O}.
Though subspaces here are neither independent nor disjoint,
the closed-form orthogonalizing transformation still significantly increases the angle between the two planes towards $\frac{\pi}{2}$ in Fig.~\ref{fig:2plane_200OT}  (note that the angle for the common line here is always 0).
Note also that the closed-form orthogonalizing transformation is of size $r \times d$, where $r$ is the sum of the dimension of each subspace, and we plot just the first 3 dimensions for visualization.
Comparing to the orthogonalizing transformation,
our leaned transformation in Fig.~\ref{fig:2plane_200T} introduces similar subspace separation, but enables significantly reduced within subspace variations, indicated by the decreased nuclear norm values (close to 1).
The same set of experiments with different samples per subspace are shown in the second row of Fig.~\ref{fig:2plane}.
Our formulation in (\ref{nuclear_obj}) not only maximizes the separations between the different classes subspaces, but also simultaneously reduces the variations within the same class subspaces.

Our learned transformation shares a similar methodology with LDA, i.e., minimizing intra-class variation and maximizing inter-class separation.
Two classes $\mathbf{Y}_+$ and $\mathbf{Y}_-$ are shown in Fig.~\ref{fig:4line_O},  each class consisting of two lines.
Our learned transformation in Fig.~\ref{fig:4line_T} shows smaller intra-class variation than LDA in Fig.~\ref{fig:4line_LDA} by merging two lines in each class, and simultaneously maximizes the angle between two classes towards $\frac{\pi}{2}$ (such two-class clustering and classification is critical for example for trees-based techniques \cite{forest_qiu}).
Note that we usually use LDA to reduce the data dimension to the number of classes minus 1; however, to better emphasize the distinction, we learn a $(d-1) \times d$ sized transformation matrix using both methods.
The closed-form orthogonalizing transformation discussed above also gives higher intra-class variations as $|\mathbf{Y}_+|_*=1.45$ and $|\mathbf{Y}_+|_*=1.68$.
Fig.~\ref{fig:2plane_200O2D} shows an example of two non-linearly separable classes, i.e., two intersecting planes, which cannot be improved by LDA, as shown in Fig.~\ref{fig:2plane_200LDA}. However, our learned transformation in Fig.~\ref{fig:2plane_200T2D}
prepares the data to be separable using subspace clustering.
As shown in \cite{forest_qiu}, the property demonstrated above makes our learned transformation a better learner than LDA  in a binary classification tree.

Lastly, we generated an interesting disjoint case: we consider three lines $A$, $B$ and $C$ on the same plane that intersect at the origin; the angles between them are $\theta_{AB}=0.08$, $\theta_{BC}=0.08$, and $\theta_{AC}=0.17$. As the closed-form orthogonalizing approach is valid for independent subspaces, it fails by producing $\theta_{AB}=0.005$, $\theta_{BC}=0.005$, $\theta_{BC}=0.01$.
Our framework is not limited to that, even if additional theoretical foundations are yet to come.
After our learned transformation, we have $\theta_{AB}=1.20$, $\theta_{BC}=1.20$, and $\theta_{AC}=0.75$. We can make two immediate observations: First, all angles are significantly increased within the valid range of $[0, \frac{\pi}{2}]$. Second, $\theta_{AB}+\theta_{BC}+\theta_{AC}=\pi$ (we made the same two observations while repeating the experiments with different subspace angles). Though at this point we have no clean interpretation about how those angles are balanced when pair-wise orthogonality is not possible, we strongly believe that some theories are behind the above persistent observations and we are currently exploring this.
}

\subsection{Discussions about Other Matrix Norms}

We now discuss the advantages of replacing the rank function in (\ref{rank_obj}) with the nuclear norm over other (popular) matrix norms, e.g., the induced 2-norm and the Frobenius norm.

\begin{proposition} \label{matrix_ineq}
Let $\mathbf{A}$ and $\mathbf{B}$ be matrices of the same row dimensions, and $\mathbf{[A,B]}$ be the concatenation of $\mathbf{A}$ and $\mathbf{B}$,  we have
\begin{align} \nonumber
||\mathbf{[A,B]}||_2  \le ||\mathbf{A}||_2 + ||\mathbf{B}||_2,
\end{align}
with equality if at least one of the two matrices is zero.
\end{proposition}
%\begin{proof} See Appendix~\ref{sec:matrix_ineq}.
%\end{proof}

\begin{proposition} \label{fnorm_ineq}
Let $\mathbf{A}$ and $\mathbf{B}$ be matrices of the same row dimensions, and $\mathbf{[A,B]}$ be the concatenation of $\mathbf{A}$ and $\mathbf{B}$,  we have
\begin{align} \nonumber
||\mathbf{[A,B]}||_F  \le ||\mathbf{A}||_F + ||\mathbf{B}||_F,
\end{align}
with equality if and only if at least one of the two matrices is zero.
\end{proposition}
%\begin{proof} See Appendix~\ref{sec:fnorm_ineq}.
%\end{proof}

 We choose the nuclear norm in (\ref{nuclear_obj}) for two major advantages that are not so favorable in other (popular) matrix norms:
\begin{itemize*}
  \item The nuclear norm is the best convex approximation of the rank function \cite{rank-min}, which helps to reduce the variation within the subspaces (first term in (\ref{nuclear_obj}));
  \item The objective function (\ref{nuclear_obj}) is optimized when the distance between every pair of subspaces is maximized after transformation, which helps to introduce separations between the subspaces.
  \end{itemize*}

Note that (\ref{rank_obj}), which is based on the rank, reaches the minimum when subspaces are independent but not necessarily maximally distant. Propositions \ref{matrix_ineq} and \ref{fnorm_ineq} show that the property of the nuclear norm in Theorem \ref{nuclear_ineq} holds for the induced 2-norm and the Frobenius norm. However, if we replace the rank function in (\ref{rank_obj}) with the  induced 2-norm norm or the Frobenius norm,
the objective function is minimized at the trivial solution $\mathbf{T}=0$, which is prevented by the normalization condition $||\mathbf{T}||_2=\gamma~ (\gamma>0)$.

\subsection{Online Learning Low-rank Transformations}

When data $\mathbf{Y}$ is big, we use an online algorithm to learn the low-rank transformation $\mathbf{T}$:

\begin{itemize*}
\item We first randomly partition the data set $\mathbf{Y}$ into $B$ mini-batches;
\item Using mini-batch subgradient descent, a variant of stochastic subgradient descent, the subgradient in (\ref{gradstep}) in Appendix~\ref{gradesc} is approximated by a sum of subgradients obtained from each mini-batch of samples,
%$\mathbf{T}^{(t+1)} = \mathbf{T}^{(t)} - \nu \sum_{b=1}^B \Delta \mathbf{T}_b$,
\begin{align} \label{olgradstep}
\mathbf{T}^{(t+1)} = \mathbf{T}^{(t)} - \nu \sum_{b=1}^B \Delta \mathbf{T}_b,
\end{align}
where $\Delta \mathbf{T}_b$ is obtained from (\ref{stran_sub}) Appendix~\ref{gradesc} using only data points in the $b$-th mini-batch;
\item Starting with the first mini-batch, we learn the subspace transformation $\mathbf{T}_b$ using data only in the $b$-th mini-batch, with $\mathbf{T}_{b-1}$ as warm restart.
\end{itemize*}

\subsection{Subspace Transformation with Compression}

Given data $\mathbf{Y} \subseteq \mathbb{R}^d$, so far, we considered a square linear transformation $\mathbf{T}$ of size $d \times d$. If we devise a ``fat" linear transformation $\mathbf{T}$ of size $r \times d$, where $(r<d)$, we enable dimension reduction along with transformation.
This connects the proposed framework with the literature on compressed sensing, though the goal here is to learn a ``sensing" matrix $\mathbf{T}$ for subspace classification and not for reconstruction \cite{CS1}. The nuclear-norm minimization provides a new metric for such compressed sensing design (or compressed feature learning) paradigm.  Results with this reduced dimensionality will be presented in Section~\ref{sec:expr}.

\section{Subspace Clustering using Low-rank Transformations}
\label{sec:sc}

We now move from classification, where we learned the transform from training labeled data, to clustering, where no training data is available. In particular, we address the \emph{subspace clustering} problem, meaning to partition the data set $\mathbf{Y}$ into $C$ clusters corresponding to their underlying subspaces.
We first present a general procedure to enhance the performance of existing subspace clustering methods in the literature. Then we further propose a specific fast subspace clustering technique to fully exploit the low-rank structure of (learned) transformed subspaces.

\subsection{A Learned Robust Subspace Clustering (LRSC) Framework}

In clustering tasks, the data labeling  is of course not known beforehand in practice.  The proposed algorithm, Algorithm~\ref{algorsc}, iterates between two stages:
In the first assignment stage, we obtain clusters using any subspace clustering methods, e.g., SSC (\cite{SSC}), LSA (\cite{LSA}), LBF (\cite{SLBF}).
In particular, in this paper we often use the new improved technique introduced in Section~\ref{sec:r-ssc}.
In the second update stage, based on the current clustering result, we compute the optimal subspace transformation that minimizes (\ref{nuclear_obj}). The algorithm is repeated until the clustering assignments stop changing.

{
The LRSC algorithm is a general procedure to enhance the performance of any subspace clustering methods (part of the beauty of the proposed model is that it can be applied to any such algorithm, and even beyond \cite{forest_qiu}).
We don't enforce an overall objective function at the present form for such versatility purpose.

To study convergence, one way is to adopt the subspace clustering method for the LRSC assignment step by optimizing the same LRSC update criterion (\ref{nuclear_obj}): given the cluster assignment and the transformation $\mathbf{T}$ at the current LRSC iteration, we take a point $\mathbf{y}_i$ out of its current cluster (keep the rest assignments no change) and place it into a cluster $\mathbf{Y}_c$ that minimize $\sum_{c=1}^C ||\mathbf{T Y}_c||_*$. We iteratively perform this for all points, { and then update $\mathbf{T}$ using current  $\mathbf{T}$ as warm restart.} In this way, we decrease (or keep) the overall objective function (\ref{nuclear_obj}) after each LRSC iteration.

However, the above approach is computational expensive and only allow one specific subspace clustering method. Thus, in the present implementation,  an overall objective function of the type that the LRSC algorithm optimizes can take a form such as,
 \begin{align} \label{lrsc_obj}
\underset{\mathbf{T}, \{\mathcal{S}_c\}_{c=1}^C} \arg \min
\sum_{c=1}^C \sum_{\mathbf{y}_i \in \mathcal{S}_c} || \mathbf{T} \mathbf{y}_i - P_{\mathbf{T} \mathbf{Y}_c} \mathbf{T} \mathbf{y}_i||_2^2
+\lambda [\sum_{c=1}^C ||\mathbf{T Y}_c||_* - ||\mathbf{T Y}||_*], ~~s.t. ||\mathbf{T}||_2 = \gamma,
\end{align}
where $\mathbf{Y}_c$ denotes the set of points $\mathbf{y}_i$ in the c-th subspace $\mathcal{S}_c$, and
$P_{\mathbf{T} \mathbf{Y}_c}$ denotes the projection onto $\mathbf{T} \mathbf{Y}_c$.
The LRSC iterative algorithm optimize (\ref{lrsc_obj}) through alternative minimization (with a similar form as the popular k-means, but with a different data model and with the learned transform).
While formally studying its convergence is the subject of future research, the experimental validation presented already demonstrates excellent performance, with LRSC just one of the possible applications of the proposed learned transform.

In all our experiments, we observe significant clustering error reduction in the first few LRSC iterations, and the proposed LRSC iterations enable significantly cleaner subspaces for all subspace clustering benchmark data in the literature.
The intuition behinds the observed empirical convergence is that the update step in each LRSC iteration
decreases the second term in (\ref{lrsc_obj}) to a small value close to 0 as discussed in Section~\ref{sec:form}; at the same time, the updated transformation tends to reduce the intra-subspace variation, which further reduces the first cluster deviation term in (\ref{lrsc_obj}) even with assignments derived from various subspace clustering methods.

}

\begin{algorithm}[ht]
%\SetAlFnt{\footnotesize \sf}
\footnotesize
\KwIn{A set of data points $\mathbf{Y} = \{ \mathbf{y}_i\}_{i=1}^N \subseteq \mathbb{R}^d$ in a union of $C$ subspaces.}
\KwOut{A partition of $\mathbf{Y}$ into $C$ disjoint clusters $\{ \mathbf{Y}_c\}_{c=1}^C$ based on underlying subspaces.}
\Begin{
\BlankLine
1. Initial a transformation matrix $\mathbf{T}$ as the identity matrix \;
\BlankLine
\Repeat {assignment convergence}{
\textbf{{Assignment stage:}}\\
2. Assign points in $\mathbf{TY}$  to clusters with any subspace clustering methods, e.g., the proposed R-SSC\;
\BlankLine
\textbf{{Update stage:}}\\
3. Obtain transformation $\mathbf{T}$ by minimizing (\ref{nuclear_obj}) based on the current clustering result \;
}
\BlankLine
4. Return the current clustering result $\{ \mathbf{Y}_c\}_{c=1}^C$ \;
}
\caption{Learning a robust subspace clustering (LRSC) framework.}
\label{algorsc}
\end{algorithm}

\subsection{Robust Sparse Subspace Clustering (R-SSC)}
\label{sec:r-ssc}

Though Algorithm~\ref{algorsc} can adopt any subspace clustering methods, to fully exploit the low-rank structure of the learned transformed subspaces, we further propose the following specific technique for the clustering step in the LRSC framework, called Robust Sparse Subspace Clustering (R-SSC):

\begin{enumerate*}
\item For the transformed subspaces, we first recover their low-rank representation $\mathbf{L}$ by performing a low-rank decomposition (\ref{rpca}), e.g., using RPCA (\cite{rpca}),\footnote{Note that while the learned transform $\mathbf{T}$ encourages low-rank in each sub-space, outliers might still exists. Moreover, during the iterations in Algorithm~\ref{algorsc}, the intermediate learned $\mathbf{T}$ is not yet the desired one. This justifies the incorporation of this further low-rank decomposition.}
    \begin{align} \label{rpca}
\underset{\mathbf{L}, \mathbf{S}} \arg \min ||\mathbf{L}||_* + \beta ||\mathbf{S}||_1 ~~s.t.~ \mathbf{TY} =\mathbf{L}+\mathbf{S}.
\end{align}
\item Each transformed point $\mathbf{Ty}_i$ is then sparsely decomposed over $\mathbf{L}$,
\begin{align} \label{SSC}
\underset{\mathbf{x}_i} \arg \min \|\mathbf{Ty}_i-\mathbf{L}\mathbf{x}_i\|_{2}^{2} ~~s.t.~ \|\mathbf{x}_i\|_{0}\leq K ,
\end{align}
where $K$ is a predefined sparsity value ($K > d$).
As explained in \cite{SSC}, a data point in a linear or affine subspace of dimension $d$ can be written as a linear or affine combination of $d$ or $d+1$ points in the same subspace. Thus, if we represent a point as a linear or affine combination of all other points, a sparse linear or affine combination can be obtained by choosing $d$ or $d+1$ nonzero coefficients.
\item As the optimization process for (\ref{SSC}) is computationally demanding, we further simplify (\ref{SSC}) using Local Linear Embedding (\cite{LLE}, \cite{llc}). Each transformed point $\mathbf{Ty}_i$ is represented using its $K$ Nearest Neighbors (NN) in $\mathbf{L}$, which are denoted as $\mathbf{L}_i$,
 \begin{align} \label{LLE}
 \underset{\mathbf{x}_i} \arg \min \|\mathbf{Ty}_i-\mathbf{L}_i\mathbf{x}_i\|_{2}^{2} ~~s.t.~ \|\mathbf{x}_i\|_{1} = 1 .
\end{align}
Let $\mathbf{\bar{L}}_i = \mathbf{L}_i- \mathbf{1}\mathbf{Ty}_i^T$. $\mathbf{x}_i$ can then  be efficiently obtained in closed form,
\[
\mathbf{x}_i = \mathbf{\bar{L}}_i \mathbf{\bar{L}}_i^T \setminus \mathbf{1} ,
\]
where $\mathbf{x} = \mathbf{A} \setminus \mathbf{B}$ solves the system of linear equations $\mathbf{A}\mathbf{x} = \mathbf{B}$.
As suggested in \cite{LLE}, if the correlation matrix $\mathbf{\bar{L}}_i \mathbf{\bar{L}}_i^T$ is nearly singular, it can be conditioned by adding a small multiple of the identity matrix. From experiments, we observe this simplification step dramatically reduces the running time, without sacrificing the accuracy.
\item Given the sparse representation $\mathbf{x}_i$ of each transformed data point $\mathbf{Ty}_i$, we denote the sparse representation matrix as $\mathbf{X} = [\mathbf{x}_1 \ldots \mathbf{x}_N]$. It is noted that $\mathbf{x}_i$ is written as an $N$-sized vector with no more than $K<<N$ non-zero values ($N$ being the total number of data points).
    The pairwise affinity matrix is now defined as
    $
    \mathbf{W} = |\mathbf{X}|+|\mathbf{X}^T|,
    $
and the subspace clustering is obtained using spectral clustering (\cite{spectral}).
\end{enumerate*}
Based on experimental results presented in Section~\ref{sec:expr}, the proposed R-SSC outperforms  state-of-the-art subspace clustering techniques, in both accuracy and running time, e.g., about 500 times faster than the original SSC using the implementation provided in \cite{SSC}.  Performance is further enhanced when R-SCC is used as an internal step of LRSC in Algorithm~\ref{algorsc}.

\section{Classification using Single or Multiple Low-rank Transformations}
\label{sec:rec}

In Section~\ref{sec:form}, learning one global transformation over all classes has been discussed, and then incorporated into a clustering framework in Section~\ref{sec:sc}.
The availability of data labels for training  enables us to consider instead learning individual class-based linear transformation.
The problem of class-based linear transformation learning can be formulated as (\ref{nuclear_obj2}).
\begin{align} \label{nuclear_obj2}
\underset{\{\mathbf{T}_c\}_{c=1}^C} \arg \min \sum_{c=1}^C [ ||\mathbf{T}_c \mathbf{Y}_c||_* - \lambda||\mathbf{T}_c \mathbf{Y}_{\neg c}||_* ],
\end{align}
where $\mathbf{T}_c \in \mathbb{R}^{d \times d}$ denotes the transformation for the c-{th} class, $\mathbf{Y}_{\neg c} = \mathbf{Y}  \setminus \mathbf{Y}_c$ denotes all data except the c-{th} class,
and $\lambda$ is a positive balance parameter.

When a global transformation matrix $\textbf{T}$ is learned, we can perform classification in the transformed space by simply considering the transformed data $\mathbf{TY}$ as the new features. For example, when a Nearest Neighbor (NN) classifier is used, a testing sample $\mathbf{y}$ uses $\mathbf{Ty}$ as the feature and searches for nearest neighbors among $\mathbf{TY}$.

To fully exploit the low-rank structure of the transformed data, we propose to perform classification through the following procedure:

\begin{itemize*}
\item For the c-th class, we first recover its low-rank representation $\mathbf{L}_c$ by performing low-rank decomposition (\ref{rpca2}), e.g., using RPCA (\cite{rpca}):\footnote{Note that this is done only once and can be considered part of the training stage. As before, this further low-rank decomposition helps to handle outliers not addressed by the learned transform.}
    \begin{align} \label{rpca2}
\underset{\mathbf{L}_c, \mathbf{S}_c} \arg \min ||\mathbf{L}_c||_* + \beta ||\mathbf{S}_c||_1 ~~s.t.~ \mathbf{TY}_c =\mathbf{L}_c+\mathbf{S}_c.
\end{align}
\item Each testing image $\mathbf{y}$ will then be assigned to the low-rank subspace $\mathbf{L}_c$ that gives the minimal reconstruction error through sparse decomposition (\ref{omp}), e.g., using OMP (\cite{omp}),
\begin{align} \label{omp}
\underset{\mathbf{x}} \arg \min \|\mathbf{Ty}-\mathbf{L}_i\mathbf{x}\|_{2}^{2} ~~s.t.~ \|\mathbf{x}\|_{0}\leq T ,
\end{align}
where $T$ is a predefined sparsity value.
\end{itemize*}
When class-based transformations $\{\mathbf{T}_c\}_{c=1}^C$ are learned, we perform recognition in a similar way. However, now we apply all the learned transforms $\mathbf{T}_c$ to each testing data point and then pick the best one using the same criterion of minimal reconstruction error through sparse decomposition (\ref{omp}).

\section{Experimental Evaluation}
\label{sec:expr}

This section first presents experimental evaluations on subspace clustering using three public datasets (standard benchmarks): the MNIST handwritten digit dataset, the Extended YaleB face dataset (\cite{yaleb}) and the Hopkins 155 database of motion segmentation.
 The MNIST dataset consists of 8-bit grayscale handwritten digit images of ``0"  through
``9" and 7000 examples for each class.
The Extended YaleB face dataset contains 38 subjects with near frontal pose under 64 lighting conditions.
All the images are resized to $16 \times 16$.
{
The classical  Hopkins 155 database of motion segmentation, which is available at \url{http://www.vision.jhu.edu/data/hopkins155}, contains 155 video sequences along with extracted feature trajectories, where 120 of the videos have two motions and 35 of the videos have three motions.
}

Subspace clustering methods compared are SSC (\cite{SSC}), LSA (\cite{LSA}), and LBF (\cite{SLBF}).
Based on the studies in \cite{SSC}, \cite{SubspaceClustering} and \cite{SLBF}, these three methods exhibit state-of-the-art  subspace clustering performance.
We adopt the LSA and SSC implementations provided in \cite{SSC} from \url{http://www.vision.jhu.edu/code/}, and the LBF implementation provided in \cite{SLBF} from
\url{http://www.ima.umn.edu/~zhang620/lbf/}.
We adopt similar setups as described in \cite{SLBF} for experiments on subspace clustering.

This section then presents experimental evaluations on classification using two public face datasets: the CMU PIE dataset (\cite{pie}) and the Extended YaleB dataset. The PIE dataset consists of 68 subjects imaged simultaneously under 13 different poses and 21 lighting conditions.
All the face images are resized to $20 \times 20$. We adopt a NN classifier unless otherwise specified.

\subsection{Subspace Clustering with Illustrative Examples}

\begin{figure*} [ht]
\centering
  \subfloat[Original subspaces for digits \{1, 2\}.] {\label{fig:12-1} \includegraphics[angle=0, height=0.16\textwidth, width=.8\textwidth]{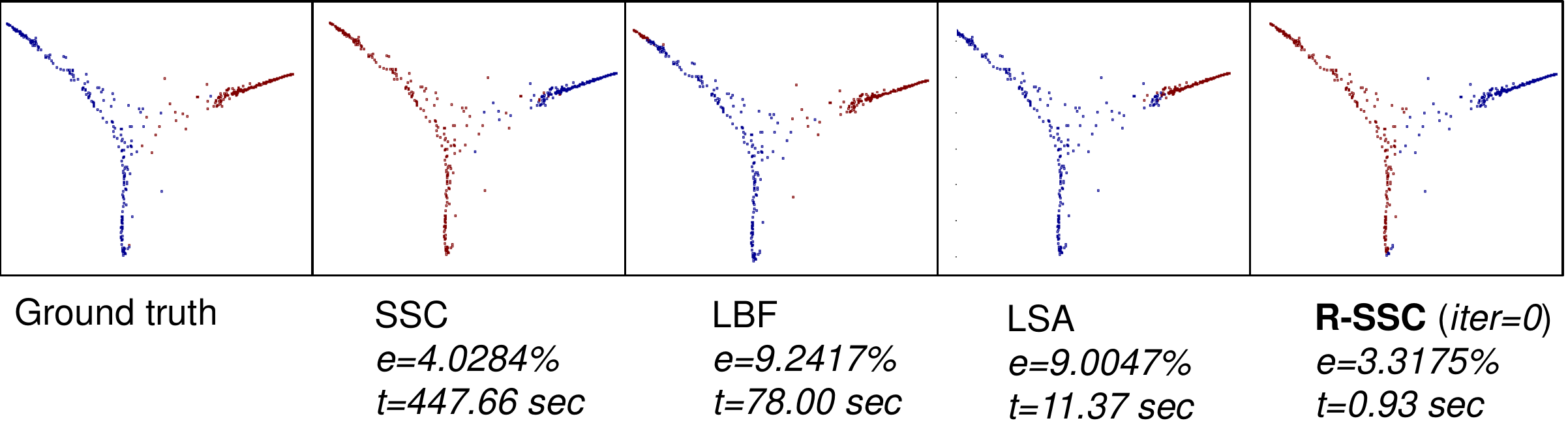}}\\
 \subfloat[Transformed subspaces for digits \{1, 2\}.] {\label{fig:12-2} \includegraphics[angle=0, height=0.21\textwidth, width=.5\textwidth]{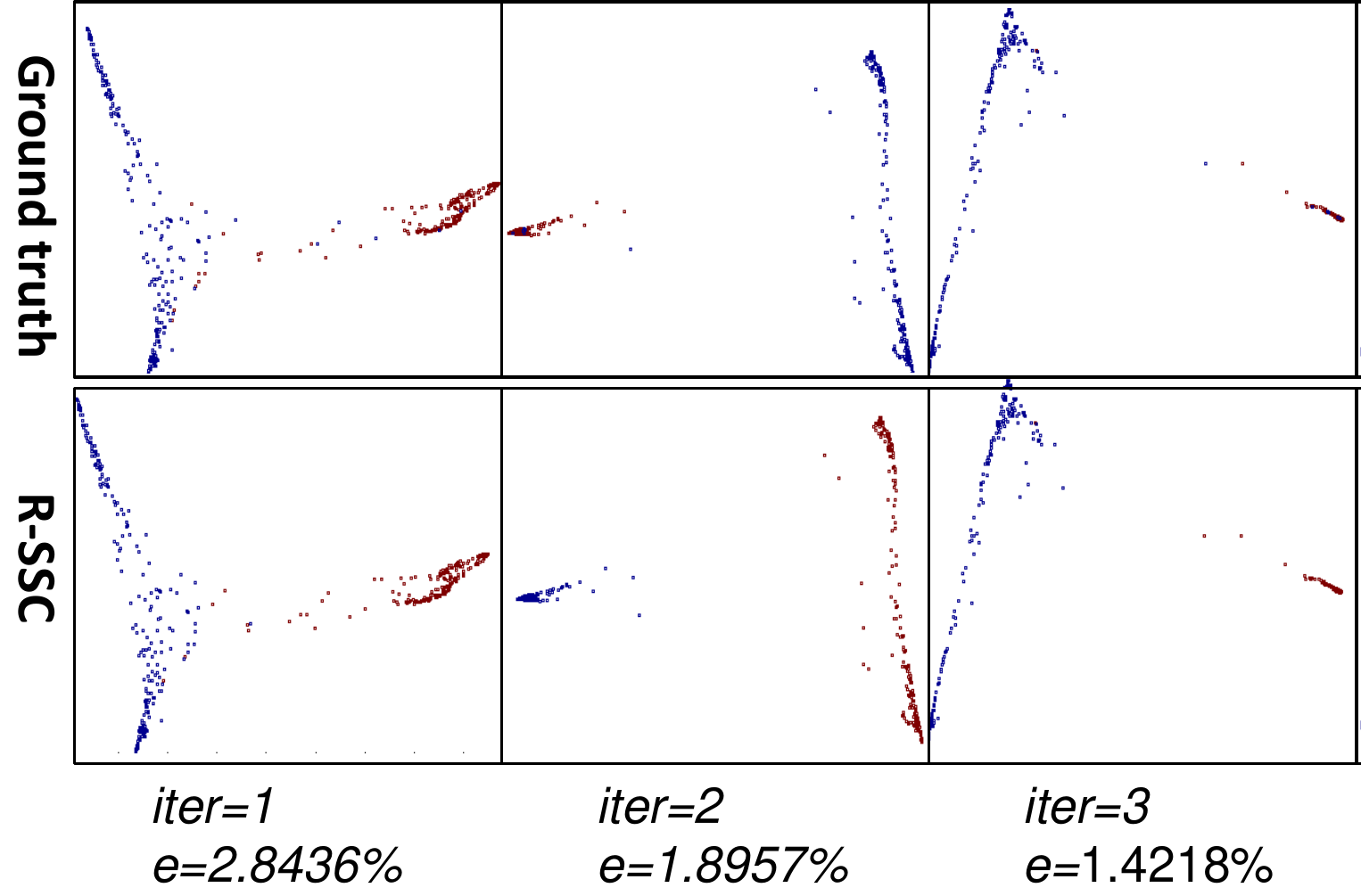} } \\
 \subfloat[Original subspaces for digits \{1, 7\}.] {\label{fig:17-1} \includegraphics[angle=0, height=0.16\textwidth, width=.8\textwidth]{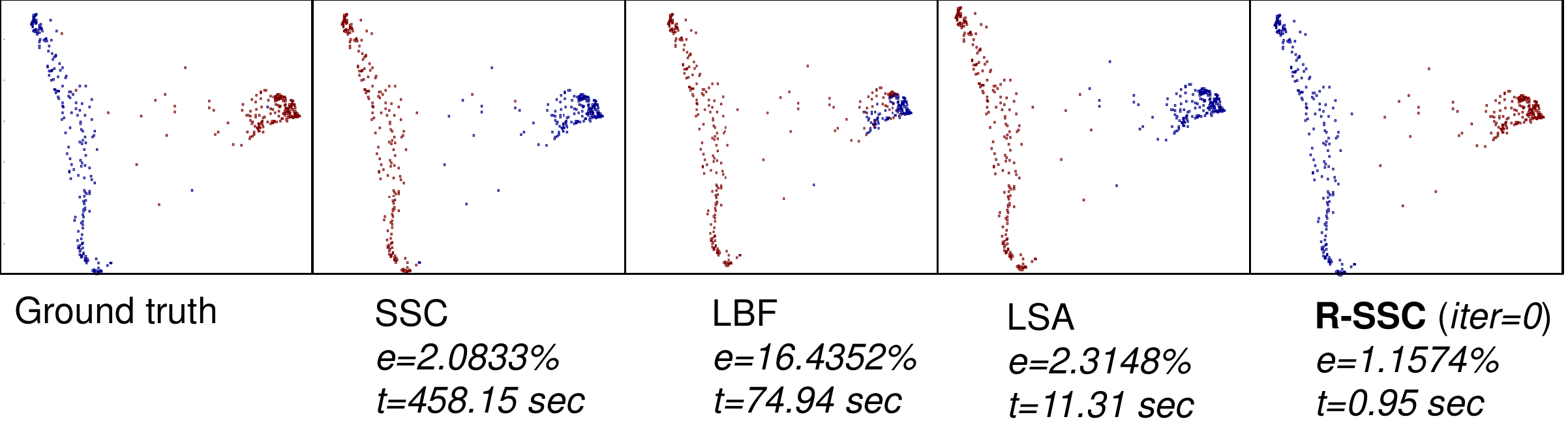}} \\
  \subfloat[Transformed subspaces for digits \{1, 7\}.] {\label{fig:17-2} \includegraphics[angle=0, height=0.21\textwidth, width=.5\textwidth]{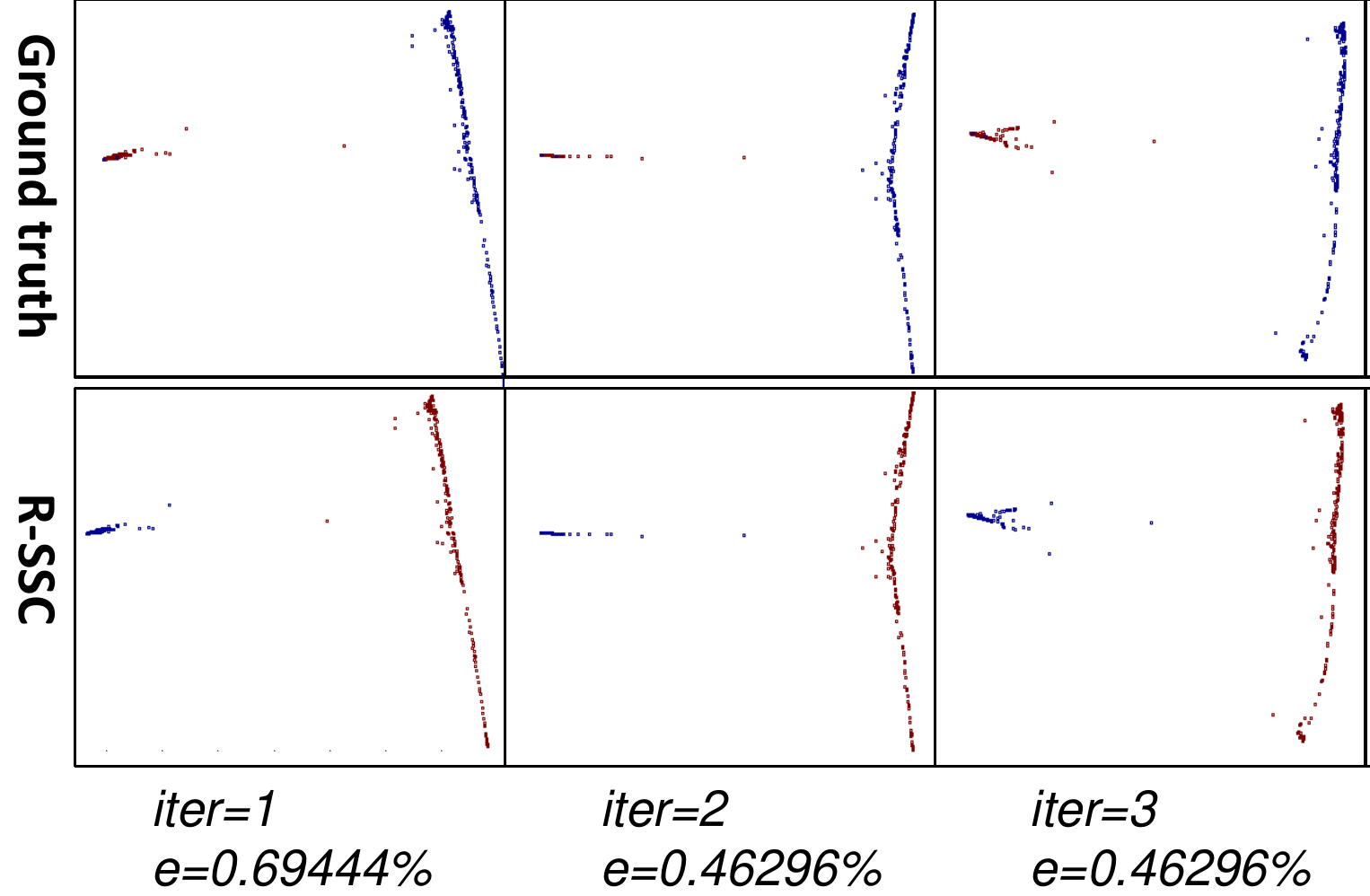}}
\caption{Misclassification rate (\emph{e})  and running time (\emph{t}) on clustering 2 digits. Methods compared are SSC \cite{SSC}, LSA \cite{LSA}, and LBF \cite{SLBF}.
For visualization, the data are plotted with the dimension reduced to 2 using Laplacian Eigenmaps \cite{eigenmap}.
Different clusters are represented by different colors and the \emph{ground truth} is plotted with the true cluster labels.
$iter$ indicates the number of LRSC iterations in Algorithm~\ref{algorsc}.
The proposed R-SSC outperforms state-of-the-art methods in terms of both clustering accuracy and running time, e.g., about 500 times faster than SSC.
The clustering performance of R-SSC is further improved using the proposed LRSC framework.
Note how the data is clearly clustered in clean subspaces in the transformed domain (best viewed zooming on screen).
}
\label{fig:2digit}
%\end{center}
\end{figure*}

\begin{figure*} [ht]
\centering
 \subfloat[Digits \{1, 2, 3\}.] {\label{fig:123} \includegraphics[angle=0, height=0.25\textwidth, width=.9\textwidth]{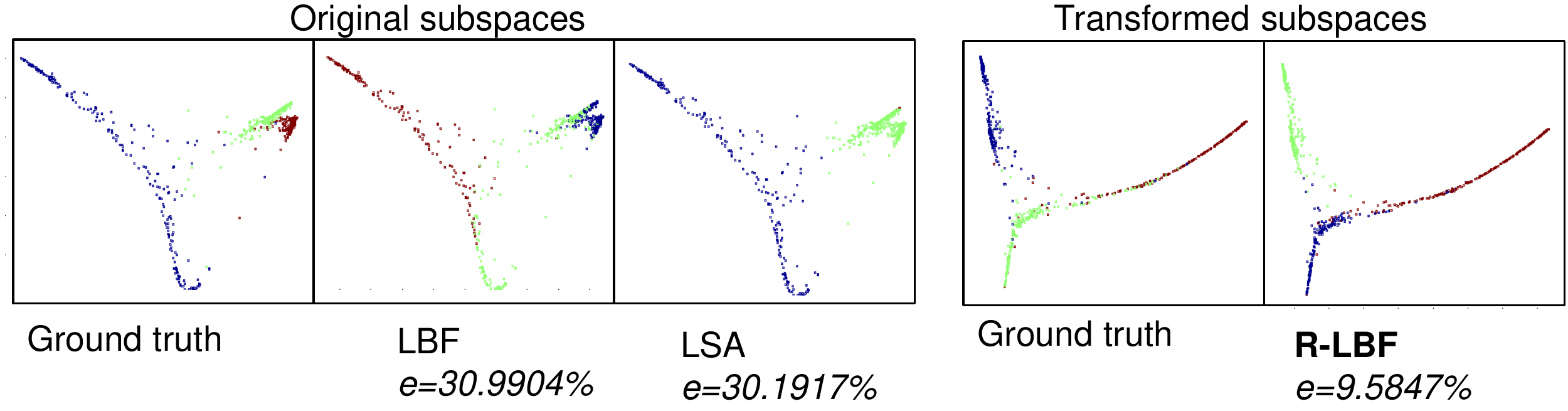} \hspace{10pt}} \\
  \subfloat[Digits \{2, 4, 8\}.] {\label{fig:248} \includegraphics[angle=0, height=0.25\textwidth, width=.92\textwidth]{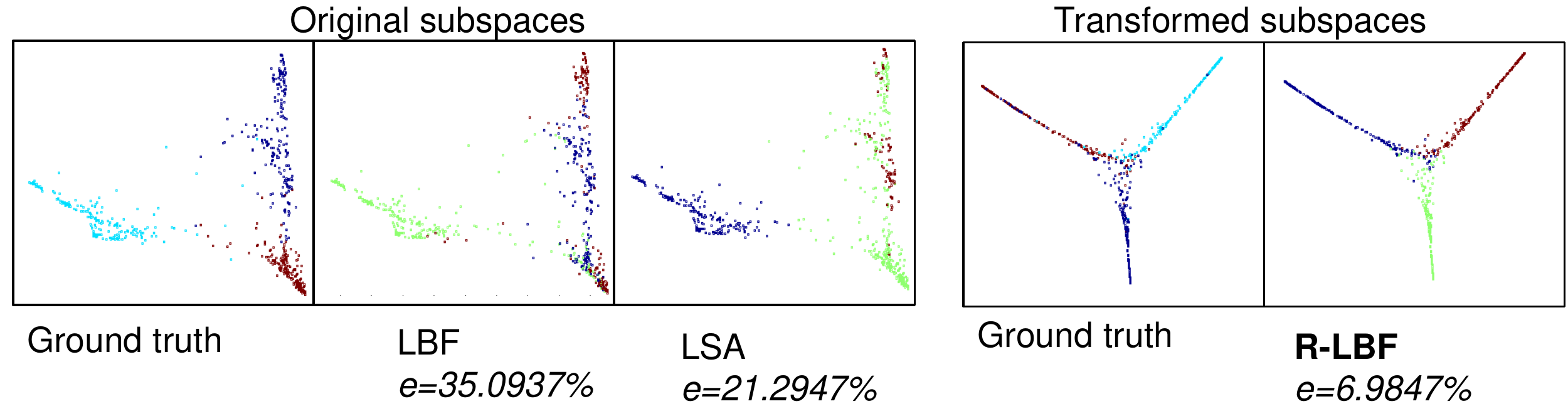}}
\caption{Misclassification rate (\emph{e}) on clustering 3 digits.
Methods compared are LSA \cite{LSA} and LBF \cite{SLBF}.
LBF is adopted in the proposed LRSC framework and denoted as {R-LBF}.
After convergence, R-LBF significantly outperforms state-of-the-art methods.}
\label{fig:3digit}
%\end{center}
\end{figure*}

For illustration purposes, we conduct the first set of experiments on a subset of the MNIST dataset.
We adopt a similar setup as described in \cite{SLBF}, using the same sets of 2 or 3 digits, and randomly choose 200 images for each digit.  We set the sparsity value $K=6$ for R-SSC, and perform $100$ iterations for the subgradient updates while learning the transformation on subspaces. The subgradient update step was $\nu =0.02$ (see Appendix~\ref{gradesc} for details on the projected subgradient optimization algorithm).

{
Unless otherwise stated, we do not perform dimension reduction, such as PCA or random projections, to preprocess the data, thereby further saving computations (please note that the learned transform can itself reduce dimensions if so desired, see Section ~\ref{sec:compress}).
In the literature, e.g., \cite{SSC}, \cite{SubspaceClustering} and \cite{SLBF}, projection to a very low dimension is usually performed to enhance the clustering performance. However, it is often not obvious how to determine the correct projection dimension for real data, and many subspace clustering methods show sensitive to the choice of the projection dimension.
This dimension reduction step is not needed in the framework here proposed.
}

Fig.~\ref{fig:2digit} shows the misclassification rate (\emph{e})  and running time (\emph{t}) on clustering subspaces of two digits.
The misclassification rate is the ratio of misclassified points to the total number of points\footnote{Meaning the ratio of points that were assigned to the wrong cluster.}.
For visualization  purposes, the data are plotted with the dimension reduced to 2 using Laplacian Eigenmaps \cite{eigenmap}.
Different clusters are represented by different colors and the ground truth is plotted using the true cluster labels.
The proposed R-SSC outperforms state-of-the-art methods, both in terms of clustering accuracy and running time.
The clustering error of R-SSC is further reduced using the proposed LRSC framework in Algorithm~\ref{algorsc} through the learned low-rank subspace transformation.
The clustering converges after about 3 LRSC iterations.
The learned transformation not only recovers a low-rank structure for
data from the same subspace, but also increases the separations between the subspaces for more accurate clustering.

Fig.~\ref{fig:3digit} shows misclassification rate (\emph{e}) on clustering subspaces of three digits.
Here we adopt LBF in our LRSC framework, denoted as Robust LBF (R-LBF),  to illustrate that the performance of existing subspace clustering methods can be enhanced using the proposed LRSC algorithm.
After convergence, R-LBF, which uses the proposed learned subspace transformation, significantly outperforms state-of-the-art methods.

{
Table~\ref{tab:mnist2} shows the misclassification rate on clustering different number of digits, $[0:c]$ denotes the subset of $c+1$ digits from digit $0$ to $c$. We randomly pick 100 samples per digit to compare the performance when a fewer number of data points per class are present.
For all cases, the proposed LRSC method significantly outperforms state-of-the-art methods.
}

\begin{table}[ht]
%\begin{center}
\centering
	\caption{{
Misclassification rate ($e\%$) on clustering different numbers of digits in the MNIST dataset,
$[0:c]$ denotes the subset of $c+1$ digits from digit $0$ to $c$. We randomly pick 100 samples per digit.
For all cases, the proposed LRSC method significantly outperforms state-of-the-art methods.
}}
{%\small
	\begin{tabular}{l|l|l|l|l|l|l|l|l}
	\hline
Subsets & [0:1] & [0:2] & [0:3] & [0:4] & [0:5] & [0:6] & [0:7] & [0:8] \\
	\hline
$C$ & 2 & 3 & 4 & 5 & 6 & 7 & 8 & 9 \\
	\hline
 \hline
LSA & 0.47 & 47.57 & 36.73 & 30.90 & 40.46 & 48.13 & 39.87 & 44.03\\
LBF & 0.47 & 23.62 & 29.19 & 51.37 & 48.99 & 53.01 & 39.87  & 38.79 \\
LRSC & \textbf{0} & \textbf{3.88} & \textbf{3.89} & \textbf{5.31} &\textbf{ 14.04} & \textbf{13.79} & \textbf{14.50}  &  \textbf{16.05} \\
\hline
	\end{tabular}
}	
	\label{tab:mnist2}
%\end{center}
\end{table}

\subsubsection{Online vs. Batch Learning}

\begin{figure*} [ht]
\centering
 \subfloat[Batch learning with various $\gamma$ values.] {\label{fig:batchconv} \includegraphics[angle=0, height=0.3\textwidth, width=.42\textwidth]{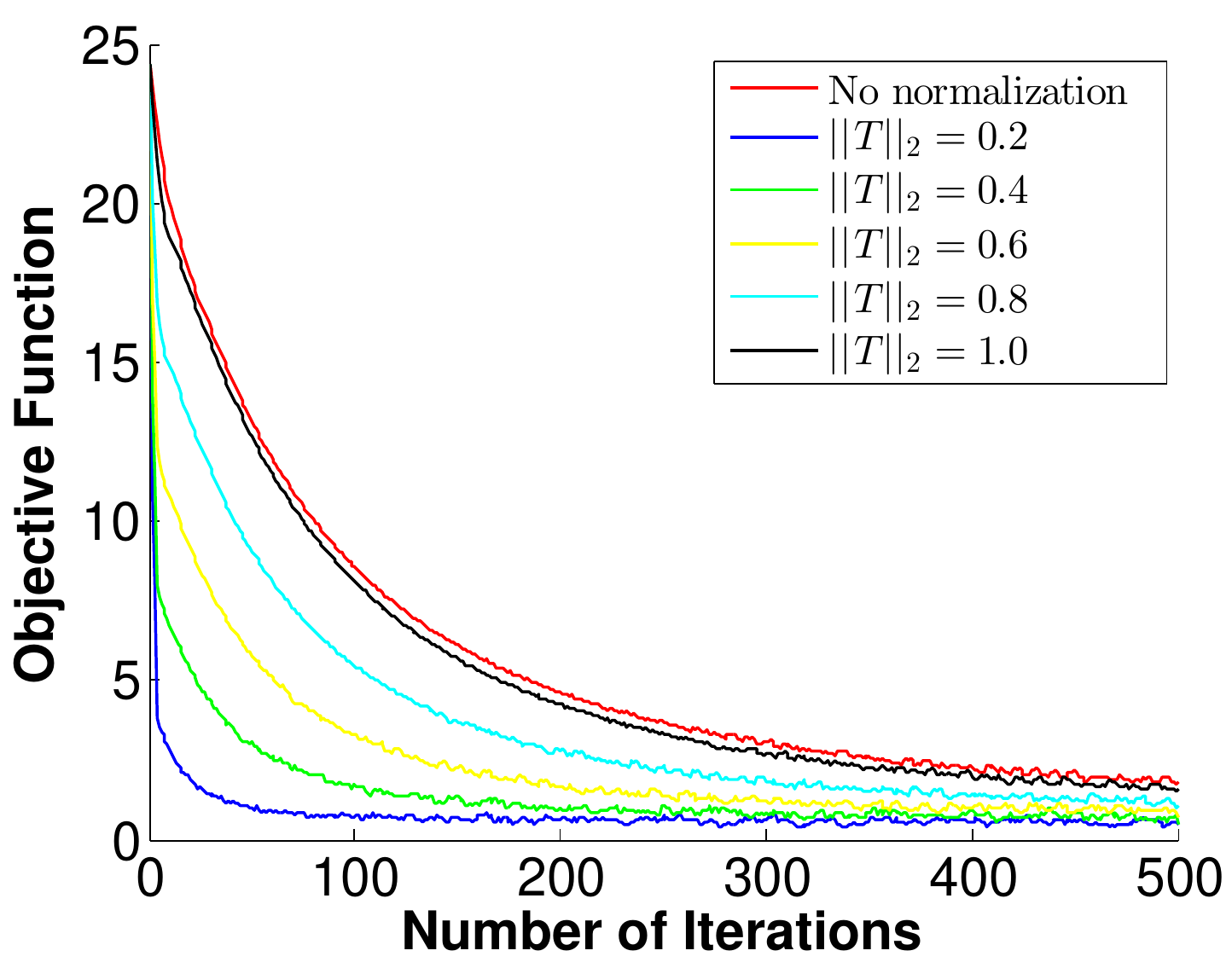} \hspace{20pt}}
  \subfloat[Online vs. batch learning ($\gamma=1$).] {\label{fig:olconv} \includegraphics[angle=0, height=0.3\textwidth, width=.4\textwidth]{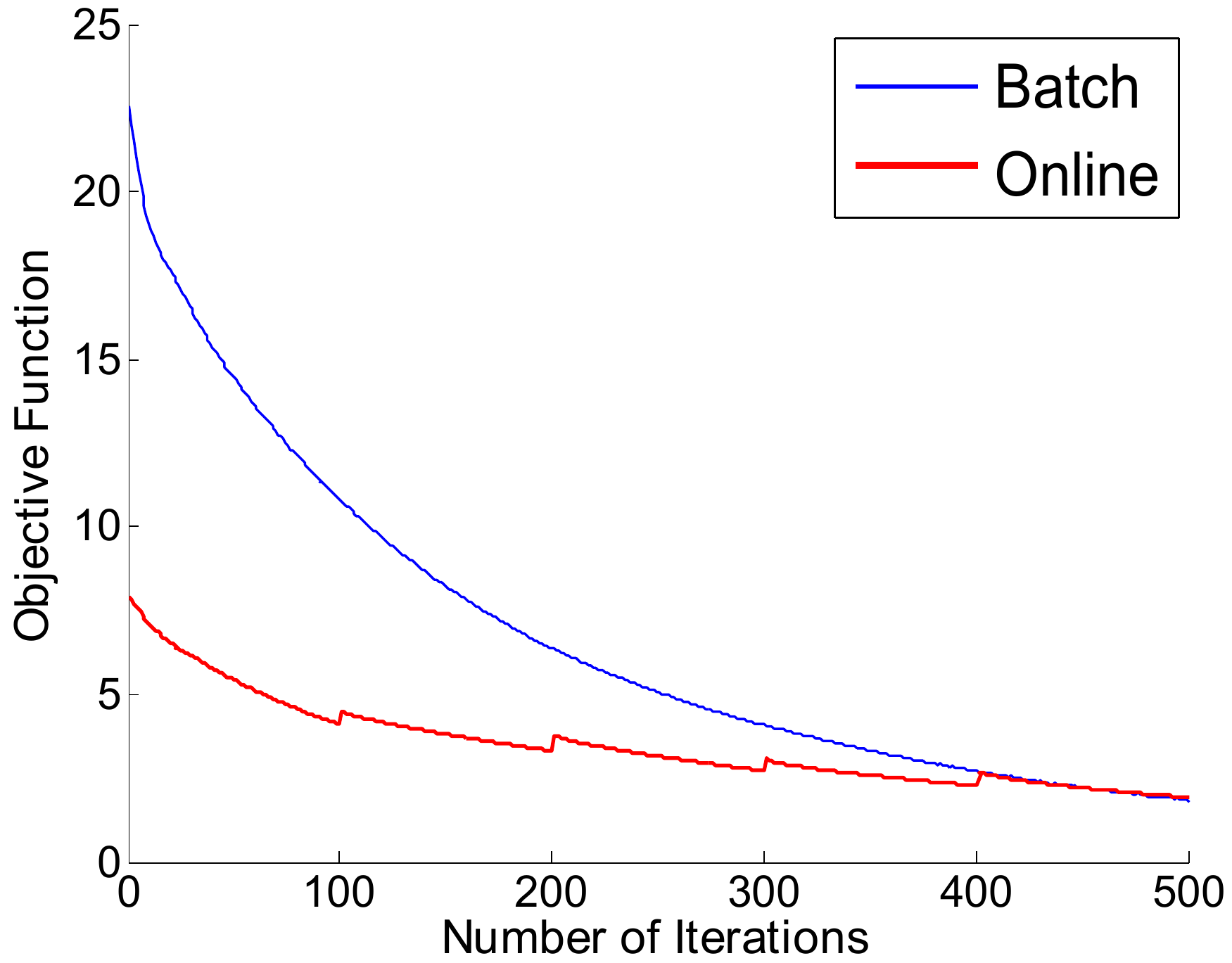}}
\caption{Convergence of the objective function (\ref{nuclear_obj}) using online and batch learning for subspace transformation. We always observe empirical convergence for both online and batch learning.
In (a), we vary the value of $\gamma$ in the norm constraint (``No normalization" denotes removing the norm constraint). More discussions on convergence can be found in Appendix~\ref{gradesc}.
In (b), to converge to the same objective function value, it takes $131.76$ sec. for online learning and $700.27$ sec. for batch learning.
}
\label{fig:convergence}
%\end{center}
\end{figure*}

In this set of experiments, we use digits \{1, 2\} from the MNIST dataset. We select 1000 images for each digit, and randomly partition them into 5 mini-batches.
{ We first perform one iteration of LRSC in Algorithm~\ref{algorsc} over all selected data with various  $\gamma$ values.
 As shown in Fig.~\ref{fig:batchconv}, we always observe empirical convergence for subspace transformation learning via (\ref{nuclear_obj}). The projected subgradient method presented in Appendix~\ref{gradesc} converges to a local minimum (or a stationary point).
 More discussions on convergence can be found in Appendix~\ref{gradesc}.
 }

Starting with the first mini-batch, we then perform one iteration of LRSC  over one mini-batch a time, with the subspace transformation learned from the previous mini-batch as warm restart. We adopt here $100$ iterations for the subgradient descent updates. As shown in Fig.~\ref{fig:olconv}, we observe similar empirical convergence for online transformation learning. To converge to the same objective function value, it takes
$131.76$ sec. for online learning and $700.27$ sec. for batch learning.

\subsection{Application to Face Clustering}

\begin{figure*} [ht]
\centering
 \subfloat[Example illumination conditions.] {\label{fig:yalelight} \includegraphics[angle=0, height=0.12\textwidth, width=.9\textwidth]{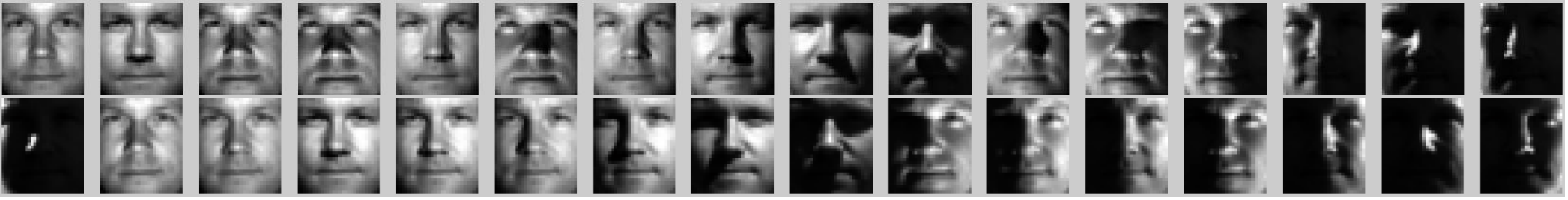} } \\
  \subfloat[Example subjects.] {\label{fig:yalesub} \includegraphics[angle=0, height=0.12\textwidth, width=.9\textwidth]{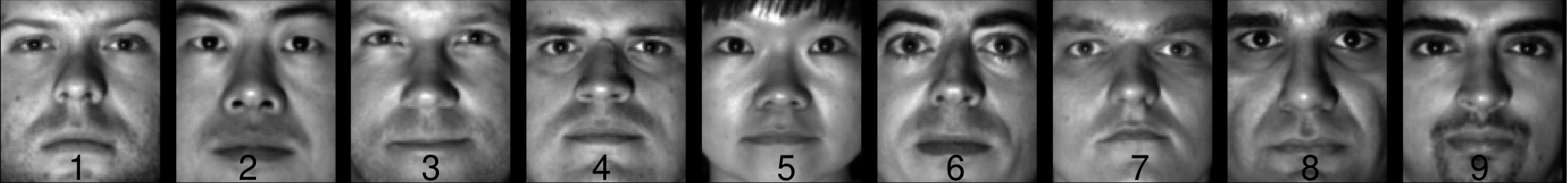}}
\caption{The extended YaleB face dataset.}
\label{fig:yaledata}
%\end{center}
\end{figure*}

\begin{figure*} [ht]
\centering
 \subfloat[Ground truth.] {\label{fig:9sub_truth} \includegraphics[angle=0, height=0.15\textwidth, width=.18\textwidth]{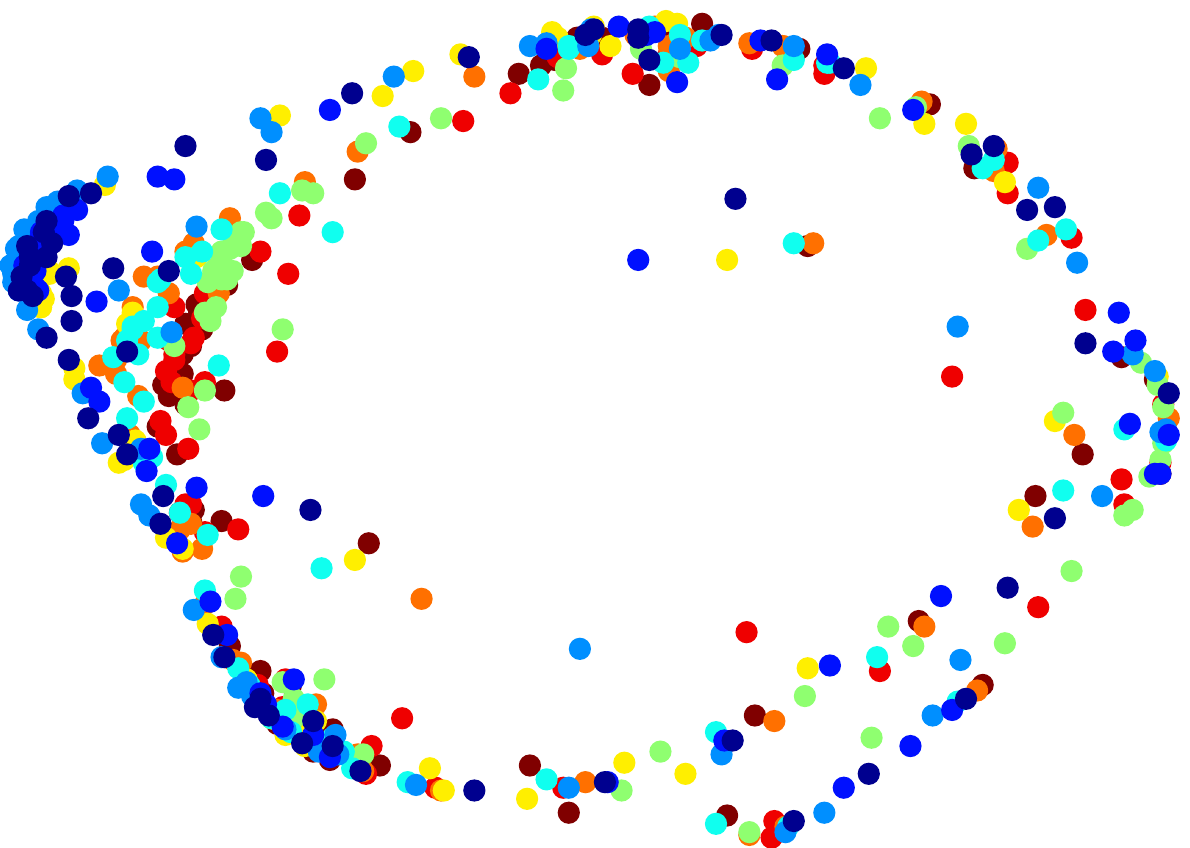}}
 \subfloat[][SSC, $e=71.25\%$, \\ $~~~~~~~t=714.99~sec$.] {\label{fig:9sub_ssc} \includegraphics[angle=0, height=0.15\textwidth, width=.2\textwidth]{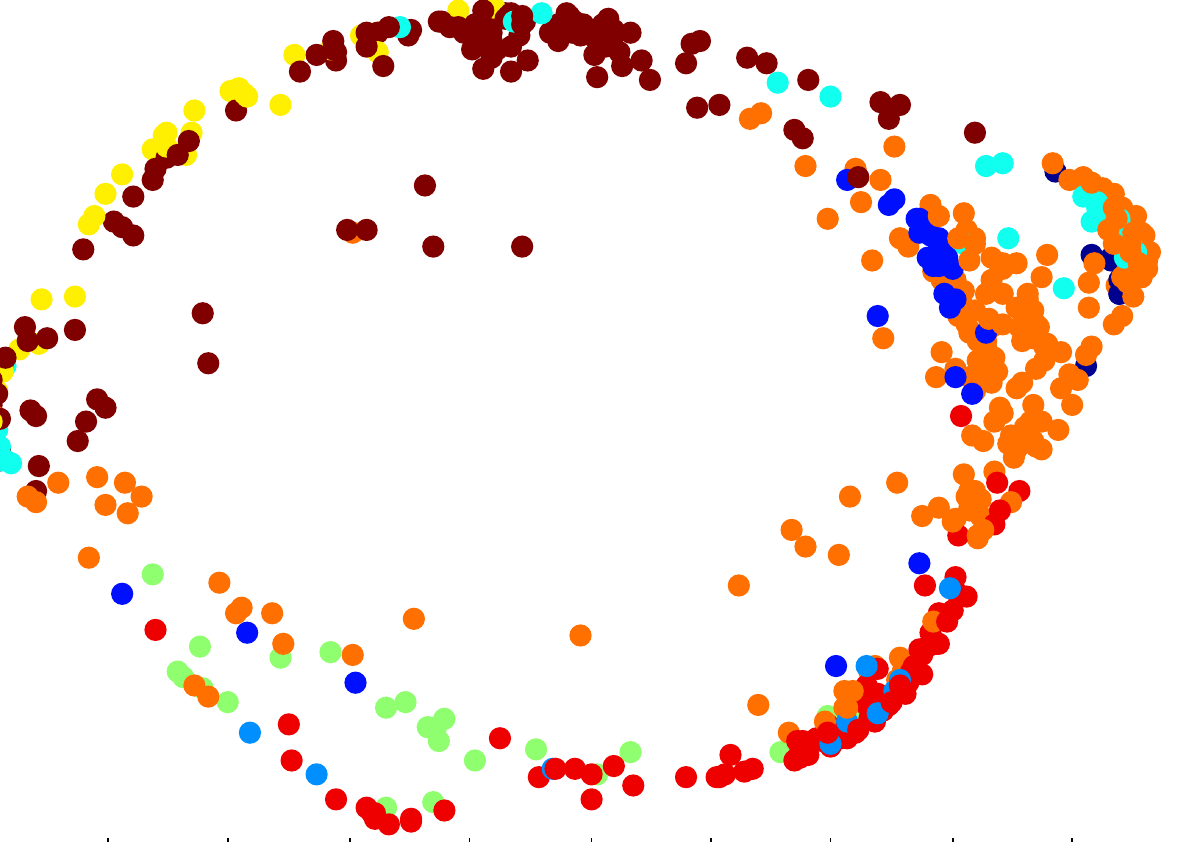}}
 \subfloat[][LBF, $e=76.37\%$, \\ $~~~~~~t=460.76~sec$.] {\label{fig:9sub_lbf} \includegraphics[angle=0, height=0.15\textwidth, width=.2\textwidth]{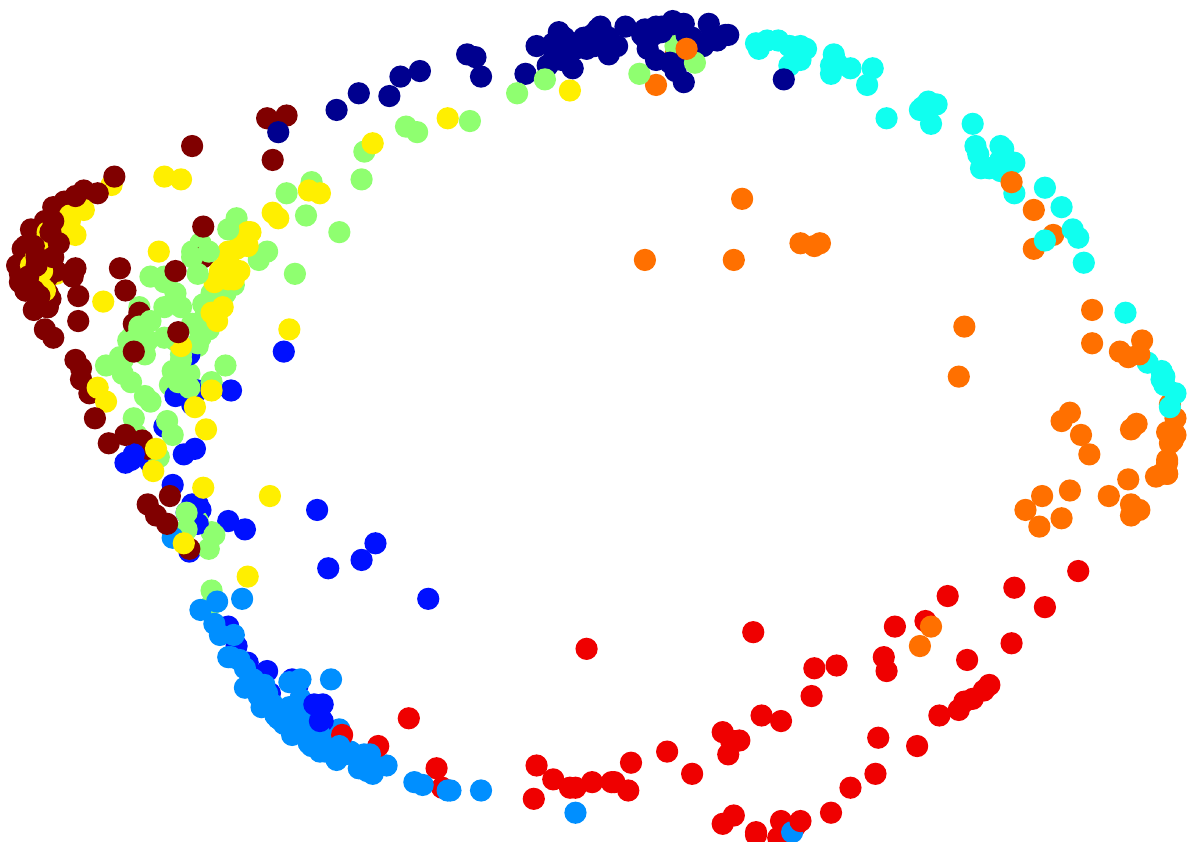}}
 \subfloat[][LSA, $e=71.96\%$, \\ $~~~~~~t=22.57~sec$.] {\label{fig:9sub_lsa} \includegraphics[angle=0, height=0.15\textwidth, width=.21\textwidth]{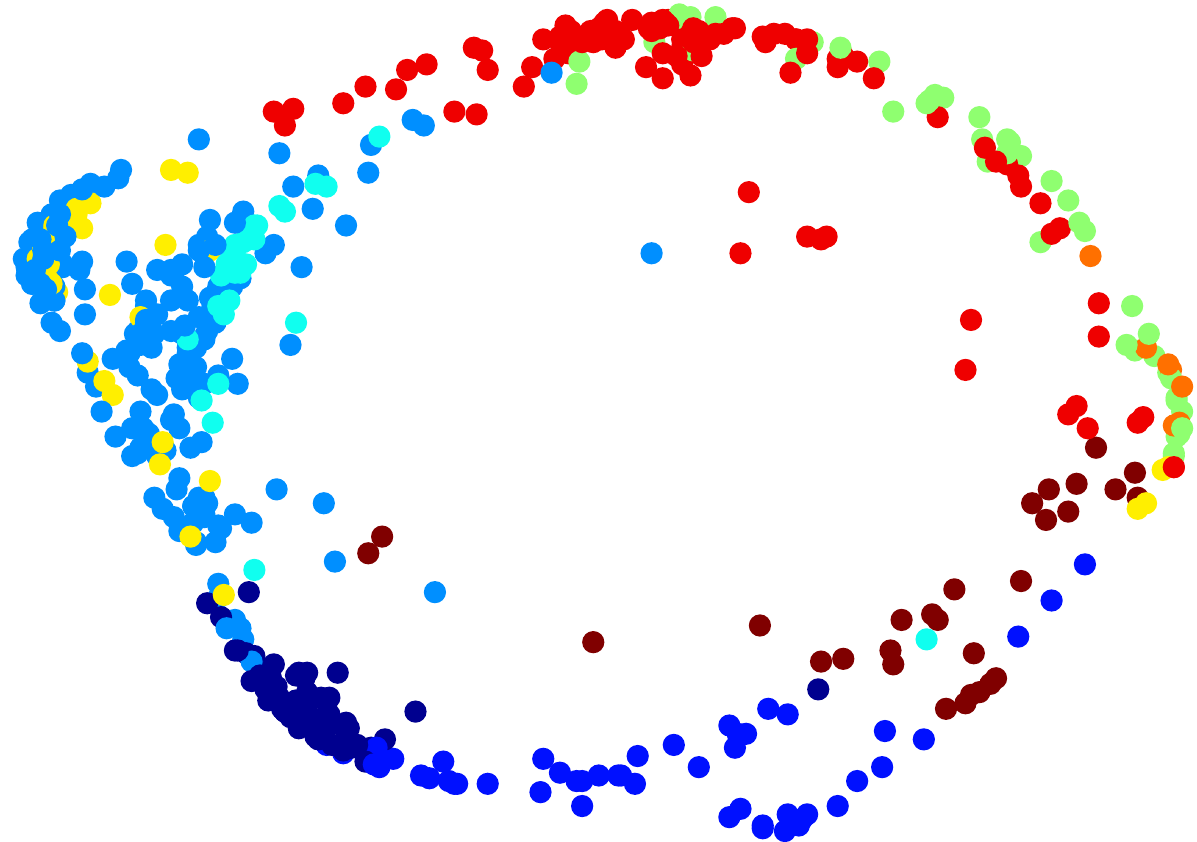}}
 \subfloat[][R-SSC, $e=\mathbf{67.37}\%$,\\ $~~~~~~t=\mathbf{1.83}~sec$.] {\label{fig:9sub_rssc} \includegraphics[angle=0, height=0.16\textwidth, width=.23\textwidth]{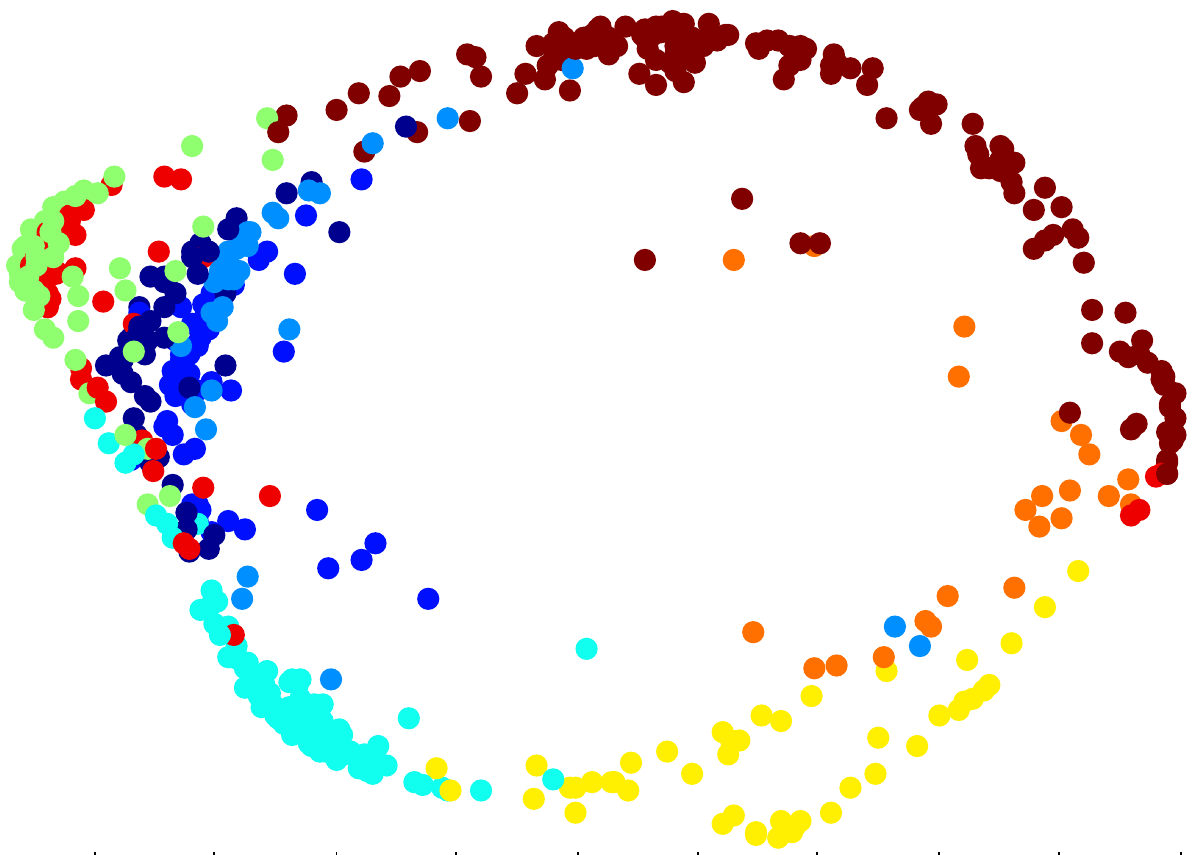}}\\
\caption{Misclassification rate (\emph{e})  and running time (\emph{t}) on clustering 9 subjects using different subspace clustering methods.
The proposed R-SSC outperforms state-of-the-art methods both in accuracy and running time.
This is further improved using the learned transform, LRSC reduces the error to 4.94\%, see Fig.~\ref{fig:9sub_lrsc}.
}
\label{fig:9sub_sc}
%\end{center}
\end{figure*}

\begin{figure*} [ht]
\centering
 \subfloat[Ground truth (iter=1).] {\label{fig:9sub1_g} \includegraphics[angle=0, height=0.2\textwidth, width=.25\textwidth]{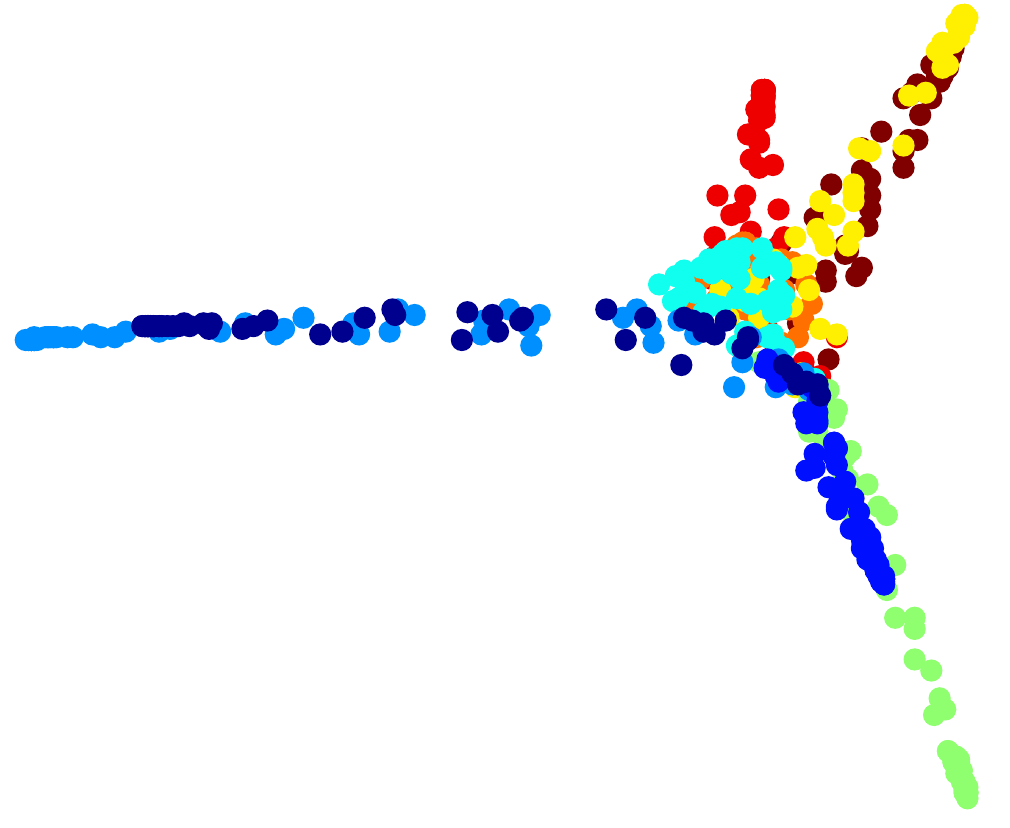}}
 \subfloat[$e=40.39\%$ (iter=1).] {\label{fig:9sub1_c} \includegraphics[angle=0, height=0.2\textwidth, width=.25\textwidth]{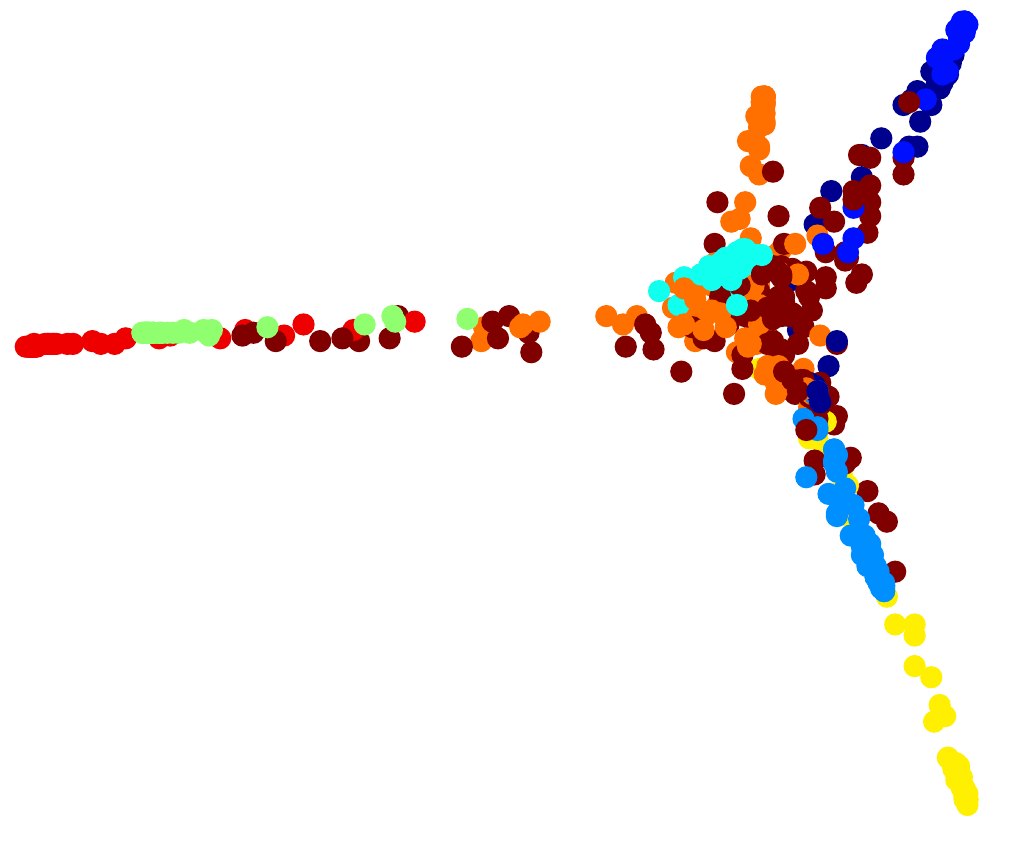}\hspace{0pt}}
  \subfloat[Ground truth (iter=2).] {\label{fig:9sub2_g} \includegraphics[angle=0, height=0.2\textwidth, width=.25\textwidth]{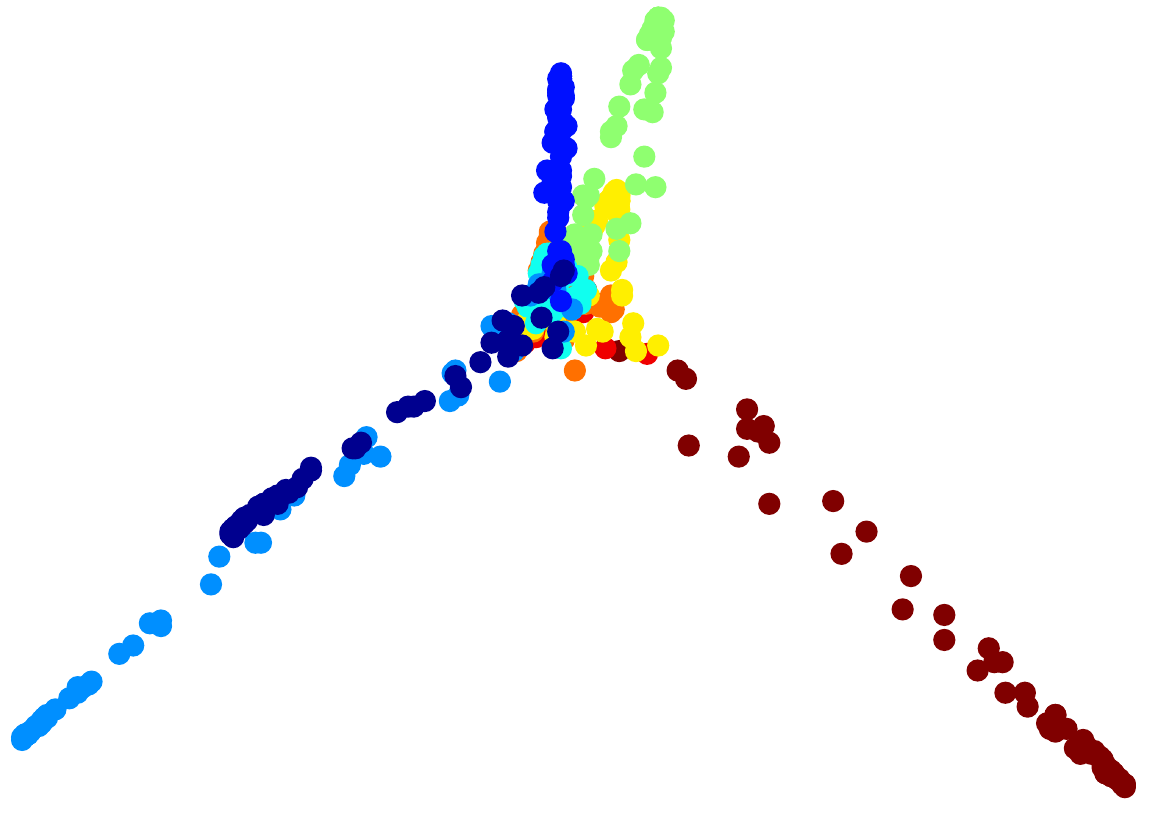}}
 \subfloat[$e=33.51\%$ (iter=2).] {\label{fig:9sub2_c} \includegraphics[angle=0, height=0.2\textwidth, width=.25\textwidth]{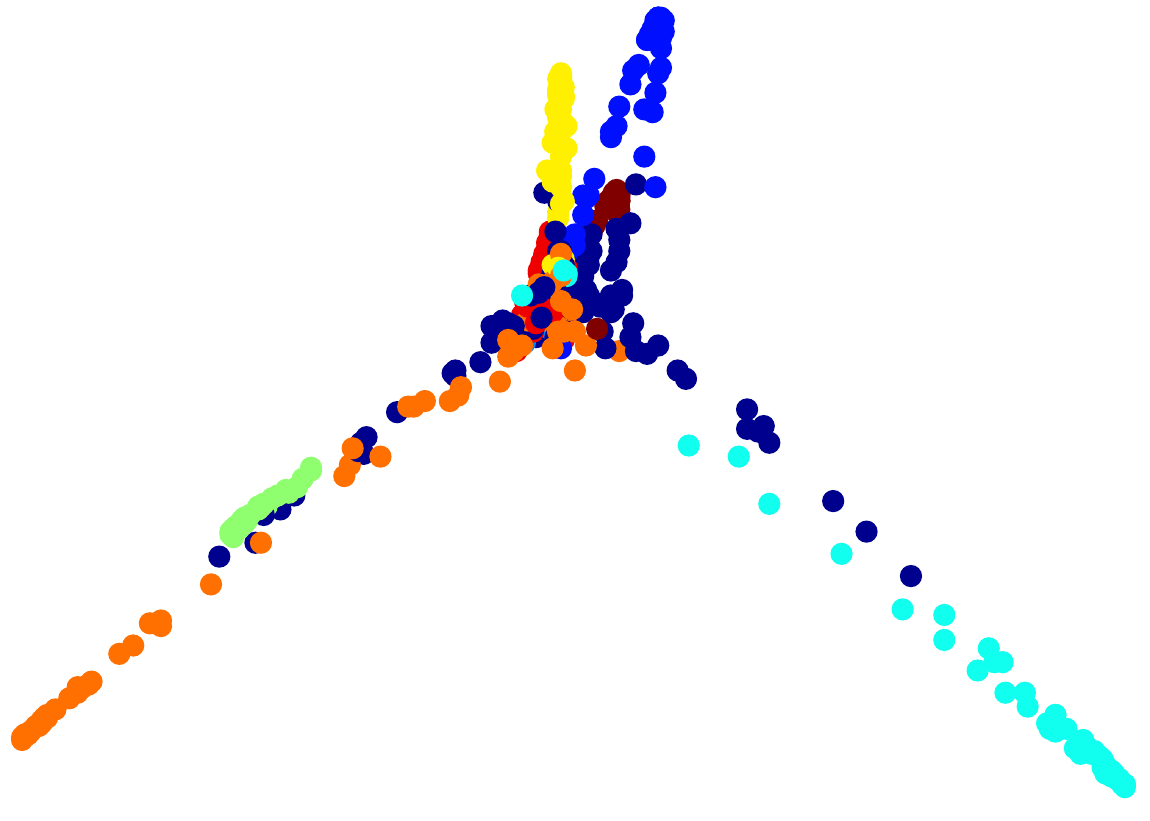}}\\
  \subfloat[Ground truth (iter=3).] {\label{fig:9sub3_g} \includegraphics[angle=0, height=0.2\textwidth, width=.25\textwidth]{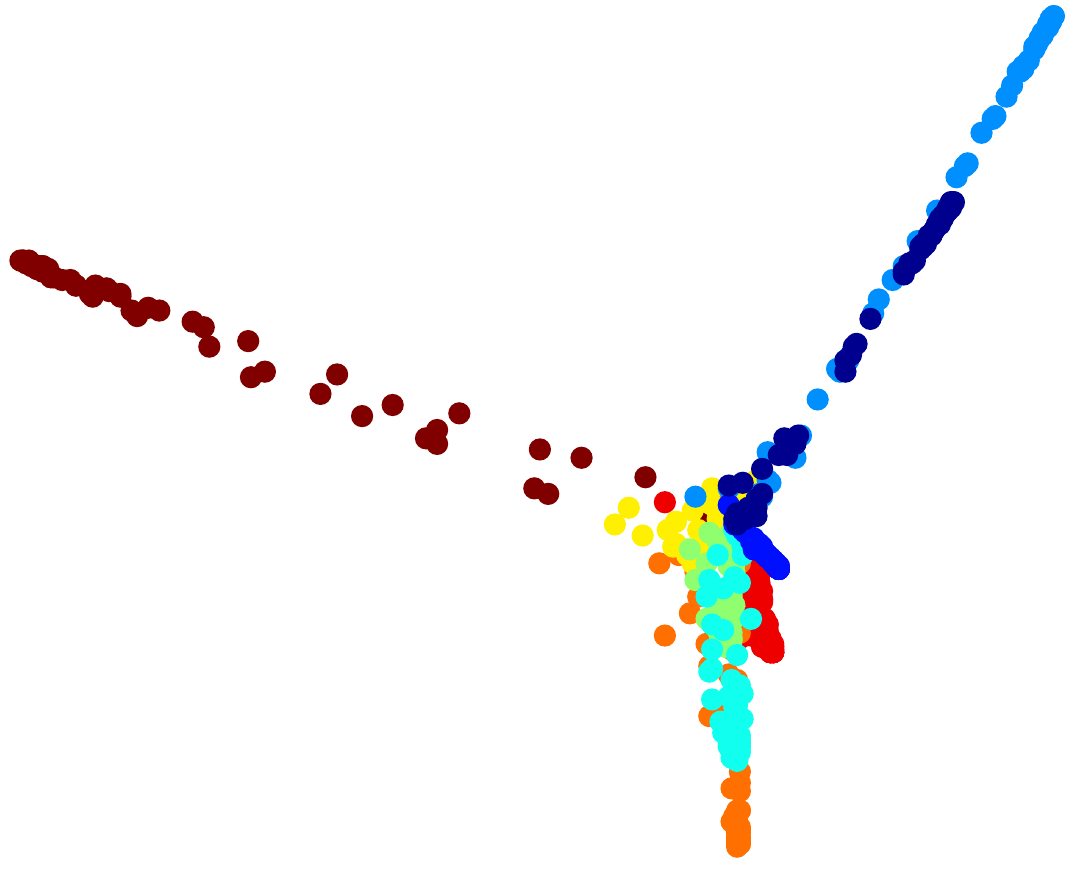}}
 \subfloat[$e=29.98\%$ (iter=3).] {\label{fig:9sub3_c} \includegraphics[angle=0, height=0.2\textwidth, width=.25\textwidth]{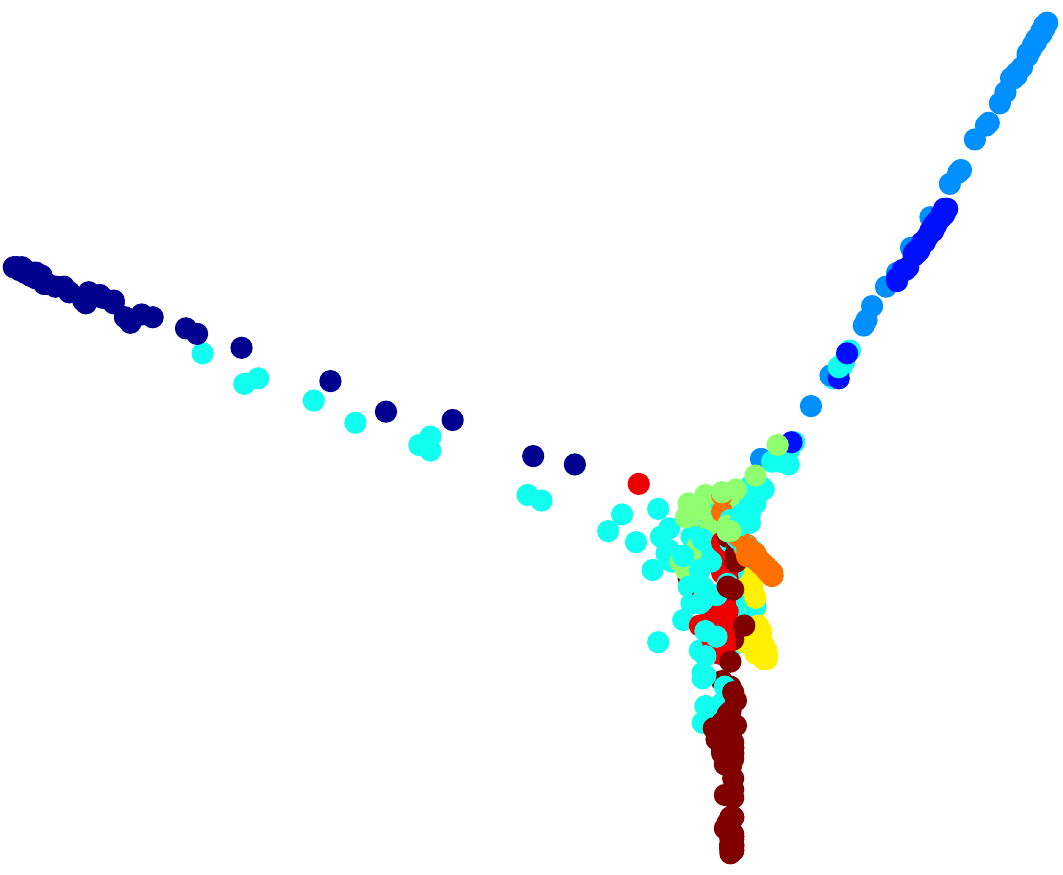}}
 \subfloat[Ground truth (iter=6).] {\label{fig:9sub6_g} \includegraphics[angle=0, height=0.2\textwidth, width=.25\textwidth]{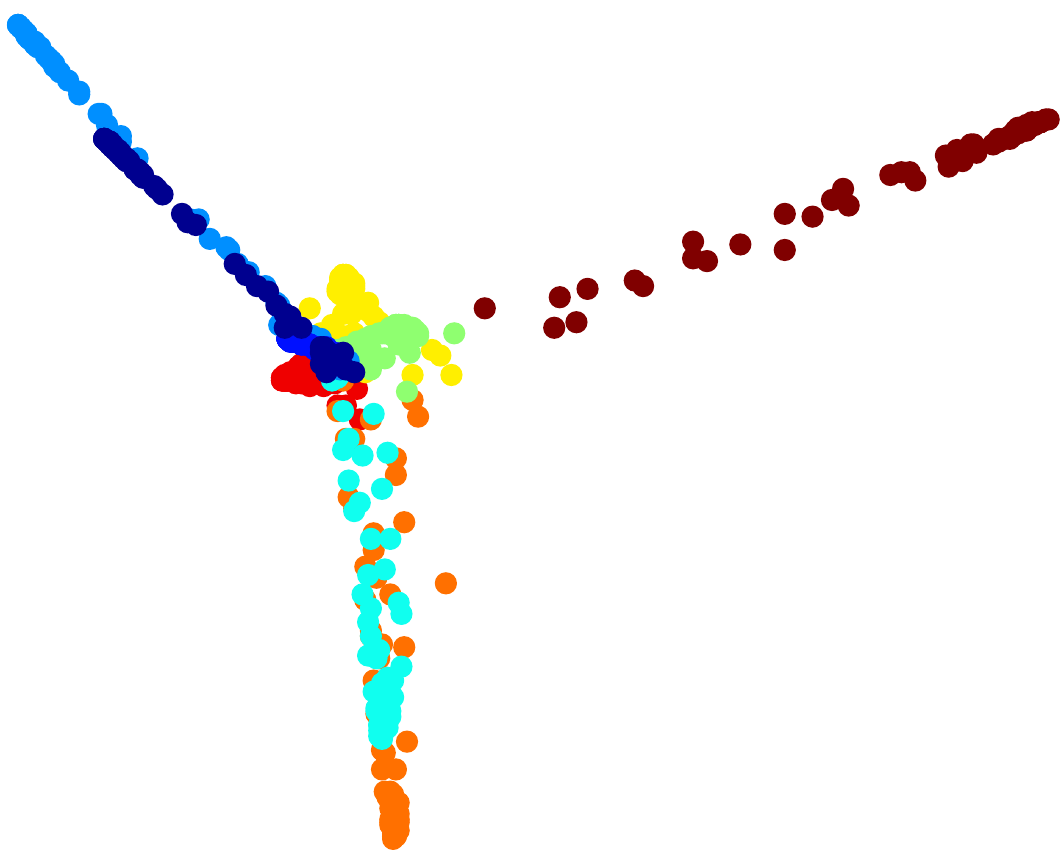}}
 \subfloat[$e=13.40\%$ (iter=6).] {\label{fig:9sub6_c} \includegraphics[angle=0, height=0.2\textwidth, width=.25\textwidth]{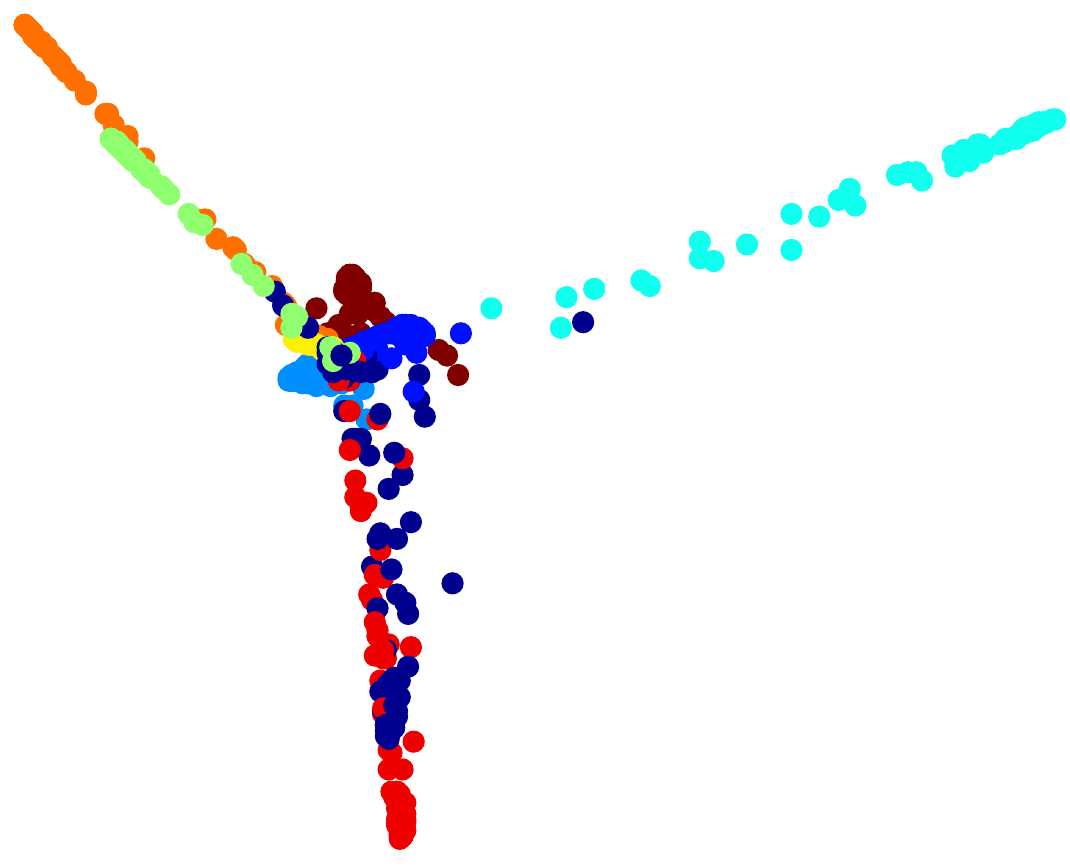}}\\
  \subfloat[Ground truth (iter=8).] {\label{fig:9sub8_g} \includegraphics[angle=0, height=0.2\textwidth, width=.25\textwidth]{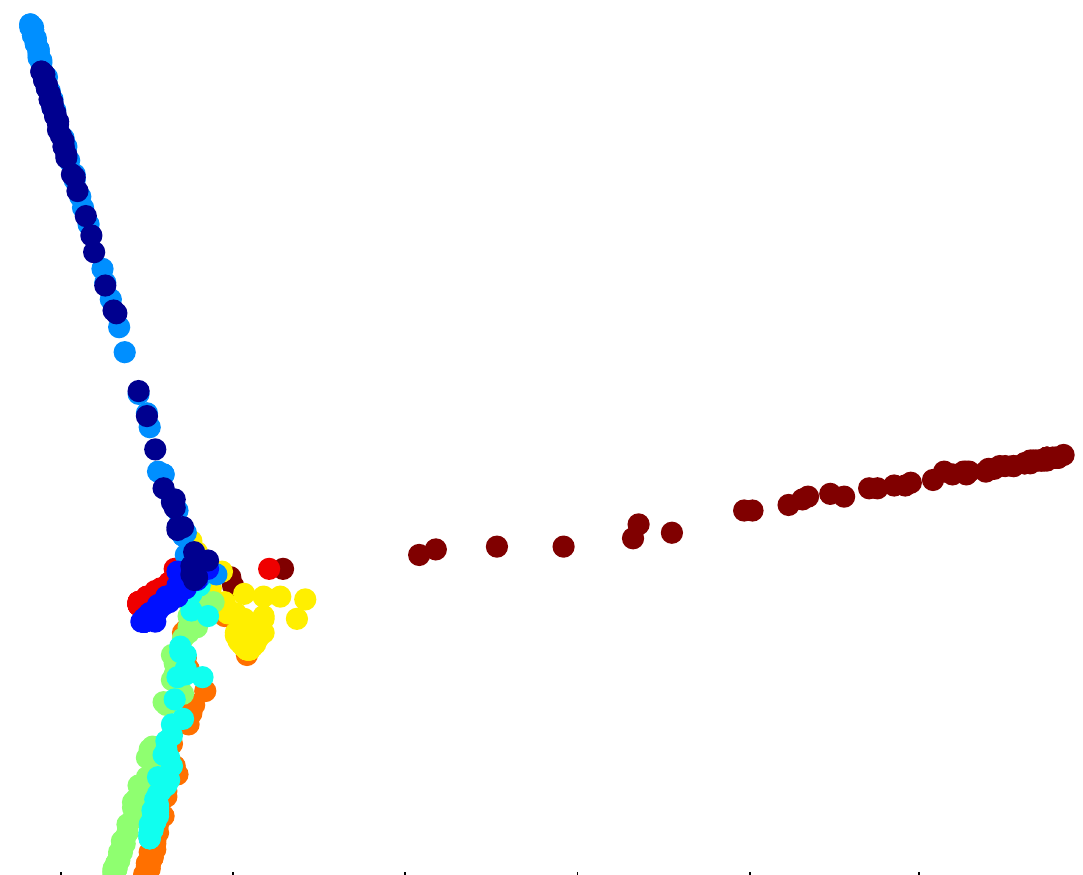}}
 \subfloat[$e=6.17\%$ (iter=8).] {\label{fig:9sub8_c} \includegraphics[angle=0, height=0.2\textwidth, width=.25\textwidth]{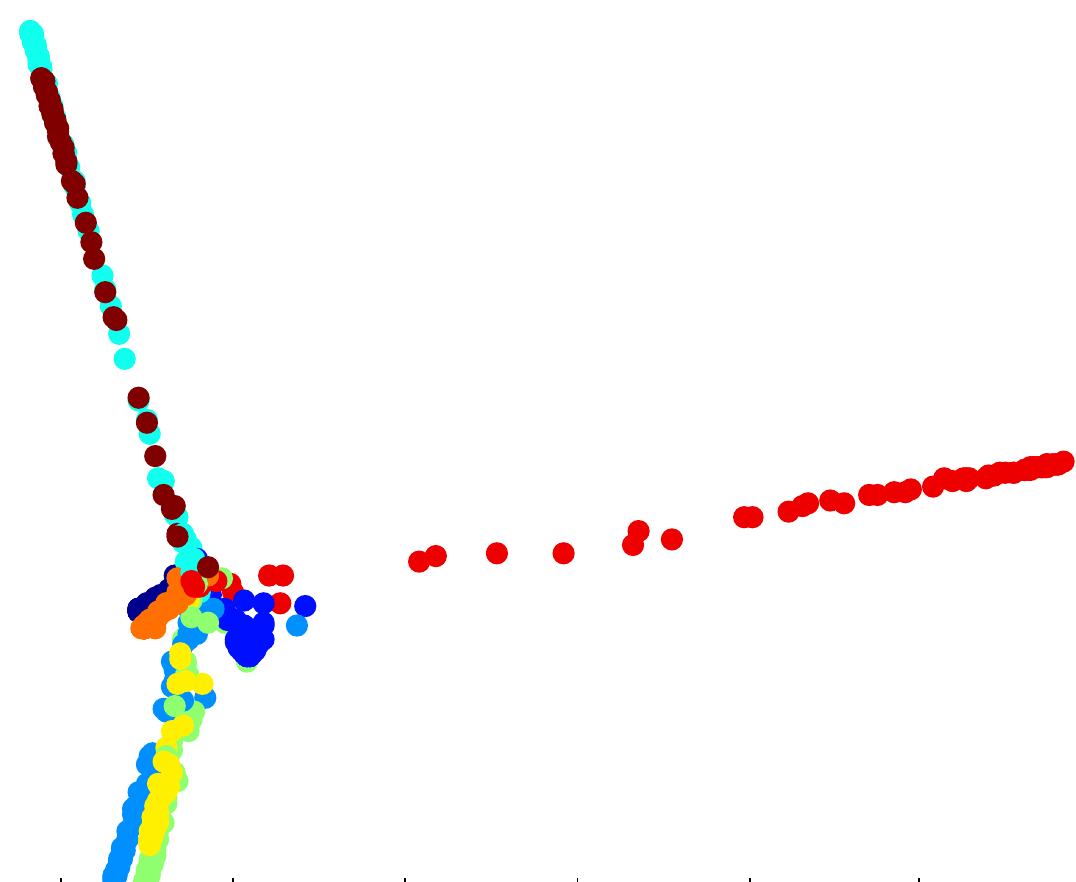}}
  \subfloat[Ground truth (iter=12).] {\label{fig:9sub12_g} \includegraphics[angle=0, height=0.2\textwidth, width=.26\textwidth]{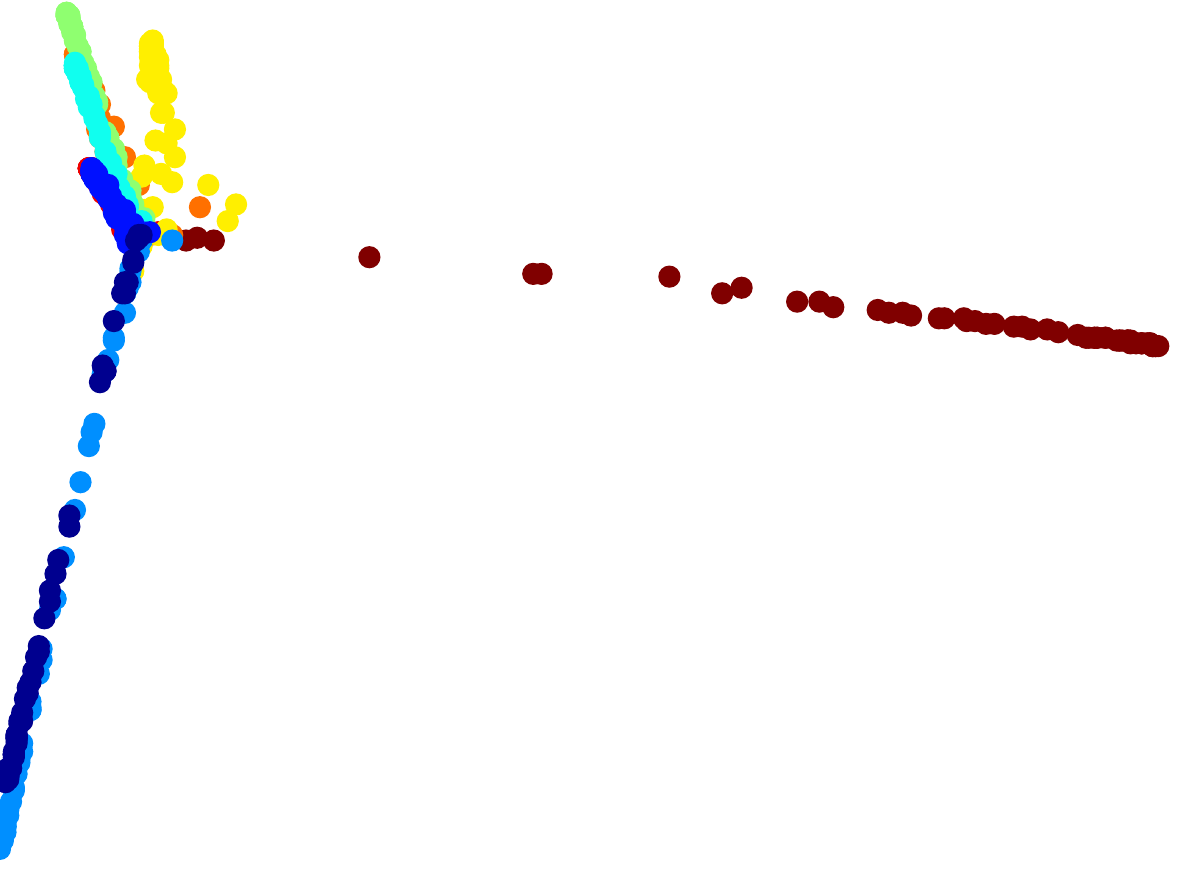}}
 \subfloat[$e=\mathbf{4.94}\%$ (iter=12).] {\label{fig:9sub12_c} \includegraphics[angle=0, height=0.2\textwidth, width=.25\textwidth]{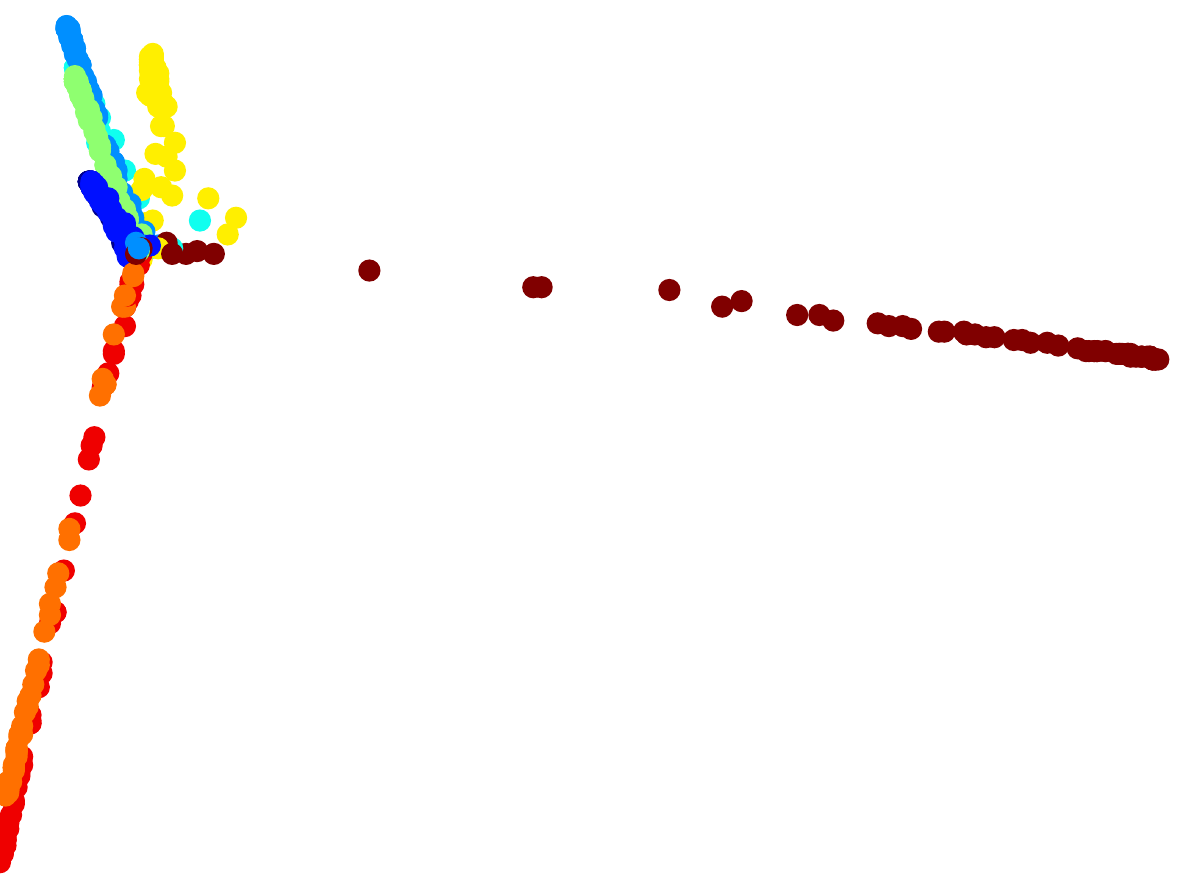}\hspace{30pt}}\\
  \subfloat[Misclassification rate.] {\label{fig:9sub_error} \includegraphics[angle=0, height=0.25\textwidth, width=.35\textwidth]{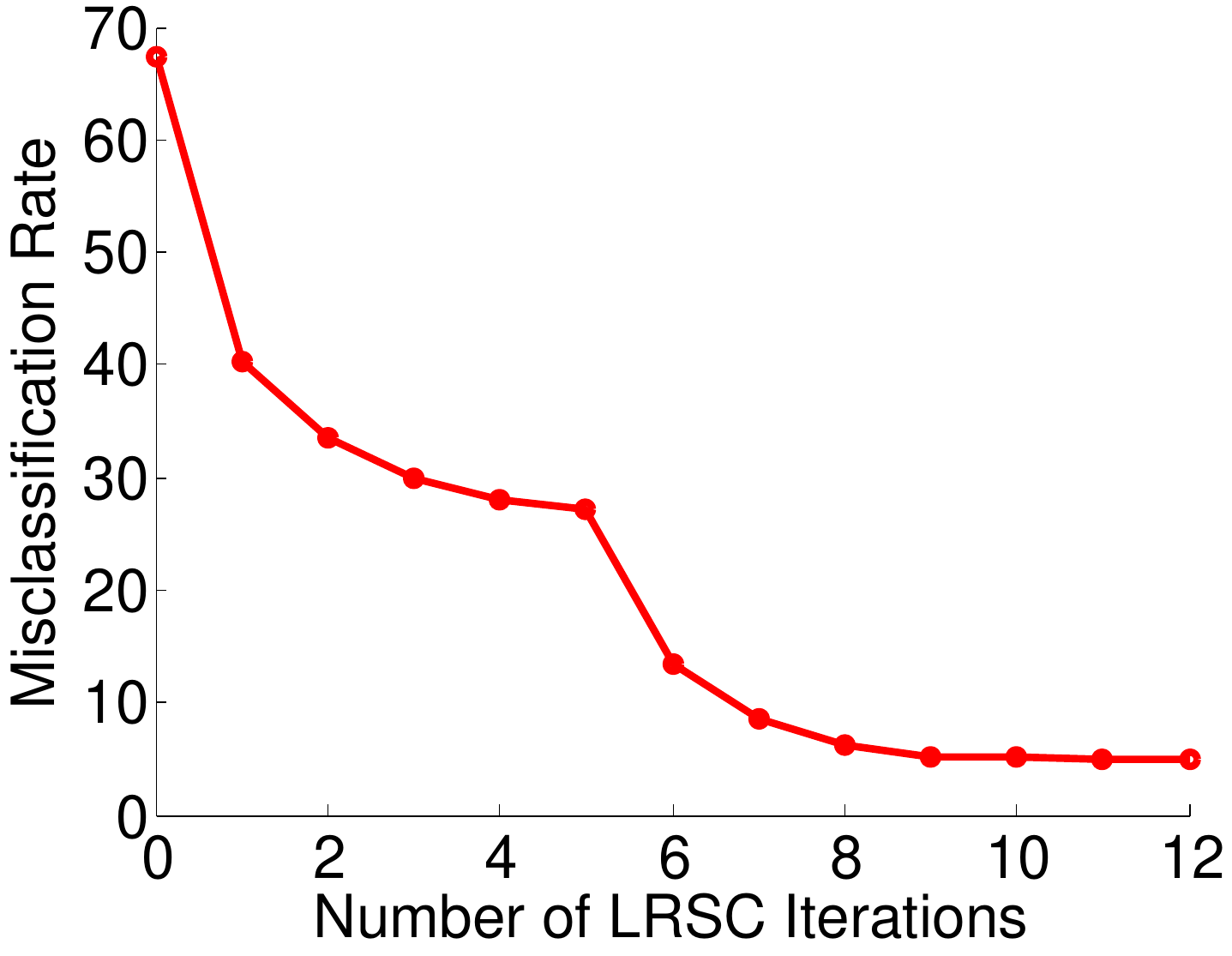}}
    \subfloat[Convergence of $\mathbf{T}$ updating] {\label{fig:conv_T_yale} \includegraphics[angle=0, height=0.25\textwidth, width=.35\textwidth]{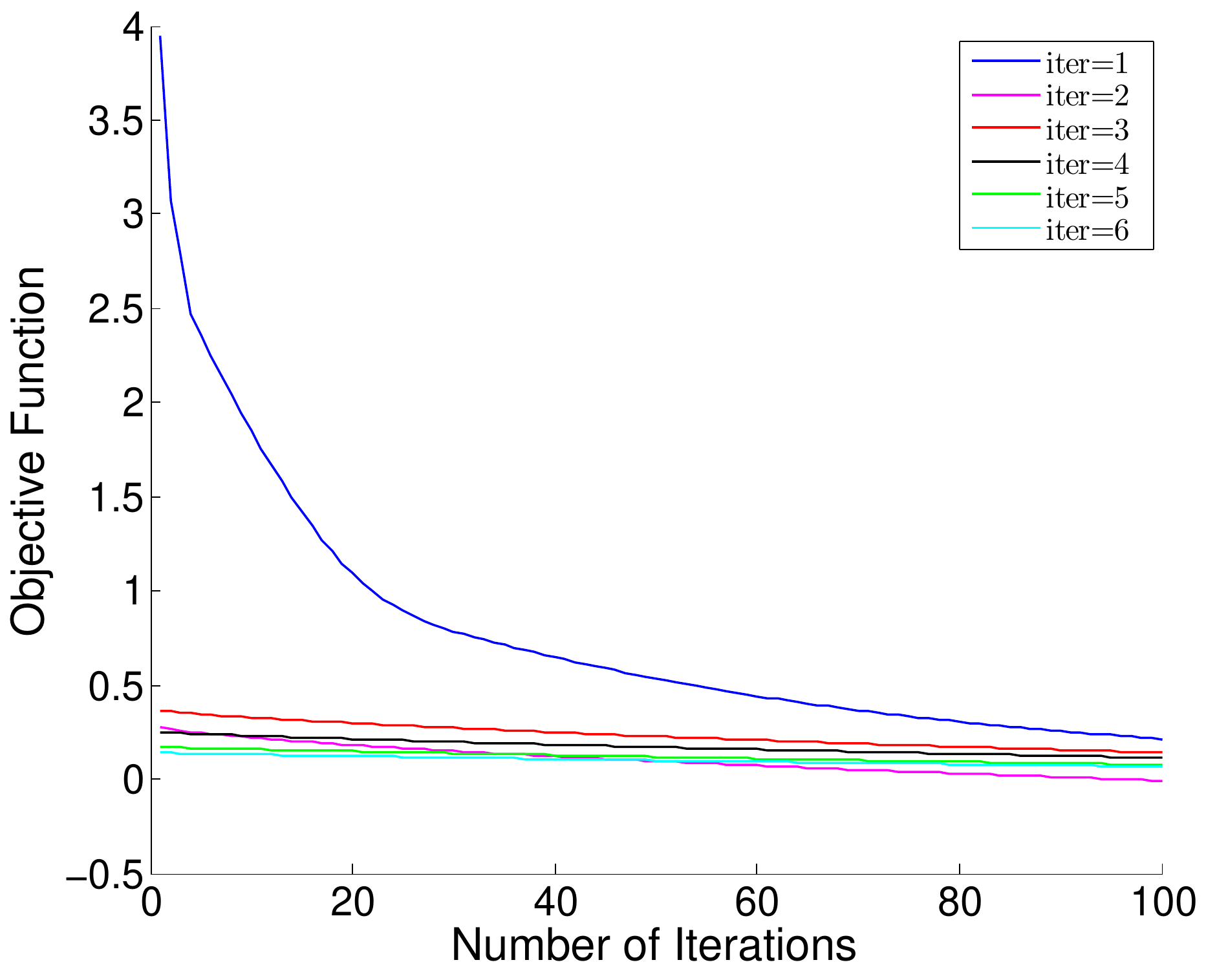}}\\
\caption{Misclassification rate (\emph{e}) on clustering 9 subjects using the proposed LRSC framework. We adopt the proposed R-SSC technique for the clustering step.
With the proposed LRSC framework, the clustering error of R-SSC is further reduced significantly, e.g., from $67.37 \%$ to $4.94 \%$ for the 9-subject case. Note how the classes are clustered in clean subspaces in the transformed domain.
}
\label{fig:9sub_lrsc}
%\end{center}
\end{figure*}

\begin{figure*} [ht]
\centering
 \subfloat[Original smallest angles.] {\label{fig:9sub0_angle} \includegraphics[angle=0, height=0.33\textwidth, width=.33\textwidth]{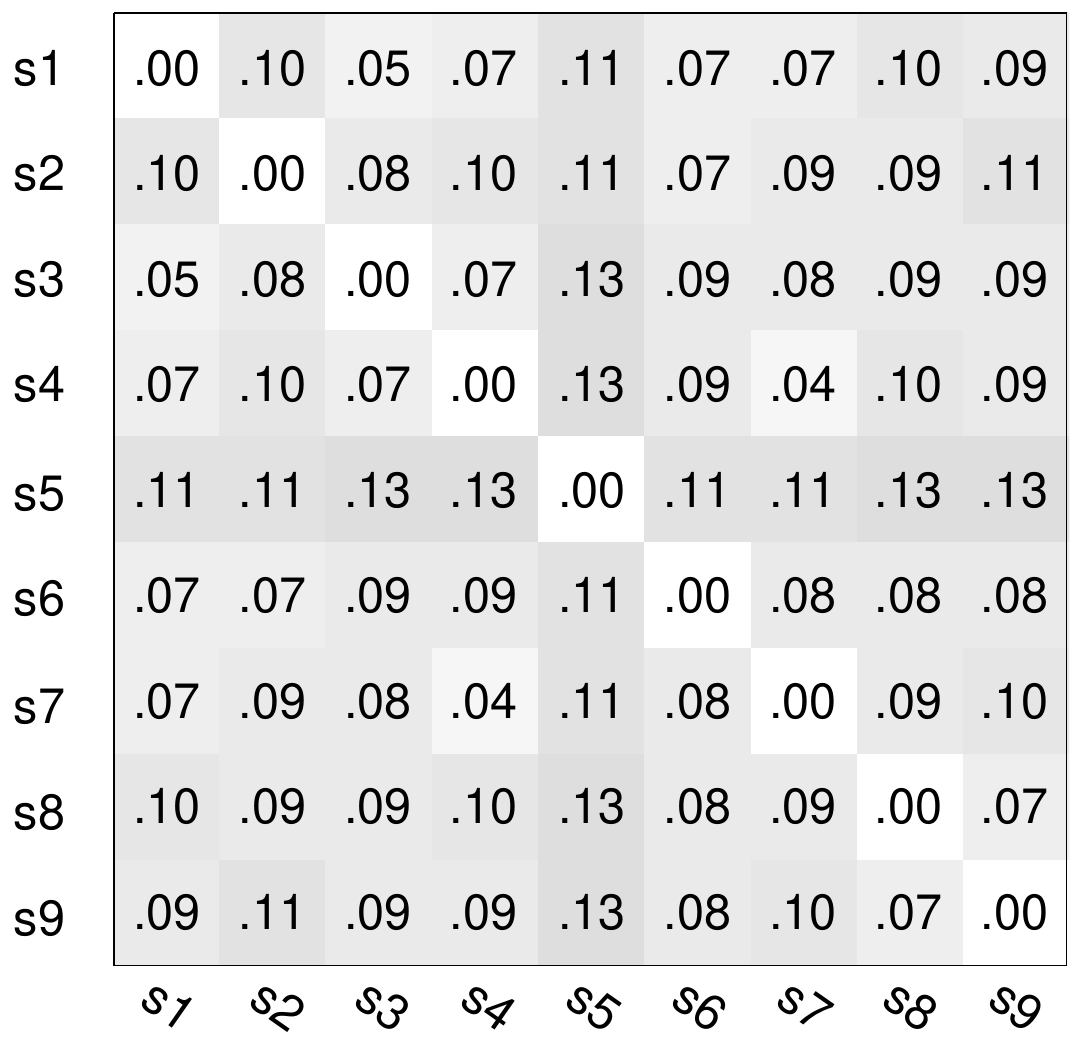}}
 \subfloat[Transformed smallest angles.] {\label{fig:9sub12_angle} \includegraphics[angle=0, height=0.33\textwidth, width=.33\textwidth]{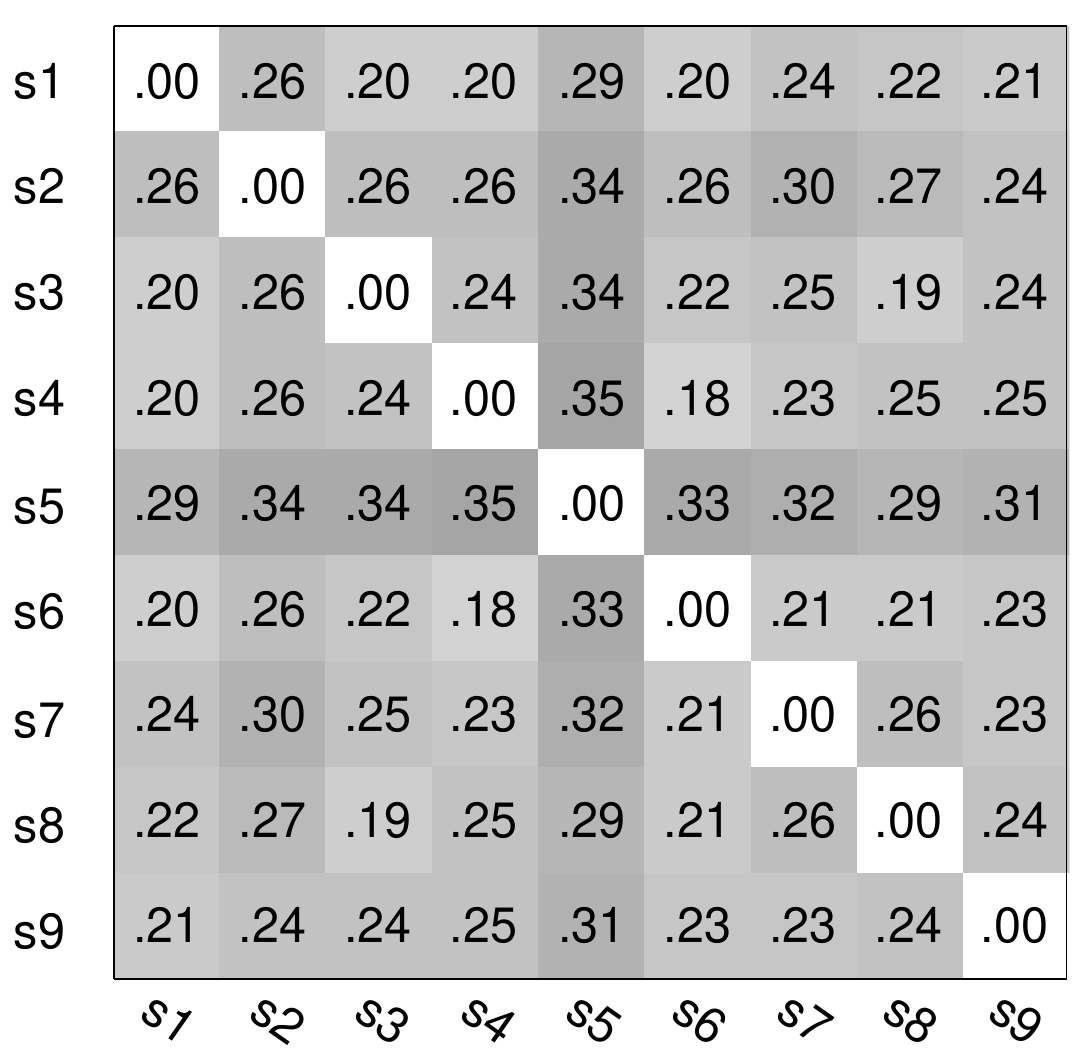}} \\
  \subfloat[Original mean cosine angles.] {\label{fig:9sub0_angle} \includegraphics[angle=0, height=0.33\textwidth, width=.33\textwidth]{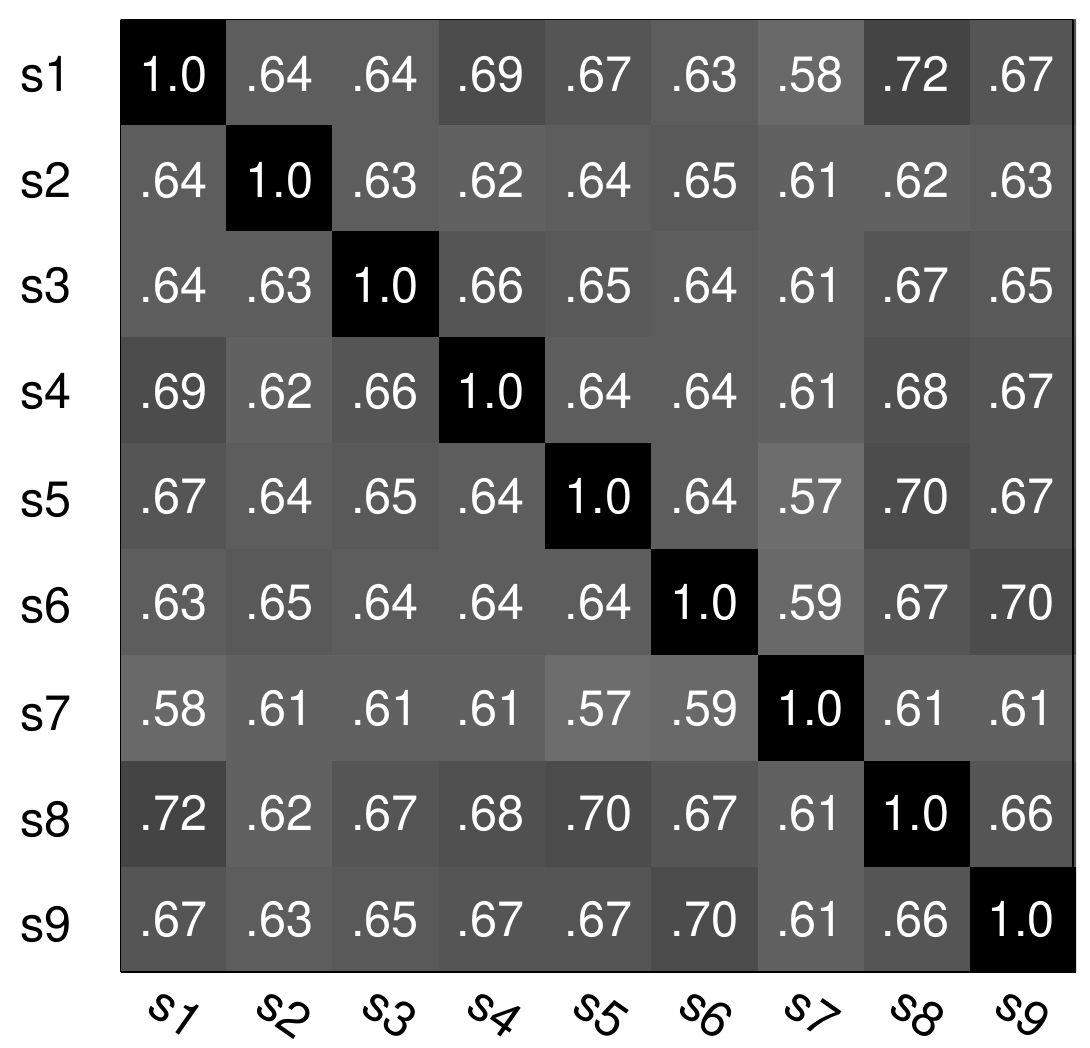}}
 \subfloat[Transformed mean cosine angles.] {\label{fig:9sub12_angle} \includegraphics[angle=0, height=0.33\textwidth, width=.34\textwidth]{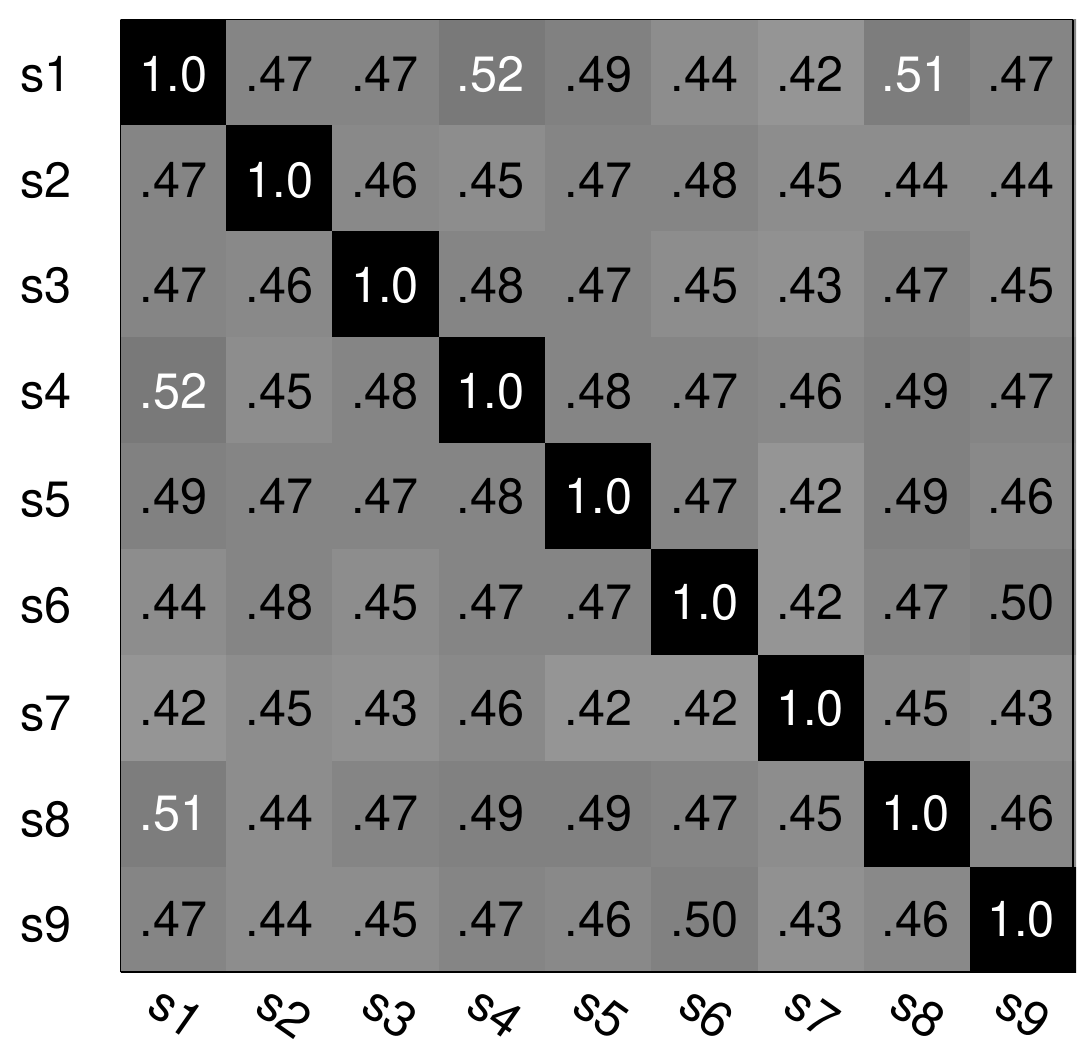}}
 \subfloat[Subspace nuclear norm.] {\label{fig:9sub_nuorm} \includegraphics[angle=0, height=0.33\textwidth, width=.32\textwidth]{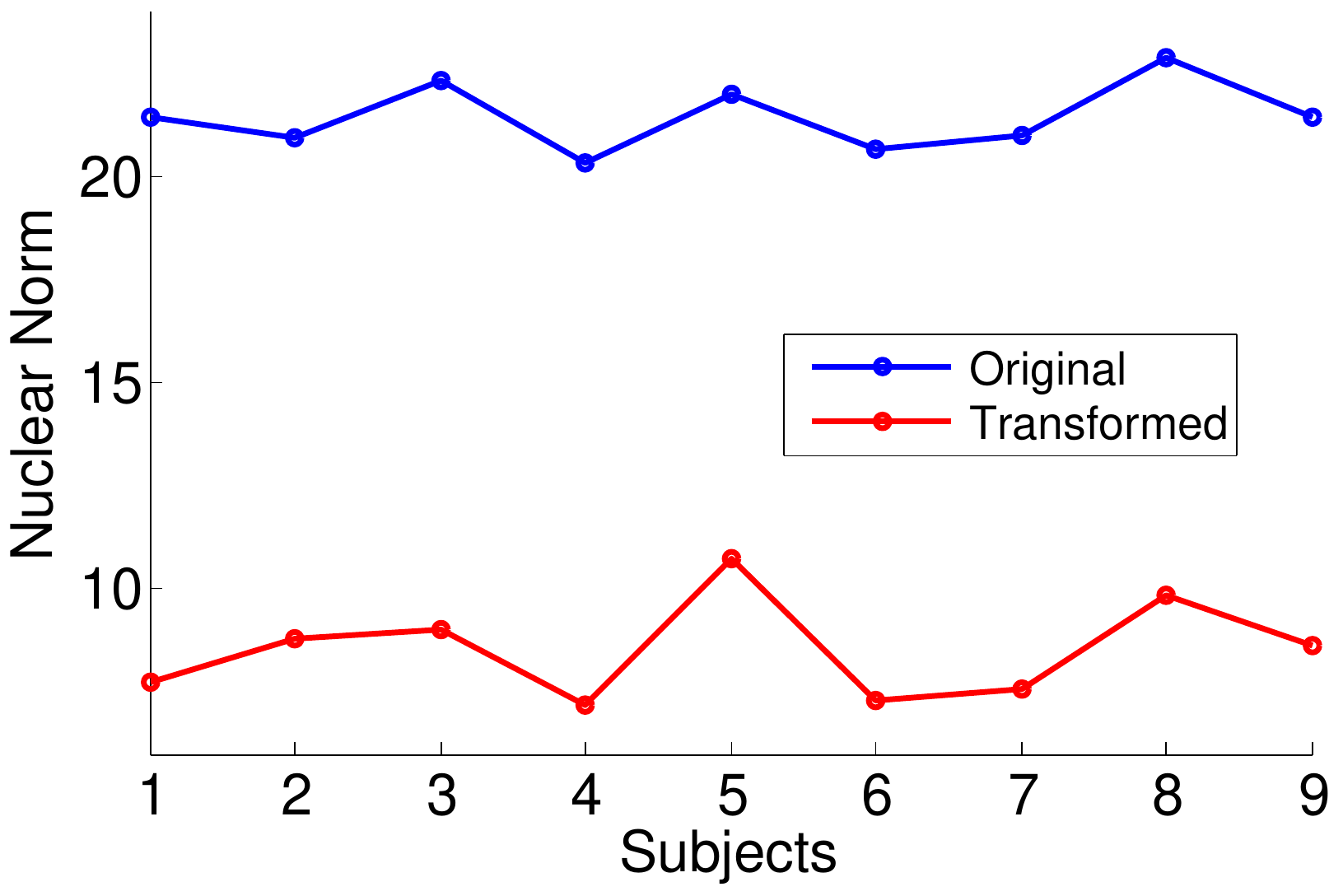}}
\caption{The smallest and mean principal angles between pairs of 9 subject subspaces and the nuclear norms of 9 subject subspaces before and after transformation.
Note that each entry in (a) and (b) denotes the smallest principal angle, and each entry in (c) and (d) denotes the average cosine over all principal angles.
%i.e., $\frac{1}{K} \sum_{k=1}^K \cos \theta^{(k)}_{ij} $, where $\theta^{(k)}_{ij}$ is the $k$-th principle angle between two subspaces $\mathcal{S}_i$ and $\mathcal{S}_j$.
We observe that the learned subspace transformation increases the angles between subspaces and also reduces the nuclear norms of subspaces.
Overall, the average smallest principal angles between subspaces increased from 0.09 to 0.26, and the average subspace nuclear norm decreased from 21.43 to 8.53.
}
\label{fig:9sub_angle}
%\end{center}
\end{figure*}

 In the Extended YaleB dataset, each of the 38 subjects is imaged under 64 lighting conditions, shown in Fig.~\ref{fig:yalelight}.
{ Under the assumption of Lambertian reflectance,  face images of each subject under different lighting conditions can be accurately approximated with a 9-dimensional linear subspace (\cite{9point}).
 }
  We conduct the face clustering experiments on the first 9 subjects shown in Fig.~\ref{fig:yalesub}.
  We set the sparsity value $K=10$ for R-SSC, and perform $100$ iterations for the subgradient descent updates while learning the transformation.

Fig.~\ref{fig:9sub_sc} shows error rate (\emph{e}) and running time (\emph{t}) on clustering
 subspaces of 9 subjects using different subspace clustering methods. The proposed R-SSC techniques outperforms state-of-the-art methods both in accuracy and running time.
As shown in Fig.~\ref{fig:9sub_lrsc}, using the proposed LRSC algorithm  (that is, learning the transform), the misclassification errors of R-SSC are further reduced significantly, for example, from $67.37 \%$ to $4.94 \%$ for the 9 subjects.
{
Fig.~\ref{fig:conv_T_yale} shows the convergence of the $\mathbf{T}$ updating step in the first few LRSC iterations.
The dramatic performance improvement can be explained in Fig.~\ref{fig:9sub_angle}. We observe, as expected from the theory presented before, that the learned subspace transformation increases the distance (the smallest principal angle) between subspaces and, at the same time, reduces the nuclear norms of subspaces.
More results on clustering subspaces of 2 and 3 subjects are shown in Fig.~\ref{fig:yaleacc}.
}

\begin{figure*} [ht]
\centering
 \subfloat[Subjects \{1, 2\}.] {\label{fig:y12-1} \includegraphics[angle=0, height=0.15\textwidth, width=.7\textwidth]{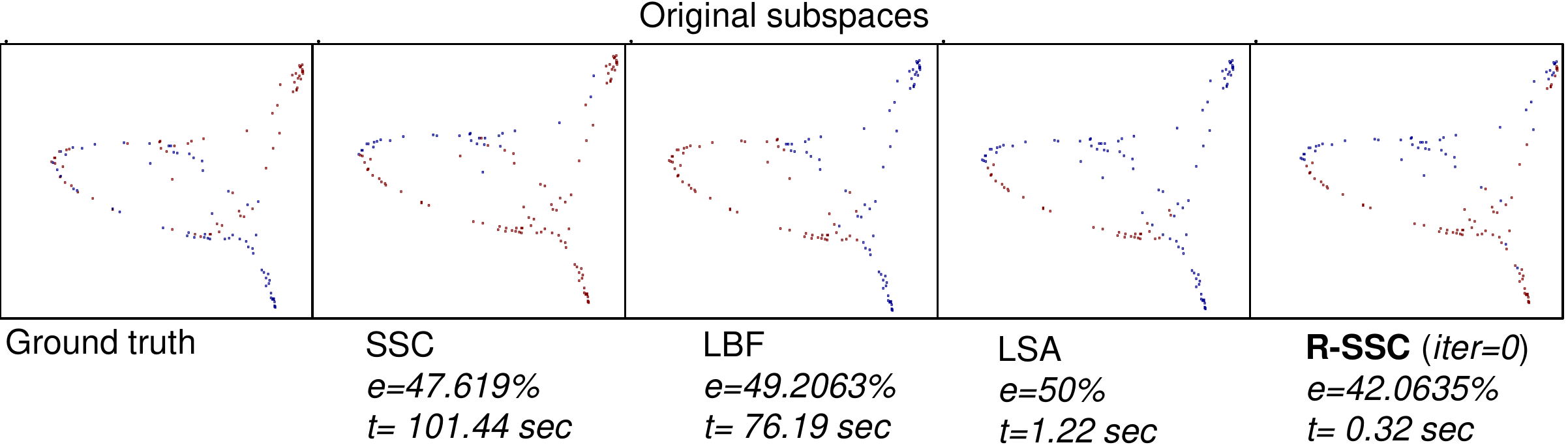} }
 \subfloat[Subjects \{1, 2\}.] {\label{fig:y12-2} \includegraphics[angle=0, height=0.15\textwidth, width=.3\textwidth]{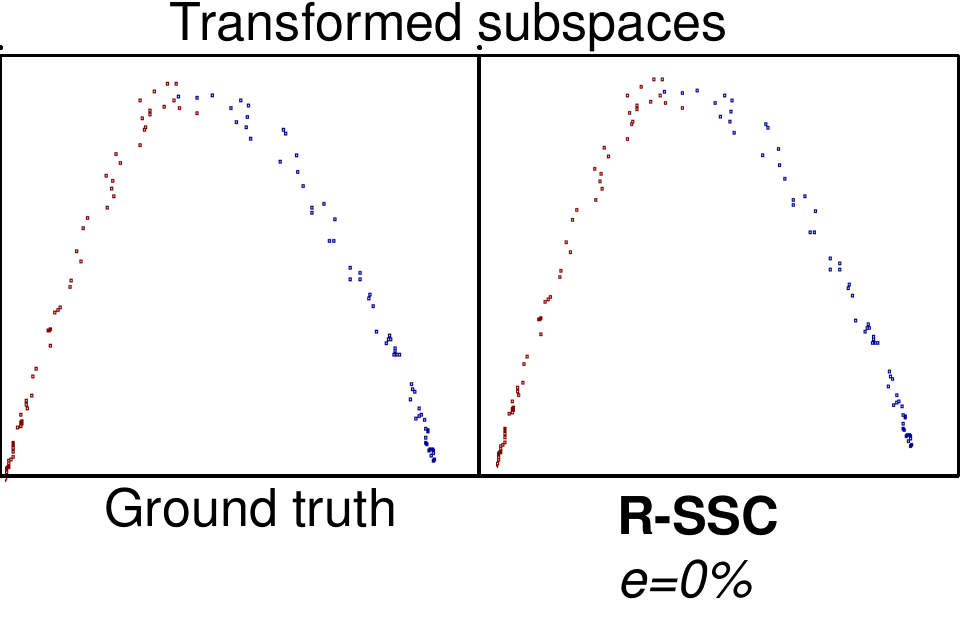}} \\
  \subfloat[Subjects \{2, 3\}.] {\label{fig:y23-1} \includegraphics[angle=0, height=0.15\textwidth, width=.7\textwidth]{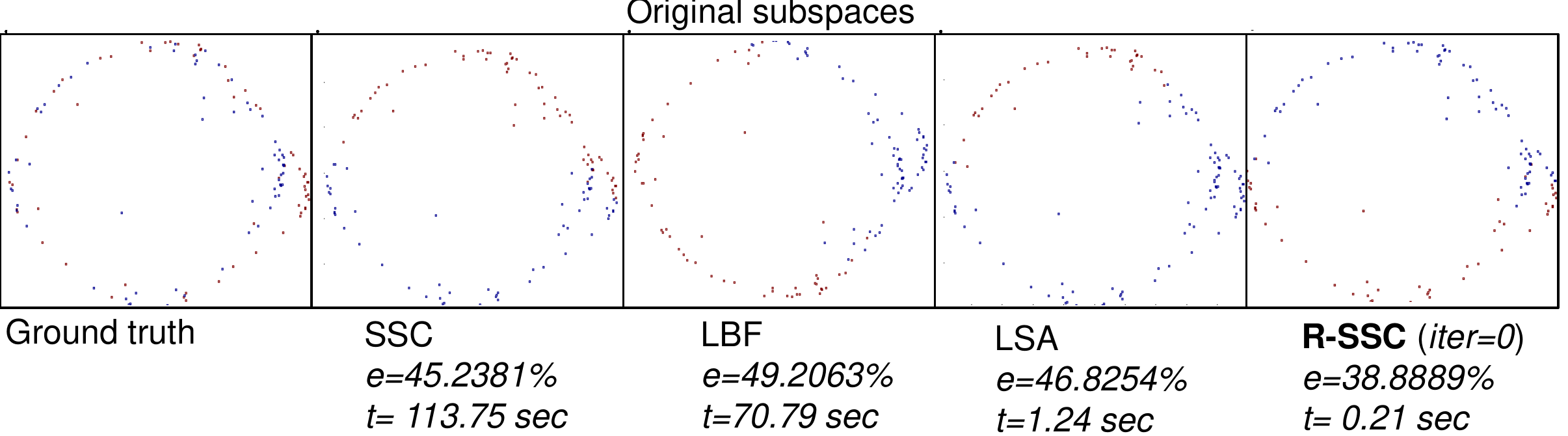}}
 \subfloat[Subjects \{2, 3\}.] {\label{fig:y23-2} \includegraphics[angle=0, height=0.15\textwidth, width=.3\textwidth]{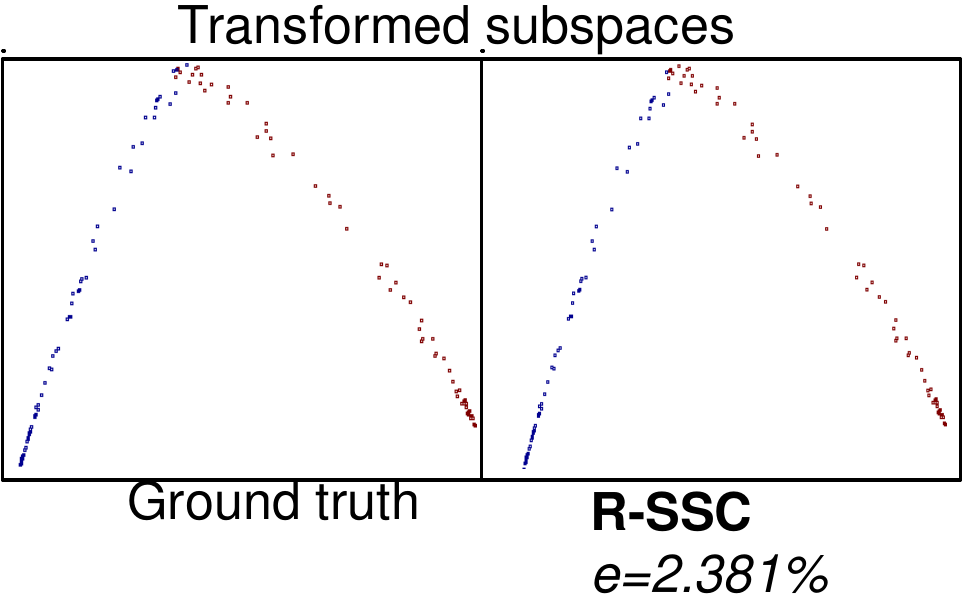} } \\
 \subfloat[Subjects \{4, 5, 6\}.] {\label{fig:y456-1} \includegraphics[angle=0, height=0.15\textwidth, width=.7\textwidth]{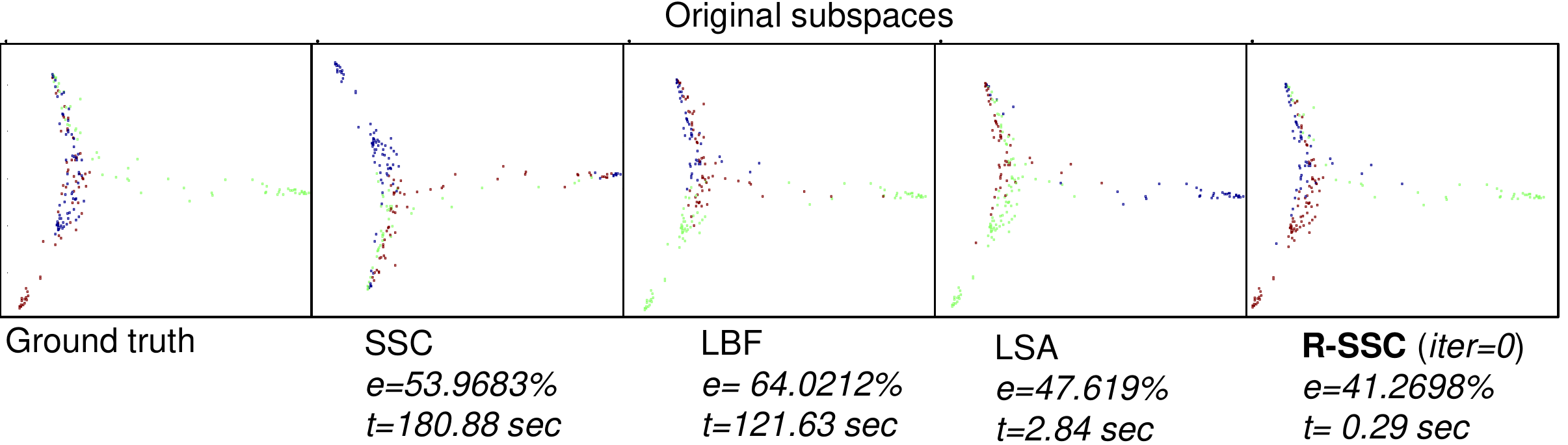}}
  \subfloat[Subjects \{4, 5, 6\}.] {\label{fig:y456-2} \includegraphics[angle=0, height=0.15\textwidth, width=.3\textwidth]{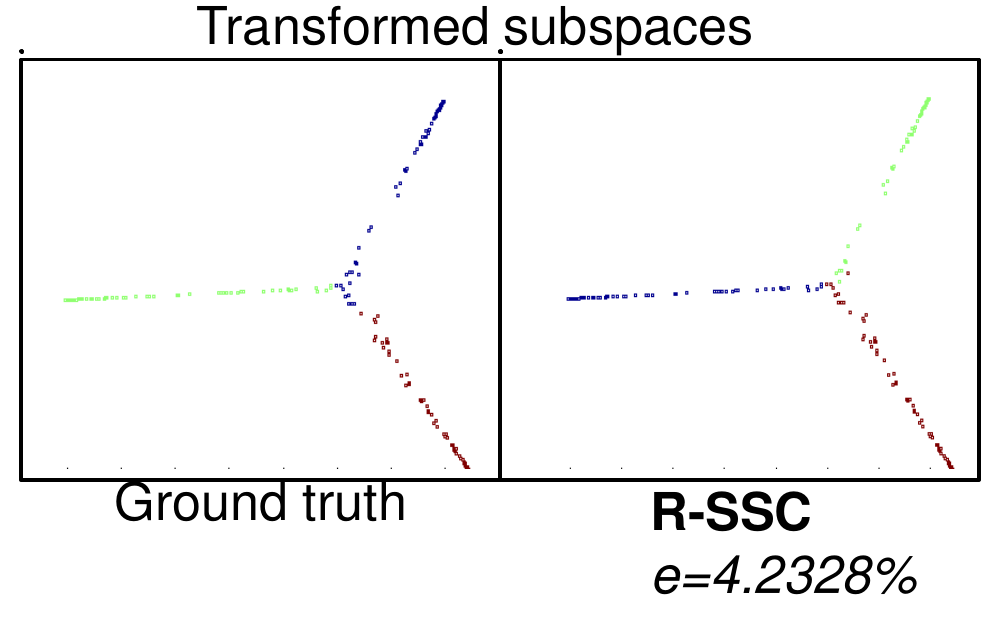}} \\
 \subfloat[Subjects \{7, 8, 9\}.] {\label{fig:y789-1} \includegraphics[angle=0, height=0.15\textwidth, width=.7\textwidth]{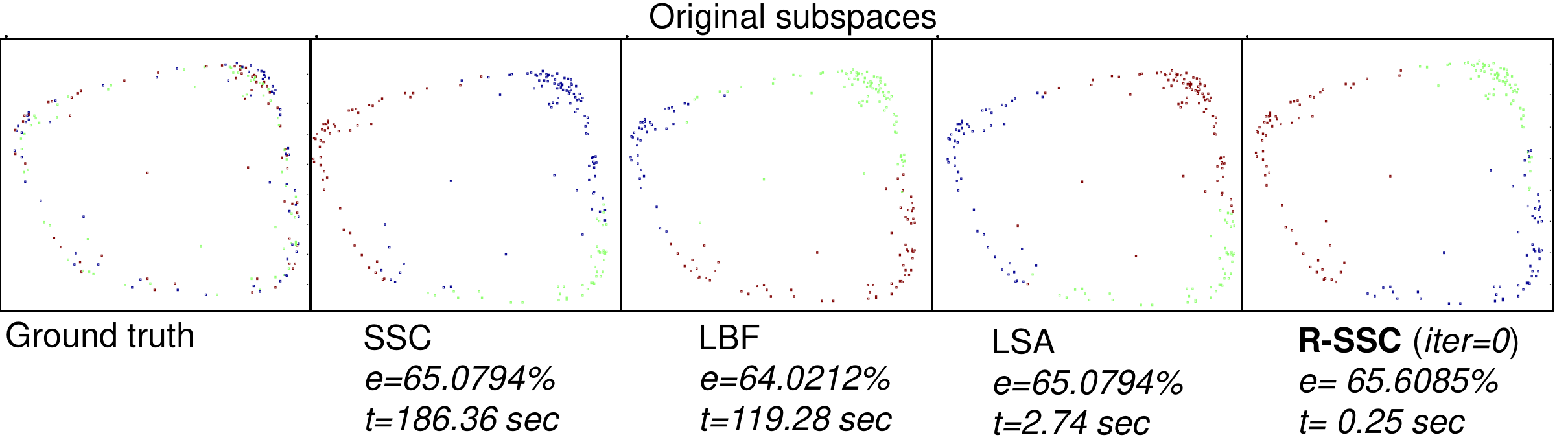}}
  \subfloat[Subjects \{7, 8, 9\}.] {\label{fig:y789-2} \includegraphics[angle=0, height=0.15\textwidth, width=.3\textwidth]{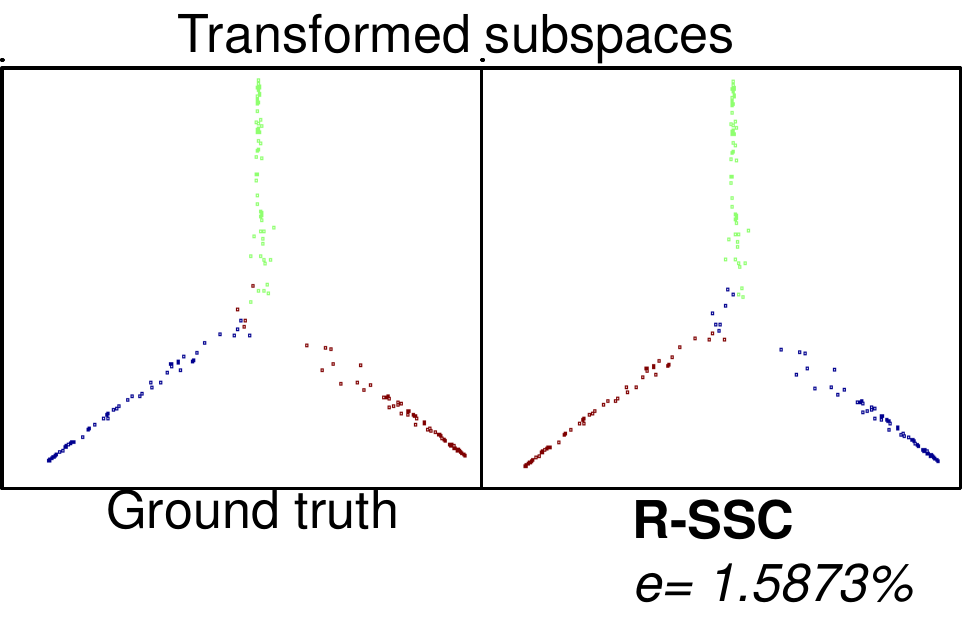}}
\caption{Misclassification rate (\emph{e})  and running time (\emph{t}) on clustering 2 and 3 subjects. The proposed R-SSC outperforms state-of-the-art methods both in accuracy and running time.
With the proposed LRSC framework, the clustering error of R-SSC is further reduced significantly. Note how the classes are clustered in clean subspaces in the transformed domain (best viewed zooming on screen).
}
\label{fig:yaleacc}
%\end{center}
\end{figure*}

\begin{table}[ht]
%\begin{center}
\centering
	\caption{{
Misclassification rate ($e\%$) on clustering different number of subjects in the Extended YaleB face dataset,
$[1:c]$ denotes the first $c$ subjects in the dataset. For all cases, the proposed LRSC method significantly outperforms state-of-the-art methods.
}}
{%\small
	\begin{tabular}{l|l|l|l|l|l|l}
	\hline
Subsets & [1:10] & [1:15] & [1:20] & [1:25] & [1:30] & [1:38]  \\
	\hline
$C$ & 10 & 15 & 20 & 25 & 30 & 38 \\
	\hline
 \hline
LSA &  78.25& 82.11 & 84.92 & 82.98 & 82.32 & 84.79    \\
LBF & 78.88 & 74.92 & 77.14 & 78.09 &  78.73&  79.53    \\
LRSC & \textbf{5.39} & \textbf{4.76} & \textbf{9.36} & \textbf{8.44} & \textbf{8.14} & \textbf{11.02}     \\
\hline
	\end{tabular}
}	
	\label{tab:yaleb2}
%\end{center}
\end{table}

{
Table~\ref{tab:yaleb2} shows misclassification rate (\emph{e}) on clustering subspaces of different number of subjects, $[1:c]$ denotes the first $c$ subjects in the extended YaleB dataset. For all cases, the proposed LRSC method significantly outperforms state-of-the-art methods.
Note that without the low-rank decomposition step in (\ref{rpca}), we obtain a misclassification rate 18.38\% for clustering all 38 subjects in the Extended YaleB dataset, which is slightly lower than the 11.02\% reported in Table~\ref{tab:yaleb2}.
Thus, pushing the subspaces apart through our learned transformation plays a major role here; and the robustness in the low-rank decomposition enhances the performance even further.
}

\begin{table}[ht]
%\begin{center}
\centering
	\caption{{
Misclassification rate ($e\%$) on clustering 38 subjects in the Extended YaleB dataset using supervised transformation learning.
The proposed transformation learning outperforms both the closed-form orthogonalizing transformation and LDA on clustering the transformed data.
}}
{%\small
	\begin{tabular}{|l|l|}
	\hline
Methods & Misclassification (\%)\\
	\hline
	\hline
orthogonalizing &  61.36   \\
LDA &   9.77  \\
Proposed &  \textbf{5.47 }  \\
\hline
	\end{tabular}
}	
	\label{tab:yaleb3}
%\end{center}
\end{table}

{
In Fig.~\ref{fig:2plane} and Fig.~\ref{fig:4line}, using synthetic examples, we previously compared our learned transformation with the closed-form orthogonalizing transformation and LDA. In Table~\ref{tab:yaleb3}, we further compare three transformations using real data. We perform supervised transformation learning on all 38 subjects in the Extended YaleB dataset using three different transformation learning algorithms, and then perform subspace clustering on the transformed data. The proposed transformation learning significantly outperforms the other two methods.
}

\subsection{Application to Motion Segmentation}

\begin{table}[ht]
%\begin{center}
\centering
	\caption{ { Misclassification rate ($e\%$) on two motions and three motions segmentation in the Hopkins 155 dataset.
 As shown in \cite{SubspaceClustering, SLBF}, the SSC method significantly outperforms all previous state-of-the-art methods on this dataset.
The proposed LRSC shows comparable results to SSC for two motions and outperforms SSC for three motions. Note that our method is orders of magnitude faster than SSC.} }
{%\small
	\begin{tabular}{l|ll|ll|ll|ll}
	\hline
 & \multicolumn{2}{c|}{Check}  & \multicolumn{2}{c|}{Traffic} & \multicolumn{2}{c|}{Articulated} & \multicolumn{2}{c}{All}  \\
 	\hline
 & Mean& Median & Mean & Median& Mean& Median& Mean & Median\\
	\hline
 	\hline
 \multicolumn{9}{l}{\textbf{2-motion}}\\
  	\hline
LSA & 2.57& 0.27 &5.43 & 1.48 & 4.10& 1.22&3.45 & 0.59  \\
LBF & 1.59& 0 &0.20  &0 &0.80 & 0& 1.16 & 0\\
SSC & 1.12& 0 & 0.02 & 0& 0.62& 0& \textbf{0.82} & 0\\
LRSC& 1.19 & 0 & 0.23 & 0 & 0.88 & 0 & 0.92 & 0\\
 	\hline
 \multicolumn{9}{l}{\textbf{3-motion}}\\
  	\hline
LSA & 5.80& 1.77 & 25.07 &23.79 & 7.25& 7.25& 9.73 &2.33 \\
LBF & 4.57& 0.94 & 0.38 &0 & 2.66 & 2.66& 3.63& 0.64  \\
SSC & 2.97& 0.27 & 0.58 & 0& 1.42& 0& 2.45 & 0.2\\
LRSC&1.59 & 0 & 0.32 & 0 & 1.60& 1.60 & \textbf{1.34} & 0\\
 \hline
	\end{tabular}
}	
	\label{tab:Hopkins155}
%\end{center}
\end{table}

{

The Hopkins 155 dataset consists of three types of videos: checker, traffic and articulated, and  120 of the videos have two motions and 35 of the videos have three motions.
The main task is to segment a video sequence of multiple rigidly moving objects into multiple spatiotemporal regions that correspond to different motions in the scene.
This motion dataset contains much cleaner subspace data than the digits and faces data evaluated above. To enable a fair comparison, we project the data into a lower dimensional subspace using PCA as explained in \cite{SubspaceClustering, SLBF}.
 Results on other comparing methods are taken from \cite{SubspaceClustering}.
 As shown in \cite{SubspaceClustering, SLBF}, the SSC method significantly outperforms all previous state-of-the-art methods on this dataset.
 From Table~\ref{tab:Hopkins155}, we can see that our method shows comparable results to SSC for two motions and outperforms SSC for three motions.  Note that our method is orders of magnitude faster than SSC as discussed  earlier.
}

\subsection{Application to Face Recognition across Illumination}

For the Extended YaleB dataset, we adopt a similar setup as described in \cite{lcksvd,Zhang10}.  We split the dataset into two halves by randomly selecting 32 lighting conditions for training, and the other half for testing.
We learn a global low-rank transformation matrix from the training data.

We report recognition accuracies in Table~\ref{tab:yaleacc}. We make the following observations. First,
 the recognition accuracy is increased from $91.77\%$ to $99.10\%$ by simply applying the learned transformation matrix to the original face images.
 Second, the best accuracy is obtained by first recovering the low-rank subspace for each subject, e.g., the third row in Fig.~\ref{fig:YaleShareTrain}. Then, each transformed testing face,  e.g., the second row in Fig.~\ref{fig:YaleShareTest}, is sparsely decomposed over the low-rank subspace of each subject through OMP, and classified to the subject with the minimal reconstruction error. A sparsity value 10 is used here for OMP.
 As shown in Fig.~\ref{fig:YaleSTavglowrank}, the low-rank representation for each subject shows reduced variations caused by illumination.
 Third, the global transformation performs better here than class-based transformations, which can be due to the fact that illumination in this dataset varies in a globally coordinated way across subjects.
   Last but not least, our method outperforms state-of-the-art sparse representation based face recognition methods.

\begin{table}[ht]
%\begin{center}
\centering
	\caption{Recognition accuracies (\%) under illumination variations for the Extended YaleB dataset. The recognition accuracy is increased from $91.77\%$ to $99.10\%$ by simply applying the learned low-rank transformation (LRT) matrix to the original face images.}
{%\small
	\begin{tabular}{|l|l|}
	\hline
Method & Accuracy (\%) \\
	\hline
 \hline
D-KSVD \cite{Zhang10} & 94.10 \\
LC-KSVD \cite{lcksvd} & 96.70 \\
SRC \cite{Wright09} & 97.20 \\
\hline
\hline
Original+NN & 91.77 \\
Class LRT+NN & 97.86 \\
Class LRT+OMP  & 92.43 \\
Global LRT+NN & 99.10 \\
Global LRT+OMP  & \textbf{99.51} \\
\hline
	\end{tabular}
}	
	\label{tab:yaleacc}
%\end{center}
\end{table}

\begin{figure*} [t]
\centering
\subfloat[Low-rank decomposition of globally transformed training samples] {\label{fig:YaleShareTrain} \includegraphics[angle=0, height=0.29\textwidth, width=.9\textwidth]{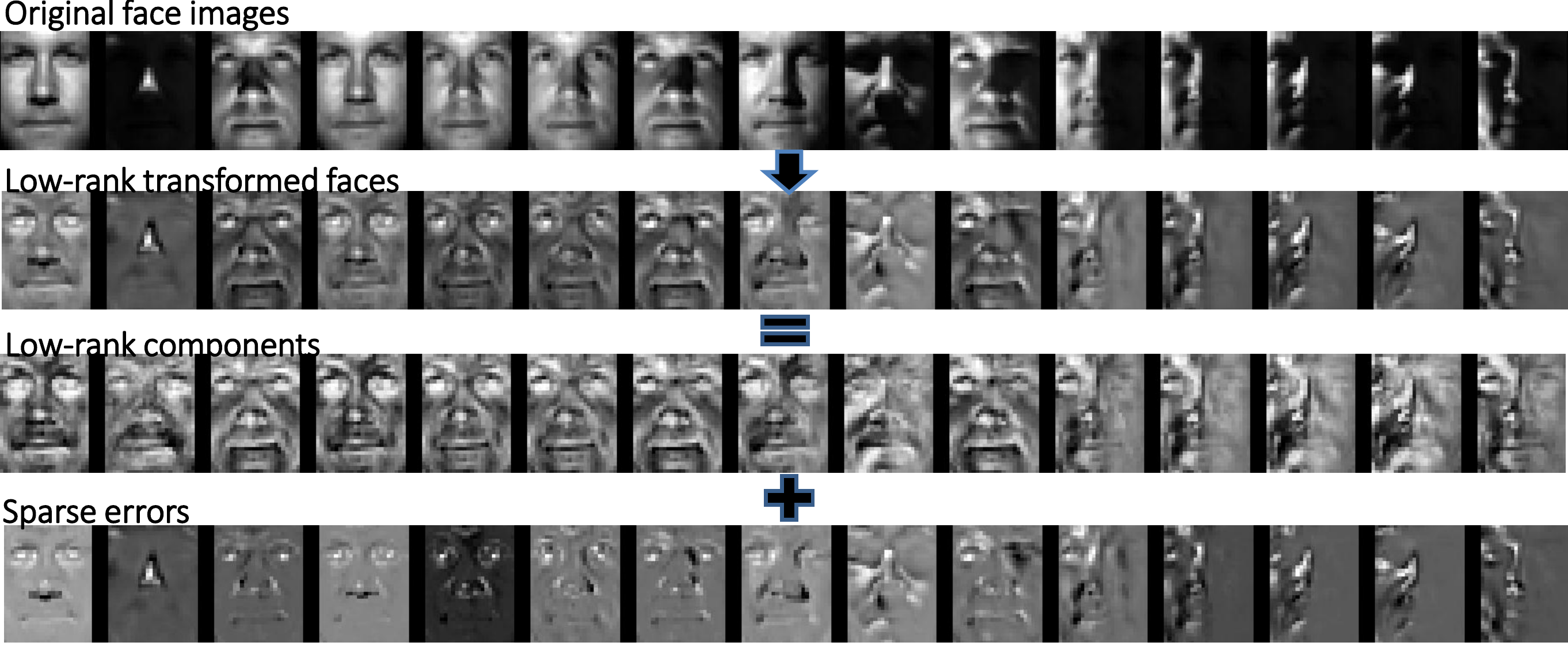} \hspace{0pt}} \\
\subfloat[Globally transformed testing samples] {\label{fig:YaleShareTest} \includegraphics[angle=0, height=0.15\textwidth, width=.9\textwidth]{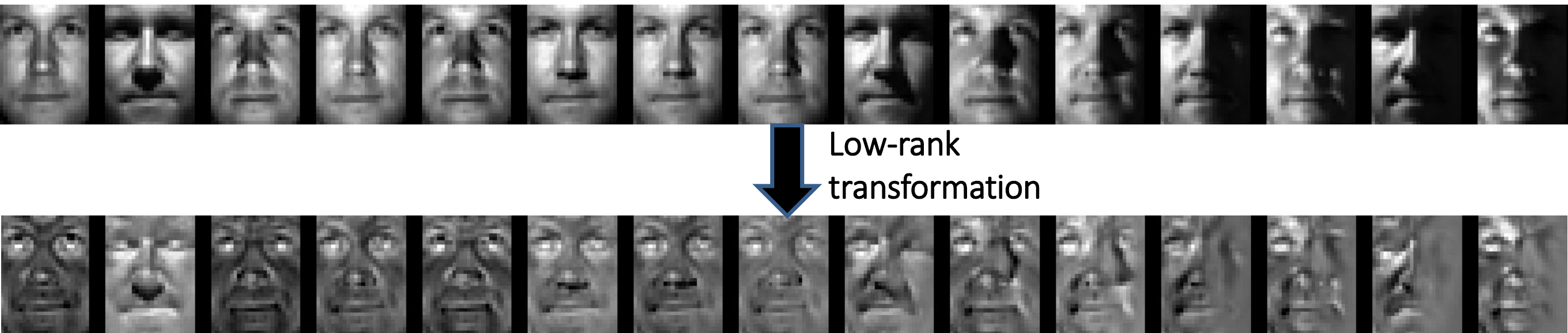}} \\
\subfloat[Mean low-rank components for subjects in the training data] {\label{fig:YaleSTavglowrank} \includegraphics[angle=0, height=0.08\textwidth, width=.9\textwidth]{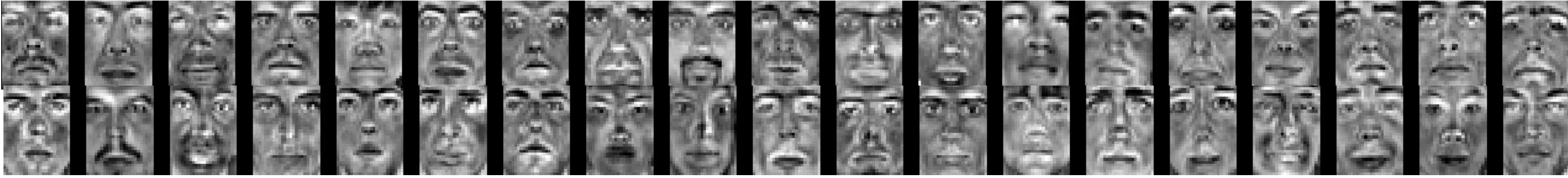}}
\caption{Face recognition across illumination using global low-rank transformation.}
\label{fig:rec_illum}
%\end{center}
\end{figure*}

\subsection{Application to Face Recognition across Pose}

\begin{table}[ht]
%\begin{center}
\centering
	\caption{Recognition accuracies (\%) under pose variations for the CMU PIE dataset.}
{%\small
	\begin{tabular}{|l|l|l|l|}
	\hline
Method & Frontal & Side & Profile \\
 & (c27) & (c05) & (c22) \\
	\hline
 \hline
SMD \cite{smd} & 83 & 82 & 57 \\
\hline
\hline
Original+NN & 39.85 & 37.65 & 17.06 \\
Original(crop+flip)+NN & 44.12 & 45.88 & 22.94 \\
Class LRT+NN &  98.97 & 96.91& 67.65\\
Class LRT+OMP  & \textbf{100} & \textbf{100} & \textbf{67.65}\\
Global LRT+NN & 97.06 & 95.58& 50\\
Global LRT+OMP  & 100 & 98.53& 57.35\\
\hline
	\end{tabular}
}	
	\label{tab:pieacc}
%\end{center}
\end{table}

\begin{figure*} [t]
\centering
\subfloat[Low-rank decomposition of class-based transformed training samples for \emph{subject3} ] {\label{fig:PIEClassTrain01} \includegraphics[angle=0, height=0.23\textwidth, width=.5\textwidth]{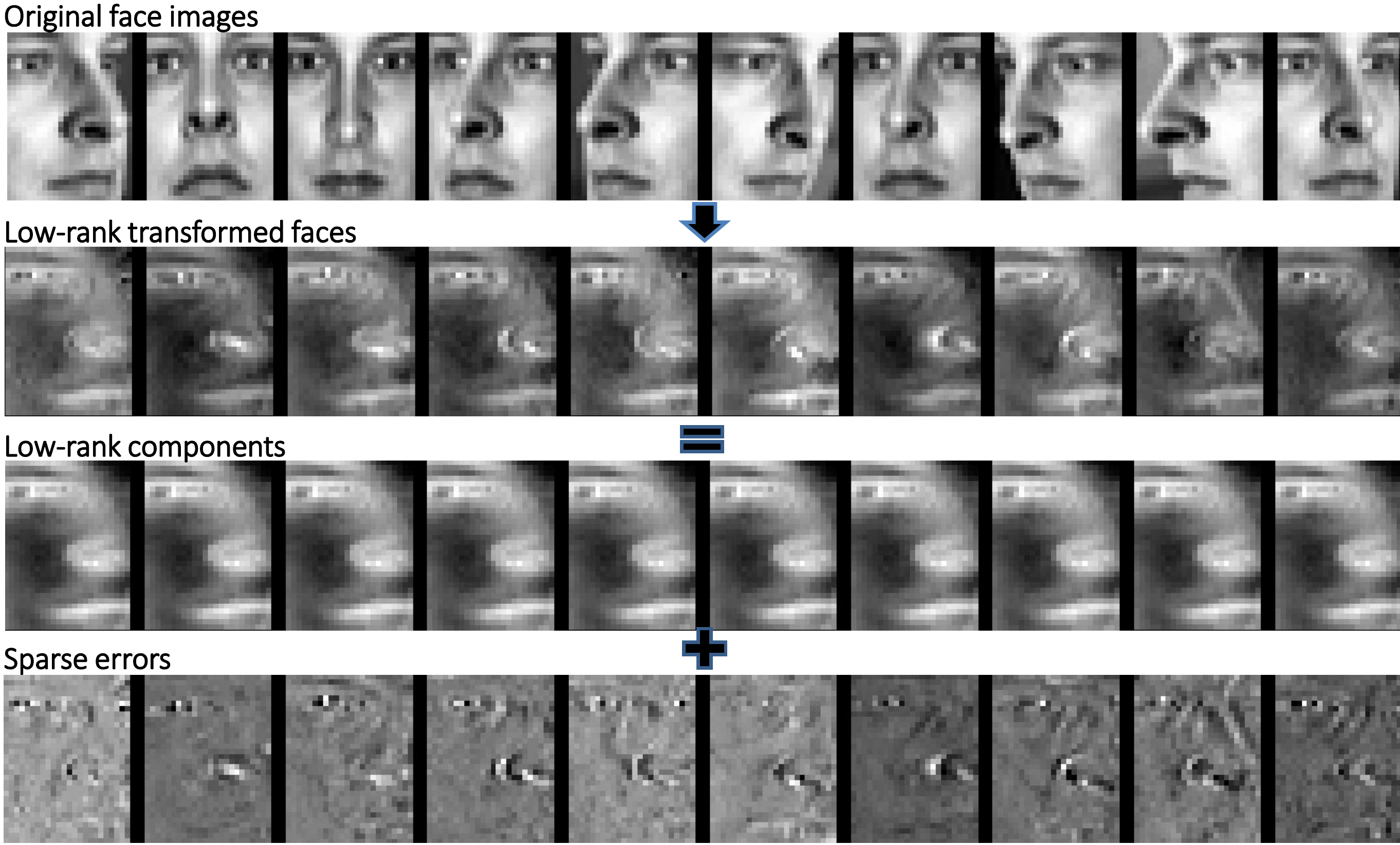} \hspace{0pt}}
\subfloat[Low-rank decomposition of class-based transformed training samples for \emph{subject1} ] {\label{fig:PIEClassTrain02} \includegraphics[angle=0, height=0.23\textwidth, width=.5\textwidth]{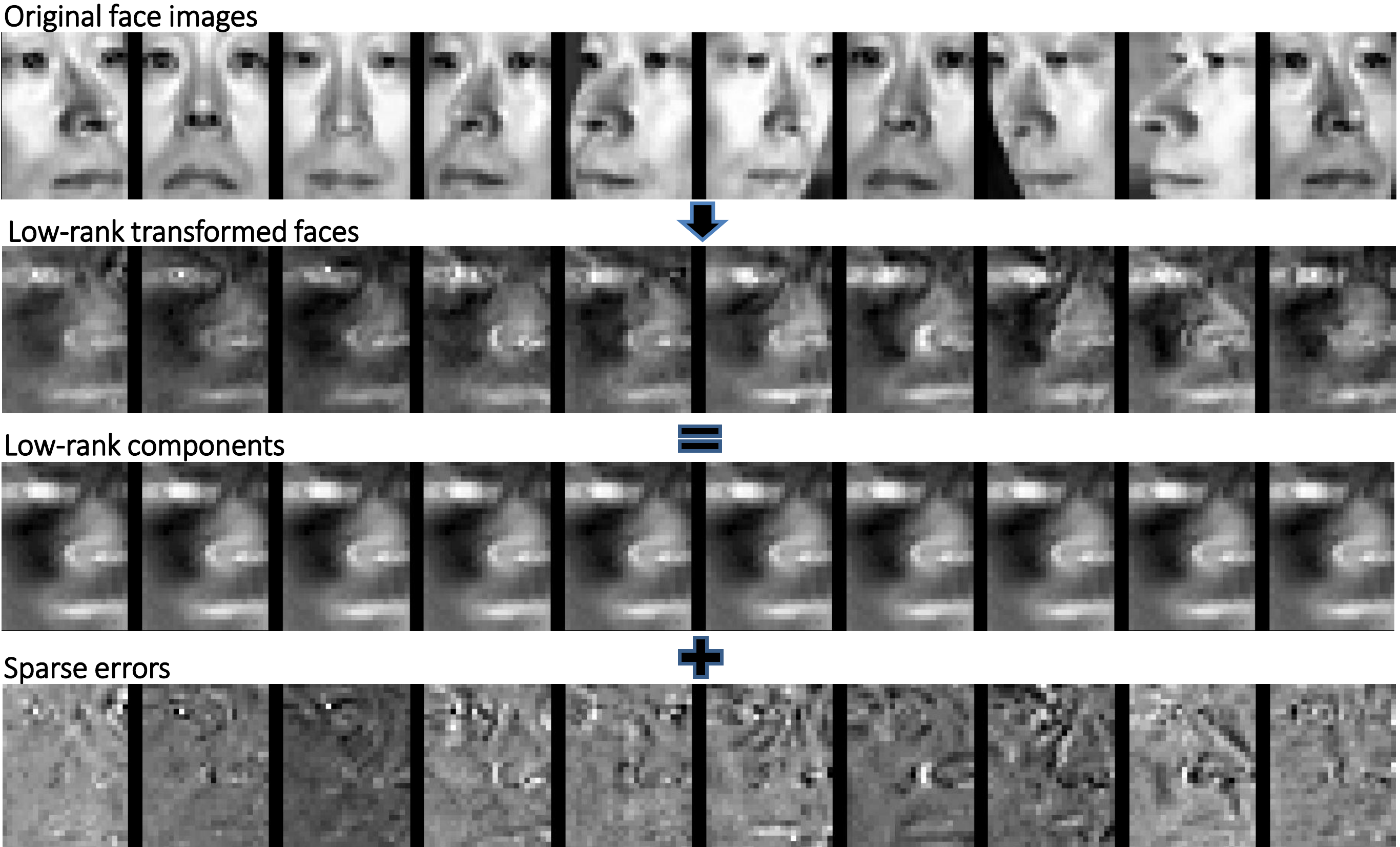} \hspace{0pt}} \\
\subfloat[class-based transformed testing samples for \emph{subject3}] {\label{fig:PIEClassTest01} \includegraphics[angle=0, height=0.06\textwidth, width=.4\textwidth]{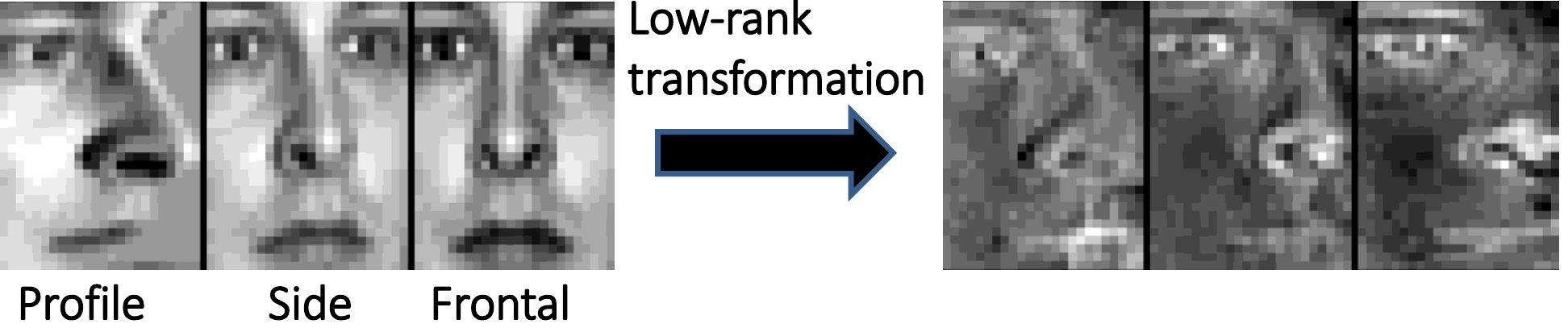} \hspace{10pt}}
\subfloat[class-based transformed testing samples for \emph{subject1}] {\label{fig:PIEClassTest02} \includegraphics[angle=0, height=0.06\textwidth, width=.4\textwidth]{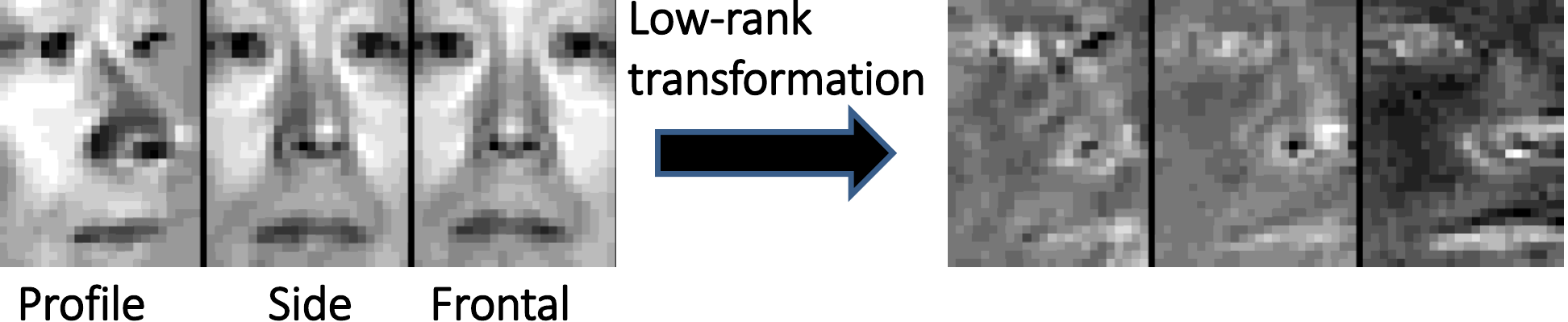}}
\caption{Face recognition across pose using class-based low-rank transformation. Note, for example in (c) and (d), how the learned transform reduces the pose-variability.}
\label{fig:rec_pose}
%\end{center}
\end{figure*}

We adopt the similar setup as described in \cite{smd} to enable the comparison. In this experiment, we classify 68 subjects in three poses, frontal (c27), side (c05), and profile (c22), under lighting condition 12. We use the remaining poses as the training data.

For this example, we learn a class-based low-rank transformation matrix per subject from the training data.
It is noted that the goal is to learn a transformation matrix to help in the classification, which may not necessarily correspond to the real geometric transform.
Table~\ref{tab:pieacc} shows the face recognition accuracies under pose variations for the CMU PIE dataset  (we applied the crop-and-flip step discussed in Fig.~\ref{fig:overview}.). We make the following observations.
First, the recognition accuracy is dramatically increased after applying the learned transformations.
Second, the best accuracy is obtained by recovering the low-rank subspace for each subject, e.g., the third row in Fig.~\ref{fig:PIEClassTrain01} and Fig.~\ref{fig:PIEClassTrain02}. Then, each transformed testing face,  e.g.,  Fig.~\ref{fig:PIEClassTest01} and Fig.~\ref{fig:PIEClassTest02}, is sparsely decomposed over the low-rank subspace of each subject through OMP, and classified to the subject with the minimal reconstruction error, Section~\ref{sec:rec}.
Third, the class-based transformation performs better than the global transformation in this case. The choice between these two settings is data dependent.
Last but not least, our method outperforms SMD, which the best of our knowledge, reported the best recognition performance in such experimental setup. However, SMD is an unsupervised method, and the proposed method requires training, still illustrating how a simple learned transform (note that applying it to the data at testing time if virtually free of cost), can significantly improve performance.

\subsection{Application to Face Recognition across Illumination and Pose}

\begin{figure*} [t]
\centering
 \subfloat[Pose c02] {\label{fig:pose02} \includegraphics[angle=0, height=0.23\textwidth, width=.25\textwidth]{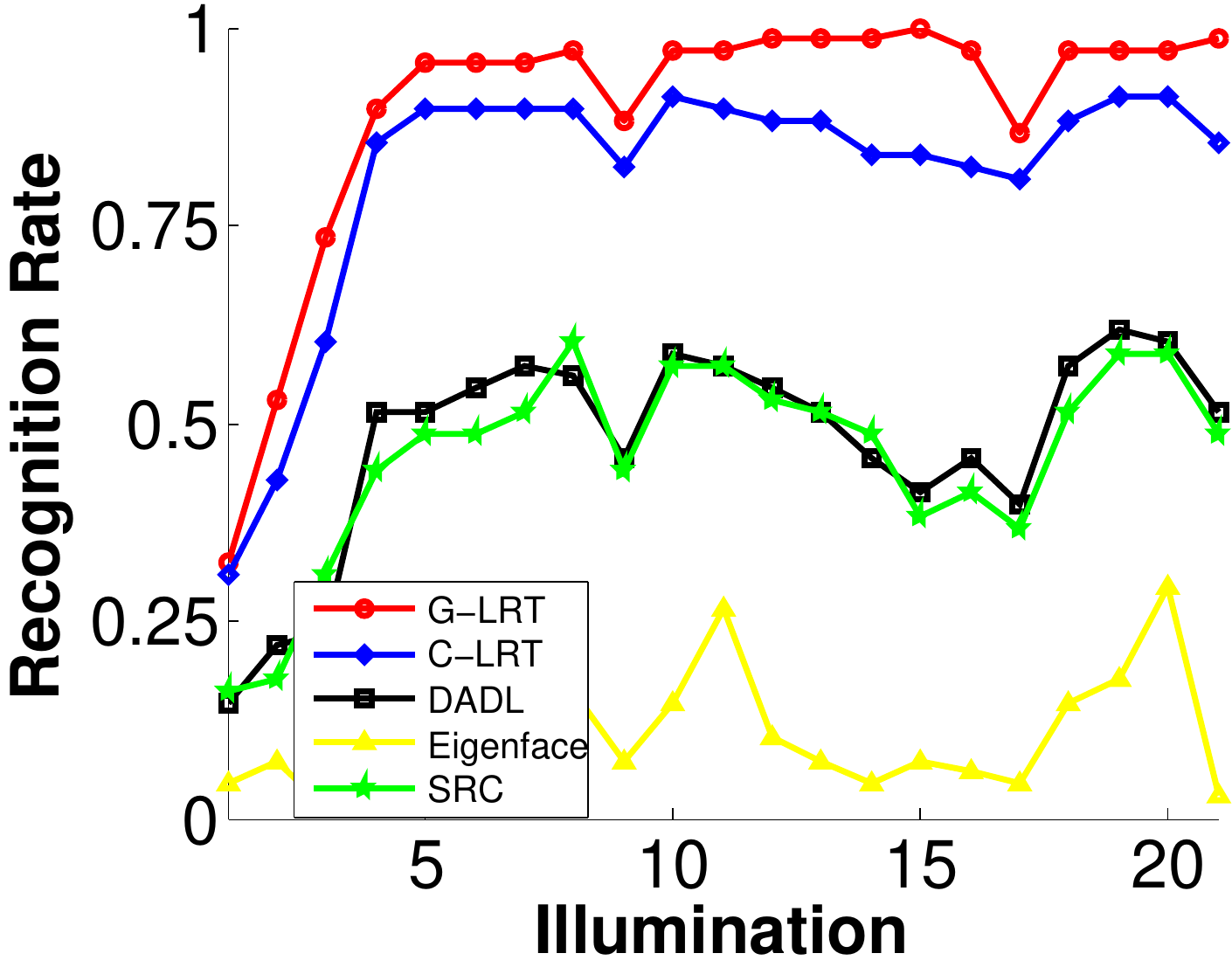} \hspace{0pt}}
  \subfloat[Pose c05] {\label{fig:pose05} \includegraphics[angle=0, height=0.23\textwidth, width=.25\textwidth]{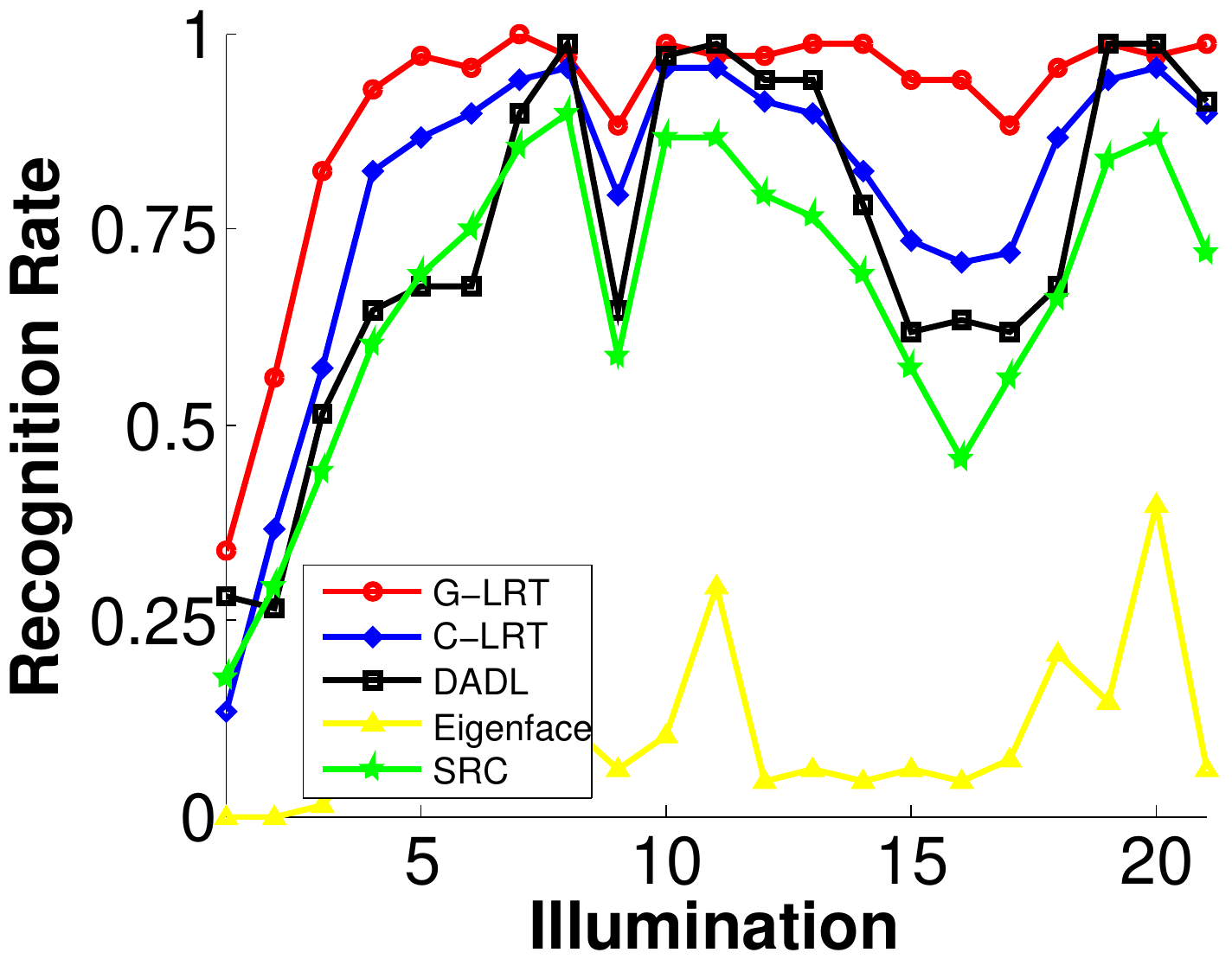}}
    \subfloat[Pose c29] {\label{fig:pose29} \includegraphics[angle=0, height=0.23\textwidth, width=.25\textwidth]{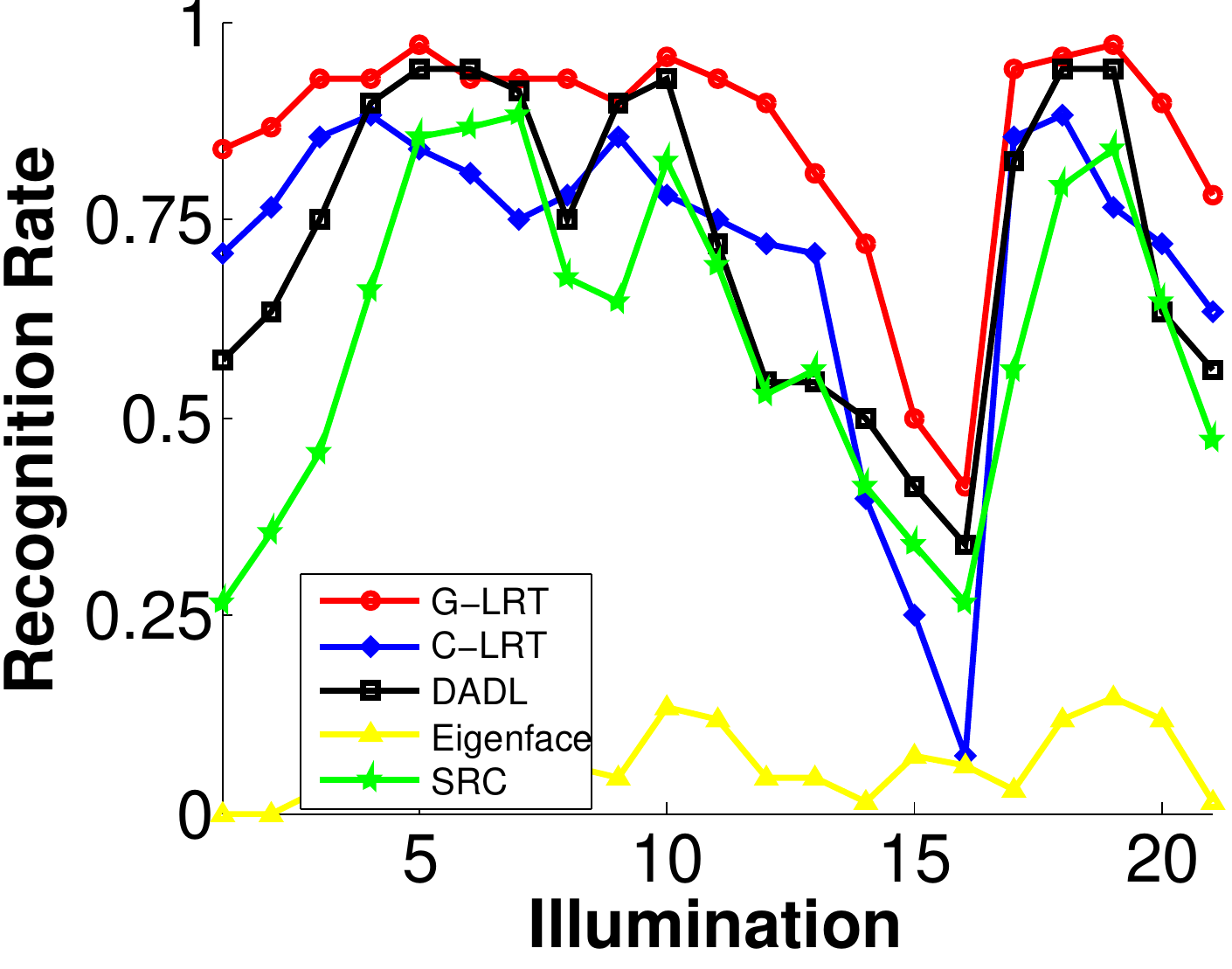}}
  \subfloat[Pose c14] {\label{fig:pose14} \includegraphics[angle=0, height=0.23\textwidth, width=.25\textwidth]{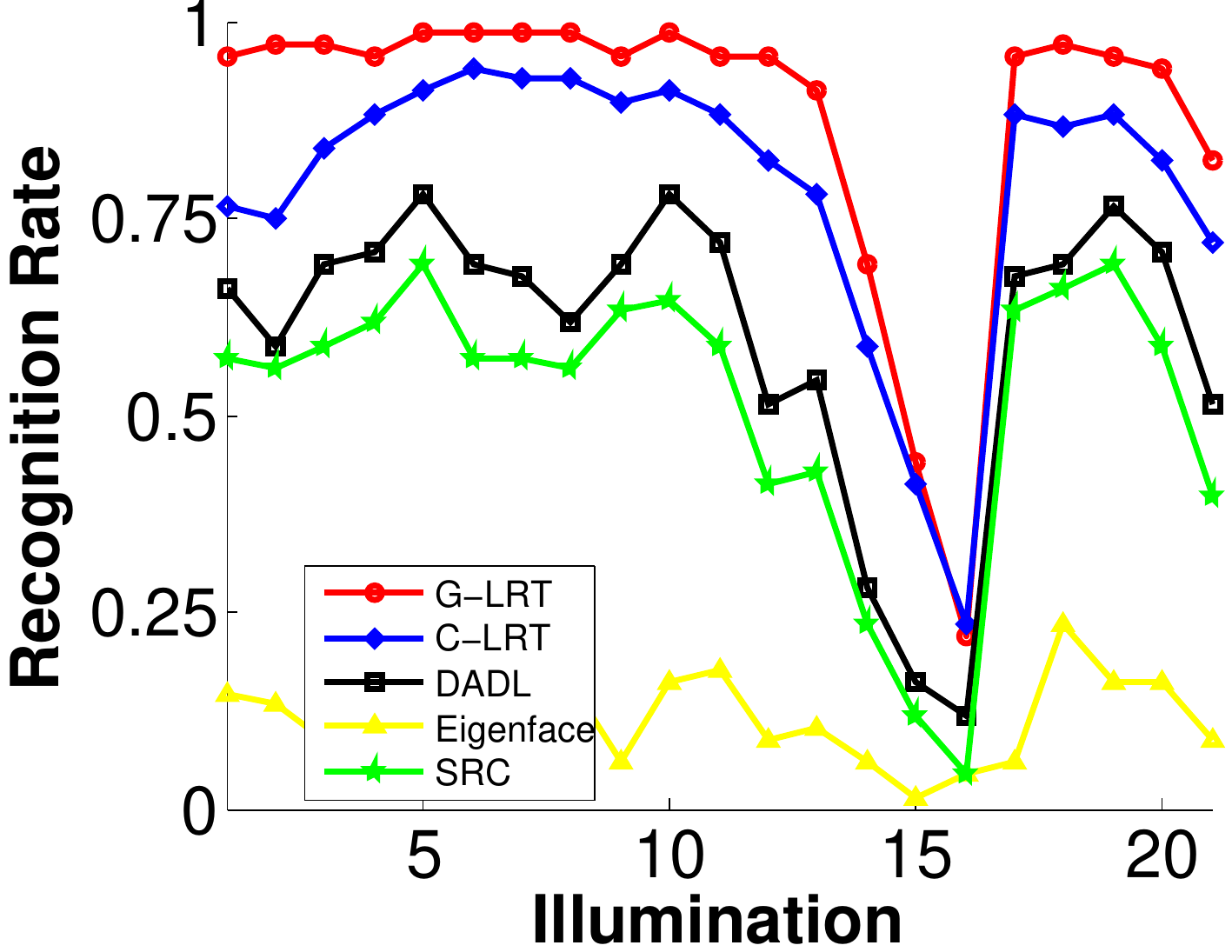}}
\caption{Face recognition accuracy under combined pose and illumination variations on the CMU PIE dataset. The proposed methods are denoted as \emph{G-LRT} in color red and \emph{C-LRT} in color blue. The proposed methods significantly outperform the comparing methods, especially for extreme poses c02 and c14.}
\label{fig:pieacc-eccv}
%\end{center}
\end{figure*}

\begin{figure*} [t]
\centering
\subfloat[Globally transformed testing samples for \emph{subject1} ] {\label{fig:PIEShareTest01eccv} \includegraphics[angle=0, height=0.14\textwidth, width=.5\textwidth]{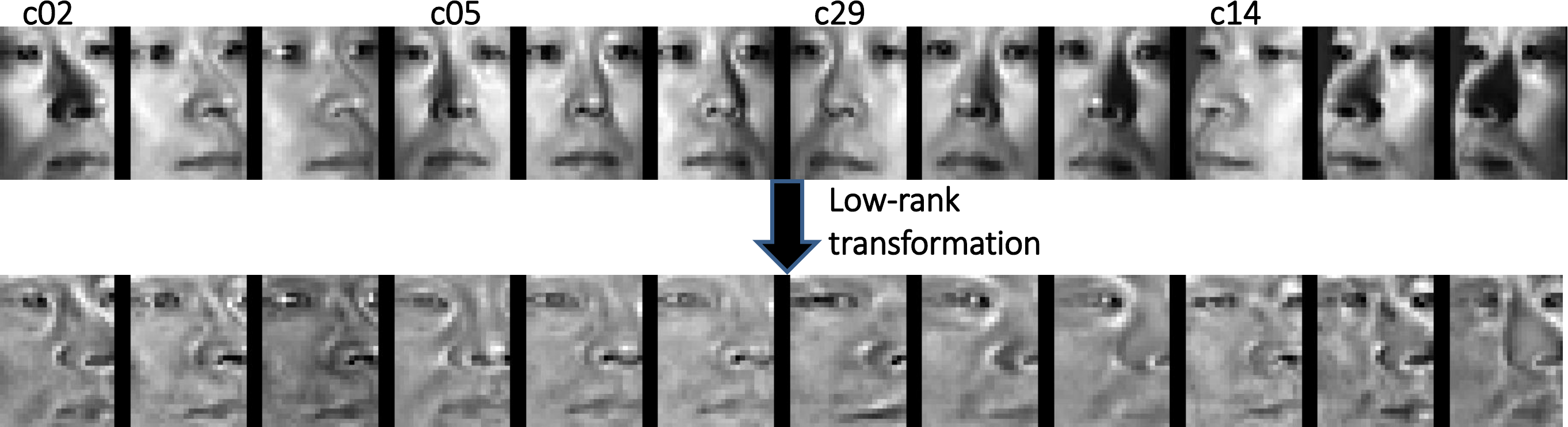} \hspace{0pt}}
\subfloat[Globally transformed testing samples for \emph{subject2} ] {\label{fig:PIEShareTest02eccv} \includegraphics[angle=0, height=0.14\textwidth, width=.5\textwidth]{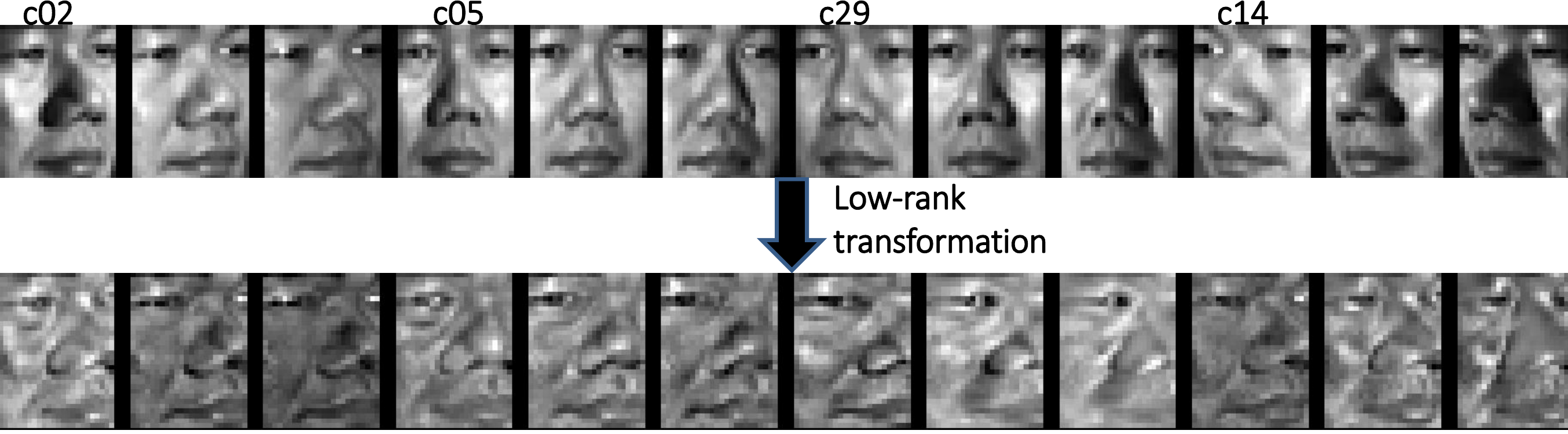}}
\caption{Face recognition under combined pose and illumination variations using global low-rank transformation.}
\label{fig:PIEShareTesteccv}
%\end{center}
\end{figure*}

To enable the comparison with \cite{dadl}, we adopt their setup for
face recognition under combined pose and illumination variations for the CMU PIE dataset.
We use 68 subjects in 5 poses, c22, c37, c27, c11 and c34,  under 21 illumination conditions for training; and classify 68 subjects in 4 poses, c02, c05, c29 and c14, under 21 illumination conditions.

Three face recognition methods are adopted for comparisons: Eigenfaces \cite{eigenface},
SRC \cite{Wright09}, and DADL \cite{dadl}.
{ SRC and DADL are both state-of-the-art sparse representation  methods for face recognition, and DADL adapts sparse dictionaries to the actual visual domains.}
As shown in Fig.~\ref{fig:pieacc-eccv}, the proposed methods, both the global LRT (G-LRT) and class-based LRT (C-LRT),  significantly outperform the comparing methods, especially for extreme poses c02 and c14.
Some testing examples using a global transformation are shown in Fig.~\ref{fig:PIEShareTesteccv}. We notice that the transformed faces for each subject exhibit reduced variations caused by pose and illumination.

\subsection{Discussion on the Size of the Transformation Matrix $\mathbf{T}$}
\label{sec:compress}

In the experiments presented above,  we learned a square linear transformation. For example, if images are resized to $16 \times 16$, the learned subspace transformation $\mathbf{T}$ is of size $256 \times 256$.
If we learn a transformation of size $r \times 256$ with $r<256$, we enable dimension reduction while performing subspace transformation (feature learning). Through experiments, we notice that the peak clustering accuracy is usually obtained when $r$ is smaller than the dimension of the ambient space.
For example, in Fig.~\ref{fig:yaleacc}, through exhaustive search for the optimal $r$, we observe the misclassification rate reduced from $2.38\%$ to $0\%$ for subjects \{2, 3\} at $r=96$, and from $4.23\%$ to $0\%$ for subjects \{4, 5, 6\} at $r=40$.
As discussed before, this provides a framework to sense for clustering and classification, connecting the work here presented with the extensive literature on compressed sensing, and in particular for sensing design, e.g., \cite{CS1}.
We plan to study in detail the optimal size of the learned transformation matrix for subspace clustering and classification, including its potential connection with the number of subspaces in the data, and further investigate such connections with compressive sensing.

\section{Conclusion}
\label{sec:con}

We introduced a subspace low-rank transformation approach for subspace clustering and classification.
Using matrix rank as the optimization criteria,  via its nuclear norm convex surrogate,  we learn a subspace transformation that reduces variations within the subspaces, and increases separations between the subspaces.
We demonstrated that the proposed approach significantly outperforms state-of-the-art methods for subspace clustering and classification, and provided some theoretical support to these experimental results.

Numerous venues of research are opened by the framework here introduced. At the theoretical level, extending the analysis to the noisy case is needed. Furthermore, understanding the virtues of the global vs the class-dependent transform is both important and interesting, as it is the study of the framework in its compressed dimensionality form.
Beyond this, considering the proposed approach as a feature extraction technique, its combination with other successful clustering and classification techniques is the subject of current research.

\appendix

\section{Proof of Theorem \ref{nuclear_ineq}}
\label{sec:nuclear_ineq}

\begin{proof} We know that (\cite{nuclear-min})
\begin{align}
||\mathbf{A}||_* =  \underset{\substack{\mathbf{U}, \mathbf{V} \\ \mathbf{A=UV'}}} \min \frac{1}{2}( ||\mathbf{U}||^2_F + ||\mathbf{V}||^2_F). \nonumber
\end{align}
 We denote $\mathbf{U_A}$ and $\mathbf{V_A}$
 the matrices that achieve the minimum; same for $\mathbf{B}$, $\mathbf{U_B}$ and $\mathbf{V_B}$; and same for the concatenation $\mathbf{[A,B]}$, $\mathbf{U_{[A,B]}}$ and $\mathbf{V_{[A,B]}}$.
\noindent We then have
\begin{align} \nonumber
||\mathbf{A}||_*  &= \frac{1}{2}( ||\mathbf{U_A}||^2_F + ||\mathbf{V_A}||^2_F), \\ \nonumber
||\mathbf{B}||_*  &= \frac{1}{2}( ||\mathbf{U_B}||^2_F + ||\mathbf{V_B}||^2_F) .\\ \nonumber
\end{align}
The matrices $[\mathbf{U_A}, \mathbf{U_B}]$ and $[\mathbf{V_A}, \mathbf{V_B}]$ obtained by concatenating the matrices that achieve the minimum for $\mathbf{A}$ and $\mathbf{B}$ when computing their nuclear norm, are not necessarily the ones that achieve the corresponding  minimum in the nuclear norm computation of the concatenation matrix  $[\mathbf{A}, \mathbf{B}]$. Thus, together with $||\mathbf{[A,B]}||_F^2  = ||\mathbf{A}||_F^2 + ||\mathbf{B}||_F^2$, we have
\begin{align} \nonumber
||\mathbf{[A,B]}||_*  &= \frac{1}{2}( ||\mathbf{U_{[A,B]}}||^2_F + ||\mathbf{V_{[A,B]}}||^2_F) \\ \nonumber
& \le \frac{1}{2}( ||[\mathbf{U_A}, \mathbf{U_B}]||^2_F + ||[\mathbf{V_A},\mathbf{V_B}]||^2_F) \\ \nonumber
&= \frac{1}{2}( ||\mathbf{U_A}||^2_F + ||\mathbf{U_B}||^2_F + ||\mathbf{V_A}||^2_F+ ||\mathbf{V_B}||^2_F) \\ \nonumber
&= \frac{1}{2}( ||\mathbf{U_A}||^2_F + ||\mathbf{V_A}||^2_F) + \frac{1}{2}(||\mathbf{U_B}||^2_F+||\mathbf{V_B}||^2_F) \\ \nonumber
&= ||\mathbf{A}||_* + ||\mathbf{B}||_* . \\ \nonumber
\end{align}
\end{proof}

\section{Proof of Theorem \ref{nuclear_eq}}
\label{sec:nuclear_eq}

\begin{proof}
We perform the singular value decomposition of $\mathbf{A}$ and $\mathbf{B}$ as
\begin{align}  \nonumber
\mathbf{A} &= [\mathbf{U_{A1}} \mathbf{U_{A2}}] \begin{bmatrix}  \mathbf{\Sigma_A} & 0 \\  0& 0 \end{bmatrix} [\mathbf{V_{A1}} \mathbf{V_{A2}}]', \\ \nonumber
\mathbf{B} &= [\mathbf{U_{B1}} \mathbf{U_{B2}}] \begin{bmatrix}  \mathbf{\Sigma_B} & 0 \\  0& 0 \end{bmatrix} [\mathbf{V_{B1}} \mathbf{V_{B2}}]', \\ \nonumber
\end{align}
where the diagonal entries of $\mathbf{\Sigma_A}$ and $\mathbf{\Sigma_B}$ contain non-zero singular values. We have
\begin{align}  \nonumber
\mathbf{AA'} &= [\mathbf{U_{A1}} \mathbf{U_{A2}}] \begin{bmatrix}  \mathbf{\Sigma_A}^2 & 0 \\  0& 0 \end{bmatrix} [\mathbf{U_{A1}} \mathbf{U_{A2}}]', \\ \nonumber
\mathbf{BB'} &= [\mathbf{U_{B1}} \mathbf{U_{B2}}] \begin{bmatrix}  \mathbf{\Sigma_B}^2 & 0 \\  0& 0 \end{bmatrix} [\mathbf{U_{B1}} \mathbf{U_{B2}}]'. \\ \nonumber
\end{align}

\noindent The column spaces of $\mathbf{A}$ and $\mathbf{B}$ are considered to be orthogonal, i.e., $\mathbf{U_{A1}}'\mathbf{U_{B1}}=0$. The above can be written as
\begin{align}  \nonumber
\mathbf{AA'} &= [\mathbf{U_{A1}} \mathbf{U_{B1}}] \begin{bmatrix}  \mathbf{\Sigma_A}^2 & 0 \\  0& 0 \end{bmatrix} [\mathbf{U_{A1}} \mathbf{U_{B1}}]', \\ \nonumber
\mathbf{BB'} &= [\mathbf{U_{A1}} \mathbf{U_{B1}}] \begin{bmatrix}  0 & 0 \\  0& \mathbf{\Sigma_B}^2 \end{bmatrix} [\mathbf{U_{A1}} \mathbf{U_{B1}}]'. \\ \nonumber
\end{align}

\noindent Then, we have
\begin{align}  \nonumber
\mathbf{[A,B][A,B]'} &= \mathbf{AA'+BB'} = [\mathbf{U_{A1}} \mathbf{U_{B1}}] \begin{bmatrix}  \mathbf{\Sigma_A}^2 & 0 \\  0& \mathbf{\Sigma_B}^2 \end{bmatrix} [\mathbf{U_{A1}} \mathbf{U_{B1}}]'. \\ \nonumber
\end{align}
The nuclear norm $||\mathbf{A}||_*$ is the sum of the square root of the singular values of $\mathbf{A}\mathbf{A}'$. Thus, $||\mathbf{[A,B]}||_*  = ||\mathbf{A}||_* + ||\mathbf{B}||_*.$
\end{proof}

\section{Projected Subgradient Learning Algorithm}
\label{gradesc}

{

We use a simple projected subgradient method to search for the transformation matrix $\mathbf{T}$ that minimizes (\ref{nuclear_obj}). Before describing it, we should note that the problem is non-differentiable and non-convex, and it deserves a proper study for efficient optimization, keeping in mind that the development of more advanced optimization techniques will just further improve the performance of the proposed framework. We selected a simple subgradient approach since the goal of this paper is to present the framework, and already this simple optimization leads to very fast convergence and excellent performance as detailed in Section~\ref{sec:expr}, significant improvements in performance when compared to state-of-the-art.

To  minimize (\ref{nuclear_obj}), the proposed projected subgradient method uses the iteration
\begin{align} \label{gradstep}
\mathbf{T}^{(t+1)} = \mathbf{T}^{(t)} - \nu \Delta \mathbf{T} ,
\end{align}
where $\mathbf{T}^{(t)}$ is the $t$-th iterate, and $\nu > 0$ defines the step size. The subgradient step $\Delta \mathbf{T}$ is evaluated as
\begin{align} \label{stran_sub}
\Delta \mathbf{T} = \sum_{c=1}^C \partial ||\mathbf{T Y}_c||\mathbf{Y}_c^T - \partial ||\mathbf{T Y}||\mathbf{Y}^T,
\end{align}
where $\partial ||\mathbf{\cdot}||$ is the subdifferential of the norm $||\mathbf{\cdot}||$ (given a matrix $\mathbf{A}$, the subdifferential $\partial ||\mathbf{A}||$ can be evaluated using the simple approach shown in Algorithm~\ref{subdifferential}, \cite{subdifferential}). After each iteration, we project $\mathbf{T}$ via $\gamma \frac{\mathbf{T}}{||\mathbf{T}||}$.
% This algorithm converges to a local minimum, as we briefly describe next.

The objective function (\ref{nuclear_obj}) is a D.C. (difference of convex functions) program (\cite{dc3}, \cite{dc1}, \cite{dc2}).
 We provide here a simple convergence analysis to the projected subgradient approach proposed above.

We first provide an analysis to the minimization of (\ref{nuclear_obj}) without the norm constraint $||\mathbf{T}||_2 = \gamma$, using the following iterative D.C. procedure (\cite{dc1}):
\begin{itemize*}
\item Initialize $\mathbf{T}^{(0)}$ with the identity matrix;
\item At the $t$-th iteration, we update $\mathbf{T}^{(t+1)}$  by solving a convex minimization sub-problem (\ref{auxiliary_cov}),
\begin{align} \label{auxiliary_cov}
\mathbf{T}^{(t+1)} = \underset{\mathbf{T}} \arg \min \sum_{c=1}^C ||\mathbf{T Y}_c||_* -  \partial ||\mathbf{T}^{(t)} \mathbf{Y}||\mathbf{Y}^T \mathbf{T}^T.
\end{align}
\end{itemize*}
where the first term in (\ref{auxiliary_cov}) is the convex term in (\ref{nuclear_obj}), and the added second term is a linear term on $\mathbf{T}$ using a subgradient of the concave term in (\ref{nuclear_obj}) evaluated at the current iteration.

We solve this sub-objective function (\ref{auxiliary_cov}) using the subgradient method, i.e., using a constant step size $\nu$, we iteratively take a step in the negative direction of subgradient, and the subgradient is evaluated as
\begin{align} \label{graddes2}
 \sum_{c=1}^C \partial ||\mathbf{TY}_c||\mathbf{Y}_c^T - \partial ||\mathbf{T^{(t)} Y}||\mathbf{Y}^T.
\end{align}
Though each subgradient step does not guarantee a decrease of the cost function (\cite{subgradient, Recht}), using a constant step size, the subgradient method is guaranteed to converge to within some range of the optimal value for a convex problem (\cite{subgradient}) (it is easy to notice that (\ref{gradstep}) is a simplified version of this D.C. procedure by performing only one iteration of the subgradient method in solving the sub-objective function (\ref{auxiliary_cov}), and we will have more discussion on this simplification later).

Given $\mathbf{T}^{(t+1)}$ as the minimizer found for the convex problem (\ref{auxiliary_cov}) using the subgradient method, we have for (\ref{auxiliary_cov}),
\begin{align} \label{auxiliary_ineq1}
\sum_{c=1}^C ||\mathbf{T}^{(t+1)} \mathbf{Y}_c||_* - \partial ||\mathbf{T}^{(t)} \mathbf{Y}||\mathbf{Y}^T \mathbf{T}^{(t+1)^T} \\ \nonumber
\le \sum_{c=1}^C ||\mathbf{T}^{(t)} \mathbf{Y}_c||_* - \partial ||\mathbf{T}^{(t)} \mathbf{Y}||\mathbf{Y}^T \mathbf{T}^{(t)^T},
\end{align}
and from the concavity of the second term in (\ref{nuclear_obj}), we have
\begin{align} \label{auxiliary_ineq2}
- ||\mathbf{T}^{(t+1)} \mathbf{Y}||_*
\le - ||\mathbf{T}^{(t)} \mathbf{Y}||_*  - \partial ||\mathbf{T}^{(t)} \mathbf{Y}||\mathbf{Y}^T (\mathbf{T}^{(t+1)} - \mathbf{T}^{(t)})^T.
\end{align}
By summing  \ref{auxiliary_ineq1} and  \ref{auxiliary_ineq2}, we obtain
\begin{align} \label{auxiliary_eq}
\sum_{c=1}^C ||\mathbf{T}^{(t+1)} \mathbf{Y}_c||_* - ||\mathbf{T}^{(t+1)} \mathbf{Y}||_*
\le  \sum_{c=1}^C ||\mathbf{T}^{(t)} {Y}_c||_* - ||\mathbf{T}^{(t)} {Y}||_*.
\end{align}
The objective (\ref{nuclear_obj}) is bounded from below by $0$ (shown in Section \ref{sec:form}), and decreases after each iteration of the above D.C. procedure (shown in (\ref{auxiliary_eq})). Thus, the convergence to a local minimum (or a stationary point) is guaranteed.
For efficiency considerations, while solving the convex sub-objective function (\ref{auxiliary_cov}), we perform only one iteration of the subgradient method to obtain a simplified method (\ref{gradstep}), and still observe empirical convergence in all experiments, see Fig.\ref{fig:convergence} and Fig.\ref{fig:conv_T_yale}.

The norm constraint $||\mathbf{T}||_2 = \gamma$ is adopted in our formulation to prevent the trivial solution $\mathbf{T} = 0$.
As shown above,  the minimization of (\ref{nuclear_obj}) without the norm constraint always converges to a local minimum (or a stationary point), thus the initialization becomes critical while dropping the norm constraint.
By initializing $\mathbf{T}^{(0)}$ with the identity matrix, we observe no trivial solution convergence in all experiments, such as the no normalization case in Fig.\ref{fig:convergence}.

As shown in \cite{unitnorm}, the norm constraint $||\mathbf{T}||_2 = \gamma$ can be incorporated to a gradient-based algorithm using various alternatives, e.g., Lagrange multipliers, coefficient normalization, and gradients in the tangent space. We implement the coefficient normalization method, i.e., after obtaining $\mathbf{T}^{(t+1)}$ from (\ref{auxiliary_cov}),  we normalize $\mathbf{T}^{(t+1)}$ via $\gamma \frac{\mathbf{T}^{(t+1)}}{||\mathbf{T}^{(t+1)}||}$. In other words, we normalize the length of $\mathbf{T}^{(t+1)}$ without changing its direction.
As discussed in \cite{unitnorm}, the problem of minimizing a cost function subject to a norm constraint forms the basis for many important tasks, and gradient-based algorithms are often used along with the norm constraint.
Though it is expected that a norm constraint does not change the convergence behavior of a gradient algorithm (\cite{unitnorm, unitnorm2}),  Fig.\ref{fig:convergence}, to the best of our knowledge, it is an open problem to analyze how a norm constraint and the choice of $\gamma$ affect the convergence behavior of a gradient/subgradient method.

}

\begin{algorithm}[ht]
%\SetAlFnt{\footnotesize \sf}
\footnotesize
\KwIn{An $m \times n$ matrix $\mathbf{A}$, a small threshold value $\delta$}
\KwOut{The subdifferential of the matrix norm $\partial ||\mathbf{A}||$.}
\Begin{
\BlankLine
1. Perform singular value decomposition: \\
$\mathbf{A}=\mathbf{U} \mathbf{\Sigma} \mathbf{V}$ \;
\BlankLine
2. $s \leftarrow$ the number of singular values smaller than $\delta$ , \\
3. Partition $\mathbf{U}$ and $\mathbf{V}$ as \\
$\mathbf{U} = [\mathbf{U}^{(1)}, \mathbf{U}^{(2)}]$, $\mathbf{V} = [\mathbf{V}^{(1)}, \mathbf{V}^{(2)}]$ \;
where $\mathbf{U}^{(1)}$ and $\mathbf{V}^{(1)}$ have $(n-s)$ columns. \\
\BlankLine
4. Generate a random matrix $\mathbf{B}$ of the size $(m-n+s)\times s$, \\
$\mathbf{B} \leftarrow \frac{\mathbf{B}}{||\mathbf{B}||}$ \;
\BlankLine
5. $\partial ||\mathbf{A}|| \leftarrow \mathbf{U}^{(1)} \mathbf{V}^{(1)T} + \mathbf{U}^{(2)} \mathbf{B} \mathbf{V}^{(2)T}$ \;
\BlankLine
6. Return $\partial ||\mathbf{A}||$ \;

}
\caption{An approach to evaluate the subdifferential of a matrix norm.}
\label{subdifferential}
\end{algorithm}

\section*{Acknowledgments}

This work was partially supported by ONR, NGA, NSF, ARO, DHS, and AFOSR.
We thank Dr. Pablo Sprechmann, Dr. Ehsan Elhamifar, Ching-Hui Chen, and Dr. Mariano Tepper for important feedback on this work.

\vskip 0.2in
\bibliography{lowrankT}

\end{document}